\documentclass[final,1p,times]{elsarticle}

\usepackage{amsmath,amssymb,amsthm}
\usepackage{booktabs}
\usepackage{graphicx}
\usepackage{url}
\usepackage{hyperref}
\hypersetup{hypertexnames=false}
\usepackage{enumitem}

\usepackage{tabularx}
\usepackage{array}
\usepackage{rotating}

\usepackage{multirow}
\usepackage{float}
\usepackage{xcolor}

\newtheorem{definition}{Definition}
\newtheorem{proposition}{Proposition}
\newtheorem{lemma}{Lemma}

\newtheorem*{definition*}{Definition}
\newtheorem*{example*}{Example}

\newcommand{\repeateddefinitionname}{}

\newtheorem*{repeateddefinition}{\repeateddefinitionname}

\usepackage{tikz}
\usetikzlibrary{positioning,arrows.meta,fit,backgrounds,calc}
 
\newcommand{\CC}{\textsf{CollaborationCase}}
\newcommand{\OC}{\textsf{OrchestrationCase}}
\newcommand{\Msg}{\textsf{Message}}
\newcommand{\TPa}{\ensuremath{\mathcal{T}_{\mathsf{Pa}}}}

\usepackage[utf8]{inputenc}
\usepackage[english]{babel}
\usepackage[T1]{fontenc}

\makeatletter
\if@twocolumn
  \newenvironment{widetable}{\begin{table*}[t]}{\end{table*}}
  \newenvironment{widesidewaystable}{\begin{sidewaystable*}[p]}{\end{sidewaystable*}}
\else
  \newenvironment{widetable}{\begin{table}[!ht]}{\end{table}}
  \newenvironment{widesidewaystable}{\begin{sidewaystable}[p]}{\end{sidewaystable}}
\fi
\makeatother

\journal{Information Systems}

\begin{document}

\begin{frontmatter}

\title{A Framework for Object-Centric Predictive Monitoring of Collaborative Processes}

\author[ort]{Daniel Calegari\corref{cor1}}
\ead{calegari@ort.edu.uy}
\cortext[cor1]{Corresponding author.}

\author[udelar]{Andrea Delgado}
\ead{adelgado@fing.edu.uy}

\author[udelar]{Leonel Pe\~{n}a}
\ead{lpena@fing.edu.uy}

\author[udelar]{Mart\'{i}n Rubio}
\ead{mrubio@fing.edu.uy}

\affiliation[ort]{organization={Universidad ORT Uruguay}, addressline={Cuareim 1451}, 
            city={Montevideo},
            postcode={11100}, 
            state={Montevideo}, country={Uruguay}}
\affiliation[udelar]{organization={Instituto de Computación, Facultad de Ingeniería, Universidad de la República}, addressline={J. Herrera y Reissig 565}, 
            city={Montevideo},
            postcode={11300}, 
            state={Montevideo}, country={Uruguay}}

\begin{abstract}
Predictive Process Monitoring (PPM) of collaborative, inter-organizational processes requires reasoning over multiple interdependent entities, including participants, messages, local executions, and the global collaboration case. Existing approaches extend traditional event logs with collaboration attributes but retain a single-case perspective, leaving much of this structure implicit. Object-centric process mining (OCPM) provides an alternative by representing these entities as first-class objects with explicit relations and multiple notions of case.
This study connects collaborative PPM and OCPM through three contributions: (i) a formal semantic mapping from extended collaborative event logs to an OCED-conformant object-centric representation, serialized in OCEL 2.0; (ii) a reformulation of collaborative prediction tasks as object-centric prediction tasks; and (iii) a reproducible converter and prediction pipeline implementing the proposed mapping.
We evaluate the framework on four public collaborative event logs and a fifth derived from the BPI Challenge 2013 incident-management log by executing the fourteen reformulated tasks using five predictive strategies across tabular, sequential, and graph-native encodings. We further discuss the benefits, limitations, and threats to the approach's validity. The representation makes collaboration structure explicit and makes it natural to state prediction targets based on object relations that fall outside the case-centric taxonomy, at the cost of increased relational complexity and dependence on object-centric tooling.
\end{abstract}

%%Graphical abstract (also submitted as a separate file: graphical_abstract/graphical_abstract.pdf)
%\begin{graphicalabstract}
%    \begin{center}
%    \includegraphics[width=\linewidth]{graphical_abstract/graphical_abstract.pdf}
%    \end{center}
%\end{graphicalabstract}

%\input{highlights}


\end{frontmatter}

% ==================
\section{Introduction}\label{sec:introduction}
% ==================
 
Collaborative inter-organizational business processes span multiple autonomous participants who coordinate toward a shared goal through message exchanges, while each participant executes its own internal (local) orchestration. Compared to single-organization orchestrations, these settings raise several challenges, e.g., fragmented visibility, privacy constraints, asynchronous interactions, and the correlation of events across participants, which have motivated a growing body of work on process mining for collaborative and inter-organizational processes \citep{aalst2011collabproc,aalst2021federated,delgado2025predictcollab}.

Predictive Process Monitoring (PPM) \citep{maggi2014predictive,difrancescomarino2022ppm} uses historical execution data and observations of running cases (event logs) to predict future states or outcomes of ongoing process instances, such as the next activity, a case outcome, or the remaining time. 
These predictions can support proactive management by allowing organizations to anticipate delays, deviations, or coordination problems. Traditional PPM techniques \citep{tax2017predictive,teinemaa2019outcome,verenich2019survey}, nevertheless, assume a single-case notion: each event belongs to a single case, and the prediction is obtained from the prefix of that case. 
This assumption fits intra-organizational orchestrations, but it is restrictive in collaborative processes, where an instance involves a global collaboration, several local executions, the participants involved, and the messages exchanged among them.

Existing collaborative PPM techniques \citep{delgado2025predictcollab} address this problem by extending the event log with collaboration attributes. In Predict-Collab, each event in a global case carries attributes that identify the enacting participant, the element type (internal task, send, or receive), and, for message events, the sending and receiving participants. Although this representation is sufficient for adapting conventional PPM techniques, the collaborative structure remains implicit: participants and message exchanges are encoded as event attributes within traces identified by a single case identifier. Consequently, participant- and message-centered analyses require additional projections and preprocessing steps, which may affect traceability and make it harder to consistently obtain different collaborative viewpoints.

Object-centric process mining (OCPM) \citep{aalst2019objectcentric,aalst2020discovering} was developed precisely to lift the single-case assumption. 
In object-centric event data, events may involve multiple objects of different types, and analysis can be performed from different viewpoints. At the data-modeling level, the Object-Centric Event Data (OCED) \citep{OCEDstandard} standardization initiative seeks to establish a common conceptual and semantic foundation for representing, exchanging, and processing such data. In particular, the OCED Core Model captures the fundamental notions of events, objects, their attributes, and the explicit relations among them, i.e., qualified event-to-object and object-to-object relations. 
In this study, OCEL~2.0 \citep{berti2024ocel} is adopted as the concrete, tool-supported serialization of the object-centric representation, since it supports and extends the OCED Core Model. These capabilities operationalize the OCED concepts needed in collaborative process settings. Participants, local cases (orchestrations), communication interactions, and global collaboration cases can therefore be modeled as explicitly related objects rather than as attributes attached to events or traditional cases.

A recent object-centric research manifesto \citep{seidel2026manifesto} consolidates the terminology and behavioral characteristics of object-centric process management and identifies open challenges for object-centric monitoring and prediction. These include defining prediction targets in the absence of a unique case notion, selecting appropriate analysis perspectives, and ensuring the reproducibility of prediction pipelines. 
In parallel, several object-centric PPM approaches have been proposed \citep{gherissi2023ocppm,adams2022framework,galanti2023predictive}.

Despite this conceptual alignment, the relationship between collaborative PPM and object-centric PPM has not yet been investigated in depth \citep{delgado2025predictcollab}. Existing collaborative PPM techniques rely on extended collaborative event logs but lack an explicit mapping to an object-centric representation. Conversely, object-centric predictive approaches have primarily been evaluated on intra-organizational data and do not explicitly address collaboration-specific entities, relations, and prediction targets. Consequently, the following research questions arise:
\begin{description}
  \item[RQ1.] How can extended collaborative event logs be transformed into a semantically consistent object-centric representation?
  \item[RQ2.] How can collaborative PPM tasks be reformulated as object-centric PPM tasks without loss of label information?
  \item[RQ3.] Can these reformulated tasks be implemented and executed in an object-centric PPM tool, and at what computational cost?
  \item[RQ4.] What benefits and limitations does the object-centric representation exhibit relative to the extended collaborative representation?
\end{description}

This work addresses these questions by proposing and assessing an object-centric representation of collaborative processes for PPM. First, we define a semantic mapping from extended collaborative event logs to an object-centric representation conforming to the OCED Core Model \citep{OCEDstandard}. The representation is serialized in OCEL~2.0 \citep{berti2024ocel} and specified through formal mapping rules with machine-checkable consistency criteria. We also provide a reproducible converter with automated verification. 

Second, we provide formal object-centric definitions of the collaborative prediction tasks. We take Predict-Collab \citep{delgado2025predictcollab}, which introduced an extended collaborative event log and a case-centric taxonomy of fourteen collaborative prediction tasks. We retain that taxonomy and reformulate each of its tasks in terms of the object-centric representation.
The reformulation also makes targets over objects and their relations expressible, which a single flattened case cannot state.

Third, we demonstrate the definitions end-to-end on five collaborative event logs using a native object-centric pipeline built on the \texttt{ocpa} library \citep{adams2022ocpa}, which imports OCEL~2.0 natively and implements the object-centric feature extraction framework of \citep{adams2022framework}, so that features are extracted without flattening the log. Four of the logs are the publicly available logs from the collaborative baseline; the fifth is derived from an incident-management log and is an order of magnitude larger, so the demonstration is not confined to curated data. Over three encodings built on \texttt{ocpa} --- tabular, sequential, and graph --- we exercise five representative predictive strategies, namely Random Forest, XGBoost, Transformer, LSTM, and graph neural networks (GNNs), without proposing a new predictive model or claiming predictive superiority for any of them. The converter, the pipeline, and the complete results are available at \cite{OCPPMcollab}.

Finally, we characterize the benefits and limitations of the representation relative to the extended collaborative one, including how prediction targets outside the case-centric taxonomy arise naturally from the object structure.

The evaluation shows that the mapping and the reformulation hold on all five logs: the constructed representation satisfies every consistency criterion and the exported logs validate against the OCEL~2.0 schema, the object-centric targets reproduce their source-level labels on every evaluated prefix, and the fourteen tasks execute end-to-end for all five predictors. It also shows where executability ceases to be informative. The reformulated categorical and structural targets are learned above a trivial baseline on all but one task--log--predictor entry, whereas on the real-life log, the numeric ones are not, and on these five logs the execution cost of the pipeline is predictable from prefix count only for the graph-native encoding. We report these observations descriptively, since the fold-to-fold dispersion of the measurements does not support ranking the predictors against one another.

The rest of this paper is organized as follows. Section~\ref{sec:background} presents background information and related work on collaborative processes, object-centric process mining, and PPM. Section~\ref{sec:mapping} defines the object-centric representation and the mapping. Section~\ref{sec:tasks} reformulates the prediction tasks.
Section~\ref{sec:implementation} describes how the representation, the task labels, and the prediction pipeline are realized.
Section~\ref{sec:evaluation} presents the empirical validation of the proposed representation and tasks.
Section~\ref{sec:discussion} discusses benefits and limitations. Section~\ref{sec:threats} analyzes threats to validity. Finally, Section~\ref{sec:conclusion} draws conclusions and outlines future work.
% =======================
\section{Background and Related Work}\label{sec:background}
% =======================

 % =============================================
\subsection{Collaborative business processes}\label{sec:bg-collab}
% =============================================

A collaborative business process is enacted jointly by multiple autonomous participants---organizations or organizational roles---each of which executes an internal orchestration and coordinates with the others through message exchanges \cite{delgado2025predictcollab}. As an application example, we use the healthcare collaborative process introduced in \cite{Lorenzo22} and depicted in Figure \ref{fig:healthcare}. The process illustrates a healthcare scenario in which a patient, a gynecologist, a hospital, and a laboratory collaborate to treat a gynecological disease. 

\begin{figure}[!ht]
    \centering
    \includegraphics[width=\linewidth]{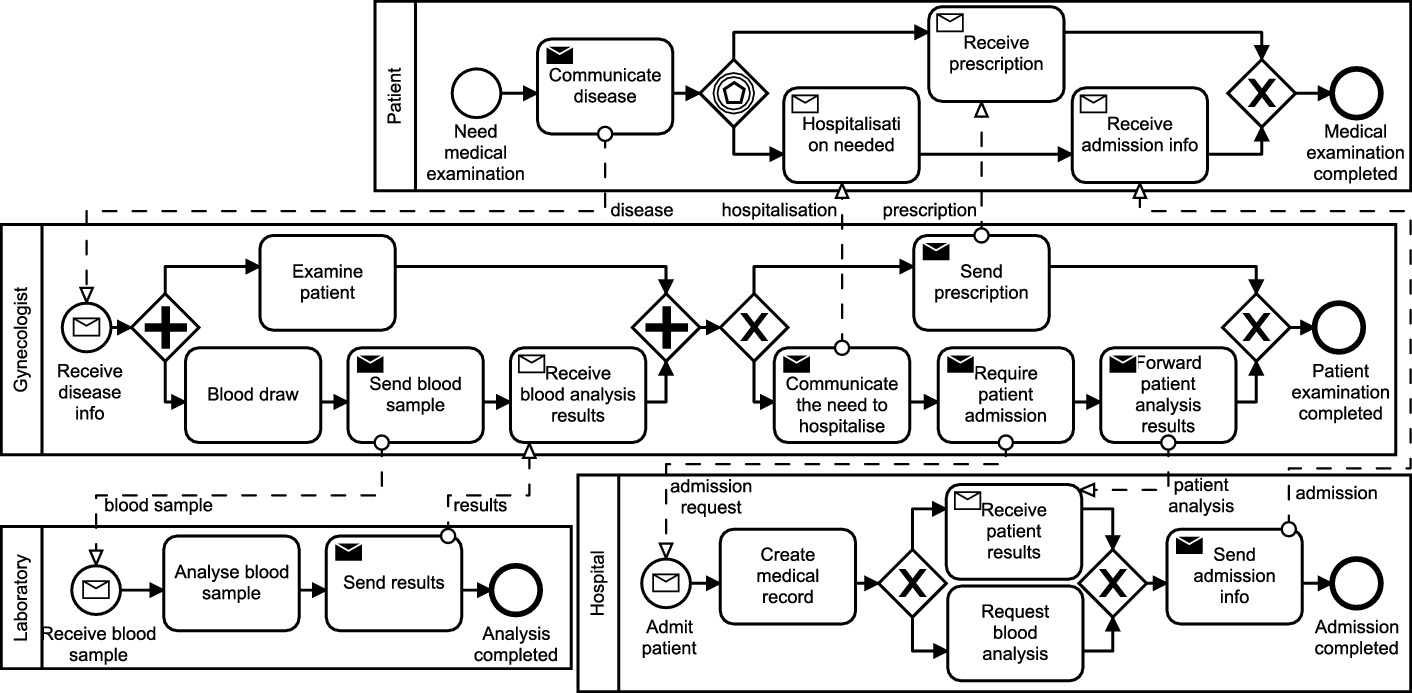}
    \caption{A BPMN 2.0 healthcare business process collaboration from \cite{Lorenzo22}}
    \label{fig:healthcare}
\end{figure}

A \emph{collaboration case} is one end-to-end execution of the collaboration,
spanning the events of all involved participants, e.g., a full execution of the healthcare scenario in Figure \ref{fig:healthcare}. Within a collaboration case, an \emph{orchestration case} is a single, identifiable execution of the local process enacted by a participant, e.g., a single run of the Laboratory process. 
An \emph{activity instance} is an event performed by exactly one participant within a single orchestration case, e.g., a specific blood analysis. 
A \emph{resource} is the individual actor---a person, a team, or a software system---that carries out an activity instance on behalf of a participant, e.g., the laboratory technician who runs a specific blood analysis. Resources and participants are distinct granularities: a participant identifier denotes an organization or an organizational role participating in the collaboration, whereas a resource identifier denotes the actor performing it. 
A \emph{message} is a unit of inter-participant coordination that manifests in the log as up to two events (a send and a receive), each carrying its own-side endpoint and, where the source records it, the counterparty endpoint, e.g., send and receive blood-sample messages carrying the blood sample.

Several challenges arise in log building and PPM in distributed settings, including fragmented data visibility, strict privacy requirements, divergent organizational objectives, and asynchronous events across loosely coupled systems \cite{aalst2011collabproc,aalst2021federated,delgado2025predictcollab}. We deliberately leave integration challenges out of scope, such as correlating message events. In this sense, an integrated log could be obtained by merging participant events under a global case identifier and extending the XES format \citep{ieee2016xes} with collaboration attributes \citep{xescollabextension,delgado2025predictcollab}: each event records the enacting participant (\texttt{participant}), its element type (\texttt{elemType} $\in$ \{task, SendTask, ReceiveTask\}), and, for message events, whichever of \texttt{fromParticipant} and \texttt{toParticipant} the source system captured. These collaboration attributes are participant-level: the format prescribes no actor identifier, so a resource becomes visible only when the source log carries one as an ordinary event attribute, typically the standard XES \texttt{org:resource} \citep{ieee2016xes}. We refer to this representation as the \emph{extended collaborative event log}, or \emph{collaborative event log} for short (see full definition in \ref{app:formal}). 

\begin{definition*}[\textbf{Collaborative event log}]
An extended \textbf{collaborative event log} (adapted from \cite{delgado2025predictcollab}) is a tuple
\[
L = (E_L, \mathcal{C}, \mathcal{P}, \mathcal{A}, \mathit{case}, \mathit{act}, \mathit{time}_L, \mathit{part}, \mathit{elem}, \mathit{from}, \mathit{to})
\]
where $E_L$ is a finite set of events; $\mathcal{C}$, $\mathcal{P}$, and $\mathcal{A}$ are sets of case identifiers, participant identifiers, and activity labels; $\mathit{case}\colon E_L \to \mathcal{C}$, $\mathit{act}\colon E_L \to \mathcal{A}$, $\mathit{time}_L\colon E_L \to \mathbb{T}$, $\mathit{part}\colon E_L \to \mathcal{P}$, and $\mathit{elem}\colon E_L \to \{\textit{task}, \textit{SendTask}, \textit{ReceiveTask}\}$ are total functions; and $\mathit{from}, \mathit{to}\colon E_L \rightharpoonup \mathcal{P}$ are partial functions defined only when $\mathit{elem}(e) \neq \textit{task}$, and possibly undefined even then, when the source does not record that endpoint. The timestamp domain $\mathbb{T}$ is totally ordered.
\end{definition*}

\begin{example*}[\textbf{Collaborative event log}]
Table \ref{table:relevantDataCase459} presents an excerpt from a collaborative event log for the healthcare collaborative process, showing case \texttt{case\_459}, in which the \texttt{fromParticipant} and \texttt{toParticipant} attributes were omitted for simplicity. This case describes one execution of the activities performed by a specific Patient, a Gynecologist, and the Laboratory during a consultation with the Patient. The execution of case \texttt{case\_459} (column \texttt{case}) registered the events corresponding to the activities (column \texttt{activity}) in the order they were performed, i.e., from top to bottom in the table with corresponding timestamps (column \texttt{timestamp}), and the extension data \texttt{elemType} (column \texttt{elemType}) and corresponding participant (column \texttt{participant}). 
\end{example*}

\begin{widetable}
\setlength{\tabcolsep}{2pt}
\caption{Relevant data from the collaborative event log in the example case \texttt{case\_459}.}%
\label{table:relevantDataCase459}
{\footnotesize
\begin{tabular}
{c|l|l|l|l}
\toprule
\textbf{case} & \textbf{timestamp} & \textbf{activity} & \textbf{elemType} & \textbf{participant} \\
\hline
\texttt{case\_459} & 2021-06-17 07:17:07 & Communicate disease            & SendTask    & Patient      \\
\texttt{case\_459} & 2021-06-17 07:17:10 & Receive disease info           & ReceiveTask & Gynecologist \\
\texttt{case\_459} & 2021-06-17 07:17:13 & Examine patient                & task        & Gynecologist \\
\texttt{case\_459} & 2021-06-17 07:17:18 & Blood draw                     & task        & Gynecologist \\
\texttt{case\_459} & 2021-06-17 07:17:21 & Send blood sample              & SendTask    & Gynecologist \\
\texttt{case\_459} & 2021-06-17 07:17:25 & Receive blood sample           & ReceiveTask & Laboratory   \\
\texttt{case\_459} & 2021-06-17 07:17:25 & Analyse blood sample           & task        & Laboratory   \\
\texttt{case\_459} & 2021-06-17 07:17:27 & Send results                   & SendTask    & Laboratory   \\
\texttt{case\_459} & 2021-06-17 07:17:30 & Receive blood analysis results & ReceiveTask & Gynecologist \\
\texttt{case\_459} & 2021-06-17 07:17:33 & Send prescription              & SendTask    & Gynecologist \\
\texttt{case\_459} & 2021-06-17 07:17:38 & Receive prescription           & ReceiveTask & Patient      \\
\bottomrule
\end{tabular}
}
\end{widetable}

% =============================================
\subsection{Object-centric process mining}\label{sec:bg-ocpm}
% =============================================
 
Object-centric process mining starts from the observation that real event data frequently involves multiple interacting entities, so forcing each event into exactly one case distorts the data: \emph{flattening} an object-centric log onto a single case notion can replicate events (convergence) or interleave unrelated behavior (divergence) \citep{aalst2019objectcentric,aalst2020discovering}. In the object-centric paradigm, events are correlated with all objects they involve, multiple notions of case coexist as object types, and analysis explicitly chooses its perspective. A recent object-centric research manifesto \citep{seidel2026manifesto} consolidates the terminology and behavioral characteristics of object-centric process management. It formulates a research agenda whose open challenges include the explicit modeling of case notions, perspective selection, object-centric feature extraction, and the reproducibility and definition of object-centric prediction. 
We adopt the term \emph{viewpoint} for the object type (or set of object types) from whose perspective a prediction target or a flattened log is defined, following its use in object-centric predictive process monitoring \citep{galanti2023predictive}, where an object-centric log is unfolded into single-case traces around a chosen viewpoint object type; the same notion is termed the \emph{leading type} in \citep{adams2022framework}.

At the conceptual level, we adopt the object-centric event data (OCED) core model (OCED Core Model) \citep{OCEDstandard} as the reference for our representation. 
The OCED Core Model is a deliberately minimal metamodel, decoupled from any particular storage format, that captures the concepts common to object-centric event data: atomic events, each of exactly one event type and one timestamp; objects, each of exactly one object type; event and object attributes, represented as typed values owned by their carrying event or object; directed binary \emph{object relations} between objects (O2O relations), each qualified by an object relation type; and a qualified relation, referred to as \emph{observes}, that correlates events with the objects they involve (E2O relations). In its baseline interpretation, attributes and object relations are \emph{static}, i.e., described as a single view rather than as a pair of related objects that identify time-indexed values and an object relation, so changes in relationships over time are not expressible at the core level \citep{OCEDstandard}. This last property constrains how transient states, such as an in-flight message, can be represented (Section~\ref{sec:mapping}). For readability, throughout the paper, we keep the abbreviations E2O for the \emph{observes} relation of the core model and O2O for its \emph{object relations}. Figure~\ref{fig:ocedmetamodel} presents the OCED core metamodel.

\begin{figure}[!ht]
    \centering
    \includegraphics[width=\linewidth]{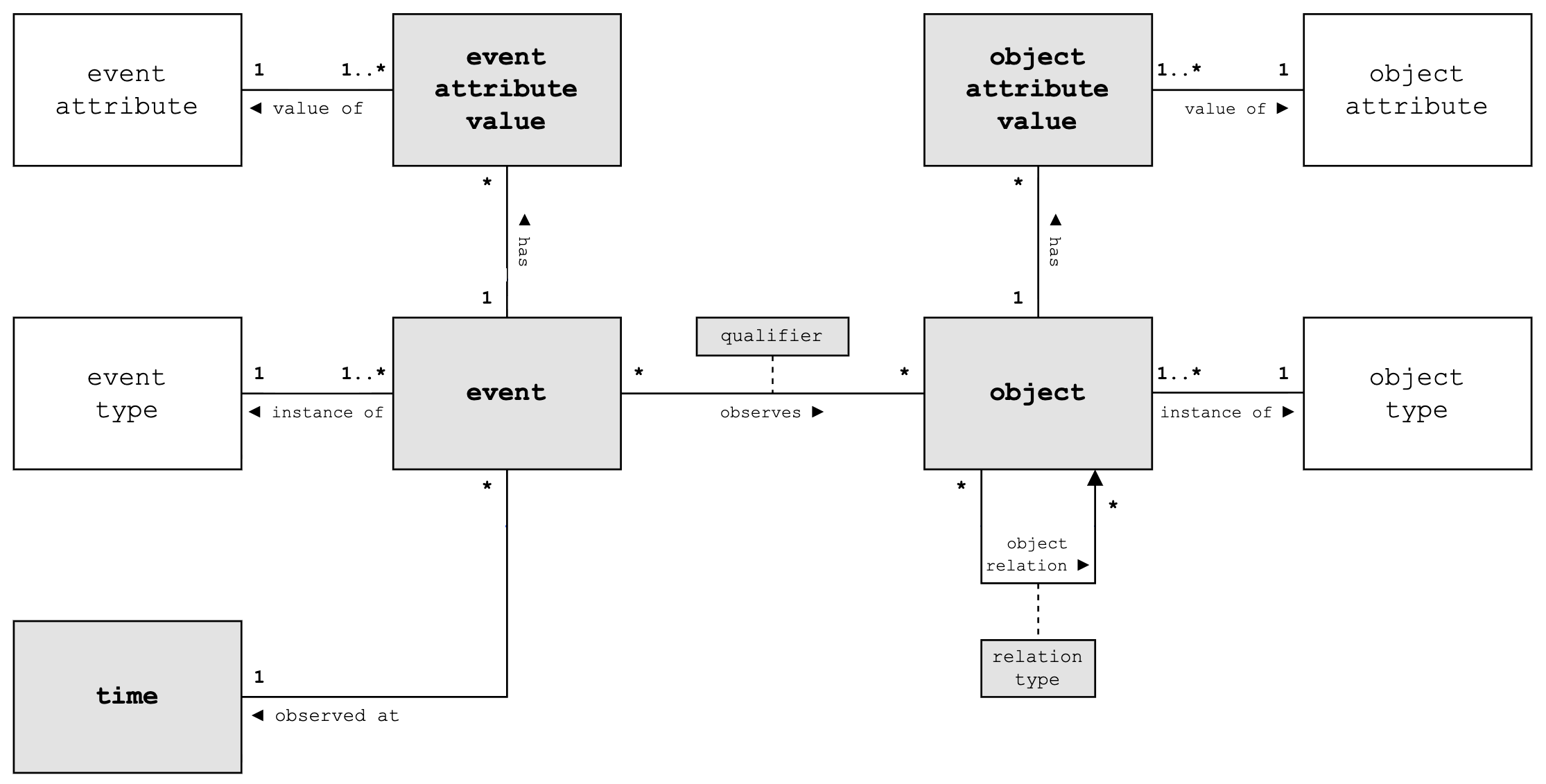}
    \caption{OCED Core Model from \cite{OCEDstandard}}
    \label{fig:ocedmetamodel}
\end{figure}

For concrete serialization and tool support, we target OCEL~2.0 \citep{berti2024ocel}, a concrete specification and serialization format for object-centric event logs, in which events reference multiple objects from typed object classes. OCEL~2.0 varies from and extends the OCED-MM Base Model, and \citep{OCEDstandard} names as the main differences a per-type attribute schema that relates admissible attributes to each event type and object type, and timestamped object attribute values for a dynamic view of the objects' history. It is worth noting that, as observed there, omitting both yields a reduced OCEL~2.0 metamodel that is similar to the OCED Core Model up to the naming of relationships. 

In \cite{berti2024ocel}, the metamodel is supported by a formalization that is crucial for understanding and using OCEL~2.0 (see full definition in \ref{app:formal}).

\begin{definition*}[\textbf{Object-centric event log}]
Let $\mathbb{U}_{\mathit{ev}}$, $\mathbb{U}_{\mathit{etype}}$, $\mathbb{U}_{\mathit{obj}}$, $\mathbb{U}_{\mathit{otype}}$, $\mathbb{U}_{\mathit{attr}}$, $\mathbb{U}_{\mathit{val}}$, $\mathbb{U}_{\mathit{time}}$ (with smallest element $0$), and $\mathbb{U}_{\mathit{qual}}$ be the pairwise disjoint universes of events, event types, objects, object types, attribute names, values, timestamps, and qualifiers. An \textbf{object-centric event log} $\mathcal{L}$\cite{berti2024ocel} is a tuple 
\[(E, O, EA, OA, \mathit{evtype}, \mathit{time}, \mathit{objtype}, \mathit{eatype}, \mathit{oatype}, \mathit{eaval}, \mathit{oaval}, \mathit{E2O}, \mathit{O2O})
\]
with $E \subseteq \mathbb{U}_{\mathit{ev}}$, $O \subseteq \mathbb{U}_{\mathit{obj}}$, $\mathit{evtype}\colon E \to \mathbb{U}_{\mathit{etype}}$, $\mathit{time}\colon E \to \mathbb{U}_{\mathit{time}}$, $EA \subseteq \mathbb{U}_{\mathit{attr}}$, $\mathit{eatype}\colon EA \to \mathbb{U}_{\mathit{etype}}$, $\mathit{eaval}\colon (E \times EA) \rightharpoonup \mathbb{U}_{\mathit{val}}$, $\mathit{objtype}\colon O \to \mathbb{U}_{\mathit{otype}}$, $OA \subseteq \mathbb{U}_{\mathit{attr}}$, $\mathit{oatype}\colon OA \to \mathbb{U}_{\mathit{otype}}$, $\mathit{oaval}\colon (O \times OA \times \mathbb{U}_{\mathit{time}}) \rightharpoonup \mathbb{U}_{\mathit{val}}$, $\mathit{E2O} \subseteq E \times \mathbb{U}_{\mathit{qual}} \times O$, and $\mathit{O2O} \subseteq O \times \mathbb{U}_{\mathit{qual}} \times O$.
\end{definition*}

% =============================================
\subsection{Predictive process monitoring}\label{sec:ppm}
% =============================================
 
Predictive process monitoring (PPM) \citep{maggi2014predictive,difrancescomarino2022ppm} trains models on historical event logs to forecast properties of running cases from their prefixes. 
Typical targets include the next activity and its attributes, case outcomes, and remaining execution time; established techniques range from classical classifiers over engineered prefix encodings to recurrent and attention-based neural architectures \citep{tax2017predictive,teinemaa2019outcome,verenich2019survey}. These works define the prediction target over the prefix of a single case, and the encoding step consumes a flat sequence of events for that case, with no explicit notion of participants, messages, or inter-organizational structure.

Few studies address PPM in collaborative processes. In \citep{delgado2025predictcollab}, the authors consider a case-centric approach in which encoding flattens events and injects the collaboration structure via event attributes. 
The proposal defines a taxonomy of outcome-based, numeric values and next-event collaborative prediction types, summarized in Table~\ref{tab:predtypes}. It implements Predict-Collab, an extension of ProcessTransformer \citep{bukhsh2021processtransformer} that consumes collaboration attributes and provides a subset of the defined predictions. Their formalization of the extended log and of the prediction functions serves as the template for our own formal definitions. 

\begin{widetable}
\setlength{\tabcolsep}{3pt}
\caption{Prediction types for collaborative processes from \cite{delgado2025predictcollab}.}%
\label{tab:predtypes}
{\footnotesize
\begin{tabular}
{l|l|l}
\toprule
\bf Category & \bf Type Id & \bf Prediction \\
\midrule
Next Event & NE-NEPa
& 
Next event that is likely to occur in a participant. \\ 
& NE-NEPr
& Next event that is likely to occur in the process. \\ 
& NE-NPaA
& Next participant that is likely to act. \\ 
& NE-NPaM
& Next participant that is likely to send/receive a message. \\ 
& NE-NMPa
& Next message that is likely to occur in a participant.\\
& NE-NMPr
& Next message that is likely to occur in the process.\\ 
\midrule
Numeric value & NV-NMPa
& 
Number of remaining/total messages of a participant.\\ 
& NV-NMPr
& Number of remaining/total messages in the process.\\ 
& NV-PaT
& Participant remaining/duration time. \\ 
& NV-PrT
& Process remaining/duration time (every participant ends). \\ 
& NV-TNE
& Time until the next event. \\
& NV-TNM
& Time until the next message to send/receive. \\ 
\midrule
Outcome-based & OB-P
&
If a participant will participate in a case. \\
& OB-M
& If a particular message will be sent/received.\\
\bottomrule
\end{tabular}
}
\end{widetable}

Object-centric predictive process monitoring transfers PPM targets to object-centric event data. Approaches differ mainly in how they encode process execution prior to prediction. Tabular and sequential encodings are standard in case-centric PPM \citep{teinemaa2019outcome,difrancescomarino2022ppm}; work \citep{adams2022framework} transfers them to object-centric event data together with a graph-based encoding, and these three families structure the discussion below. Although there are many object-centric PPM approaches, we found no prior work that provides a mapping from collaborative or inter-organizational logs to object-centric representations, nor task definitions for collaborative prediction. This paper does not introduce a new prediction model; it provides the representation and task definitions that enable existing prediction techniques to be applied to collaborative processes.

\paragraph{Tabular encoding}
Traditional machine learning algorithms, such as linear regression, random forests, and gradient boosting, require a flat feature vector per prediction point. 
Work \citep{adams2022framework} defines a general framework for extracting and encoding features from object-centric event data, avoiding traditional log-level flattening, and implements it in the \texttt{ocpa} library \citep{adams2022ocpa}. The \texttt{ocpa} tabular extraction aggregates features from multiple objects linked to an event. For a given event, the encoder extracts static and dynamic attributes from all associated objects of various types and combines them into a single feature vector. This allows the model to predict the next step or remaining time based on a snapshot of the aggregated state of all interacting objects at that specific timestamp.

In \citep{galanti2023predictive}, the authors study object-centric predictive analytics with feature-based models over object interactions. Moreover, work \citep{fioretto2025comparative} compares predictive monitoring across object-centric and classical event logs and reports that the relative merits of the two regimes remain unclear.

\paragraph{Hierarchical and sequential encoding}
Sequence-based deep learning models rely on temporal sequences. While they expect sequential input, they also need to comprehend which objects interact at each step of the process. The sequential encoding of the framework in \citep{adams2022framework}, implemented in \texttt{ocpa} \citep{adams2022ocpa}, builds event sequences in which the feature vector for each state (step) has a hierarchical structure. This tensor maps the specific object types involved in the event along with their dynamic attributes (which evolve). This is fed directly into sequential models such as LSTMs that learn to capture the temporal evolution of different object types within the sequence.

Sequential encodings need not be native: in \citep{gherissi2023ocppm}, an LSTM-based approach (PPM-OC) flattens an OCEL log for a chosen object type and predicts the next activity, the next event time, and the remaining time, thereby enriching prefixes with object-derived features. Work \citep{gherissi2025embeddings} evaluates the original sequential ProcessTransformer \citep{bukhsh2021processtransformer} as a baseline over such a flattened, single-object-type sequence, alongside an LSTM baseline and its own graph-attention-embedded LSTM, which outperforms both on most object types and prediction targets. The sequential (self-attention) Transformer architecture has therefore already been applied to object-centric event data, but only over a flattened, single-object-type sequence that discards the multi-object structure; applying it to an object-centric feature set stacked into per-case sequences instead, as opposed to flattening it to a single object type or replacing it with a graph, remains unreported.

\paragraph{Graph-based encoding}
Instead of flattening or sequencing the data, this approach preserves the exact topology of the object-centric event log. These native graph structures are fed directly into Graph Neural Networks (GNNs), which leverage message-passing architectures to learn from complex many-to-many relationships without losing any structural information. The graph-based encoding in \citep{adams2022framework} extracts graph structures in which both events and objects are interconnected nodes and feeds them into a graph convolutional network for remaining-time prediction.

In \citep{adams2023preserving}, the authors show that preserving complex object-centric graph structures can improve machine-learning tasks in process mining. Work \citep{smit2024hoeg} proposes HOEG, a heterogeneous object-event graph encoding for remaining-time prediction. Work \citep{galanti2024gnn} critically examines whether graph neural networks actually help object-centric predictive analytics and finds that richer models do not automatically improve predictions. Work \citep{adams2024improving} quantifies the contribution of object-centric process mining to PPM by comparing flattened and native pipelines over the \texttt{ocpa} ecosystem, decomposing the observed gains into the impact of flattening and the impact of object-centric innovations, and reports an improvement in remaining-time prediction by combining graph neural networks with object-centric features that have no counterpart in flattened logs. Work \citep{gherissi2025embeddings} proposes an object-centric directly-follows graph embedded by a graph attention network and fed to an LSTM. Work \citep{gherissi2025framework} compares GAT, GCN, and Graph Transformer encodings of process-execution graphs extracted from OCELs, reporting that the attention-based encodings, GAT and Graph Transformer, yield the most accurate predictions of remaining time and total event count; there, attention operates over the object-event graph topology rather than over an event sequence, unlike the sequential Transformer of \citep{bukhsh2021processtransformer}. Work \citep{deleoni2025global} addresses global predictive monitoring of object-centric processes.

Table~\ref{tab:relatedwork} summarizes these approaches, which usually operate
on OCEL logs. None of these works address collaborative logs or collaboration-specific targets, such as predicting the next participant or the next message.

\begin{widetable}
\setlength{\tabcolsep}{3pt}
\caption{Object-centric PPM approaches: encoding, model, and prediction targets}%
\label{tab:relatedwork}
{\scriptsize\renewcommand{\arraystretch}{1.1}
\begin{tabular}{@{}llp{0.20\linewidth}p{0.28\linewidth}p{0.24\linewidth}@{}}
\toprule
\textbf{Approach} & \textbf{Ref} & \textbf{Encoding} & \textbf{Model} & \textbf{Prediction types} \\
\midrule
(Adams, 2022) & \cite{adams2022framework}
& Tabular, sequential, and graph
& Regression, LSTM, and GNN
& Remaining time \\

(Gherissi, 2023) & \cite{gherissi2023ocppm}
& Flattening for a selected object type
& LSTM (PPM-OC)
& Next activity, next event time, remaining time \\

(Galanti, 2023) & \cite{galanti2023predictive}
& Tabular over object interactions
& Feature-based predictive models
& KPI and outcome \\

(Adams, 2023) & \cite{adams2023preserving}
& Native graph
& GNN, compared against graph embeddings and flattening
& Next activity, next timestamp, remaining time \\

(Galanti, 2024) & \cite{galanti2024gnn}
& Tabular, sequential, and graph
& CatBoost, LSTM, and GNN; the GNN does not outperform
& KPI \\

(Smit, 2024) & \cite{smit2024hoeg}
& Native heterogeneous object-event graph (HOEG)
& Heterogeneous GNN
& Remaining time \\

(Adams, 2024) & \cite{adams2024improving}
& Native vs.\ flattened, with object-centric features
& GNN
& Remaining time \\

(Gherissi, 2025a) & \cite{gherissi2025embeddings}
& Graph embeddings over an OCDFG, and a flattened single-object-type sequence for its baselines
& GAT and LSTM; compared against LSTM and ProcessTransformer baselines on the flattened sequence
& Next activity, next event time \\

(Gherissi, 2025b) & \cite{gherissi2025framework}
& Native graph, and Graph2Vec/FGSD baselines
& GAT, GCN, and Graph Transformer, each paired with linear, ensemble, and MLP predictors
& Remaining time, total event count \\

(de Leoni, 2025) & \cite{deleoni2025global}
& Global graph of concurrent process executions
& GAT
& Global KPI \\
\bottomrule
\end{tabular}
}
\end{widetable}

% ==============================================
\section{An Object-Centric Representation of Collaborative Processes}
\label{sec:mapping}
% ==============================================

We propose an object-centric representation of collaborative processes through a model-to-model transformation that maps a collaborative event log to a representation conforming to the OCED Core Model \citep{OCEDstandard}. We first introduce the conceptual correspondence between the elements of both models, summarized in Table~\ref{tab:mapping}, and then present rules defining this transformation, along with consistency criteria that the resulting object-centric representation is guaranteed to satisfy. For implementation purposes, the resulting representation is serialized in OCEL~2.0, following the formalization of \cite{berti2024ocel}, as detailed in Section~\ref{sec:impl-serialization}. Formal definitions and proofs are provided in \ref{app:formal}. Other representations are also possible and may enable different predictive features; their investigation is left for future work.

\subsection{Design decisions}\label{sec:map-decisions}

We distinguish three object roles: organizational objects typed by participant identifiers; coordination objects (\Msg) representing send or receive interactions --- despite the name, one object per interaction, not a single correlated object per logical send--receive exchange, as detailed below --- and execution-scoping objects: (\CC{}) for end-to-end collaboration execution and (\OC{}) for a participant's local execution within it. Within the organizational role, each participant identifier is given its own object type (e.g., Patient) rather than a single generic type, Participant, carrying a name attribute.

The extended collaborative event log of Section~\ref{sec:bg-collab} does not guarantee three properties. None of them is a limitation; each is a precondition that the mapping deliberately does not impose, so that the transformation remains applicable to any log conforming to the input format. First, participant identifiers (\texttt{collab:\discretionary{}{}{}participant}, \texttt{collab:\discretionary{}{}{}fromParticipant}, \texttt{collab:\discretionary{}{}{}toParticipant}) are not assumed to distinguish individual organizational instances; the source format only guarantees organization- or role-level labels, e.g., \texttt{Hospital} denoting a role rather than an individual organization. Second, the correlation of message events is explicitly out of scope for the source format, which guarantees neither message identifiers nor explicit correlation information. Consequently, the mapping represents each communication interaction as its own object, rather than a single correlated message instance per exchange: the log carries no send--receive correspondence. Third, the format does not guarantee a native identifier for a participant's local execution within a collaboration; the mapping therefore identifies it by the collaboration case and the participant that delimit it.

Three further decisions concern how events and attributes are recorded. The per-case source order is embedded in the event identifier itself, rather than recorded as an explicit order attribute: reconstructing a trace requires enumerating a case's events in identifier order. Object attributes are kept static, even though OCEL~2.0 supports timestamped changes to them \citep{OCEDstandard,berti2024ocel}, because the source logs are not required to provide dynamic attributes for the main collaborative concepts. Finally, source activity attributes are preserved as event attributes, even when the same information is also materialized as an object or an object-to-object relation, rather than being dropped once such a relation exists. 

These aspects are made precise by the rules that follow.

\subsection{Mapping rules}\label{sec:map-rules}

The mapping translates the collaborative concepts introduced in Section~\ref{sec:bg-collab}. Each collaboration case is represented by a \CC{} object, which provides the global scope for relating all events belonging to the same end-to-end execution. A participant object, e.g., Patient, represents each collaboration participant; its object type is the participant identifier itself, so that the participants of a collaboration are distinguishable at the type level. Each orchestration case, in turn, is materialized as an \OC{} object that groups the events a participant performs within a collaboration case.
Finally, \Msg{} represents an individual communication interaction, the sending or the reception of a message, recorded by a single event.
Activity instances are represented as events whose event type is the corresponding activity label, while the source attribute \texttt{collab:elemType} is preserved to distinguish normal, send, and receive activities.
Participant involvement, collaboration-case membership, orchestration-case membership, and communication endpoints are made explicit through qualified relations. Table~\ref{tab:mapping} summarizes these correspondences. The transformation can be defined as a mapping function $\mu$ according to the rules M1--M8 below.

\begin{widesidewaystable}
\centering
\caption{Conceptual realization of collaborative concepts in the OCED representation. Relations target objects, not types: $O_{\textsf{Participant}}$ is the set of participant objects $p_p$, whose object types $\tau(p)$ constitute $\TPa$.}%
\label{tab:mapping}
\renewcommand{\arraystretch}{1.15}
\setlength{\tabcolsep}{5pt}
\scriptsize
\begin{tabularx}{\linewidth}{ >{\raggedright\arraybackslash}p{0.2\textwidth} >{\raggedright\arraybackslash}p{0.20\textwidth} >{\raggedright\arraybackslash}p{0.23\textwidth} >{\raggedright\arraybackslash}X } 
\toprule
\textbf{Collaborative concept}
&
\textbf{OCED realization}
&
\textbf{E2O relations}
&
\textbf{O2O relations}
\\
\midrule

Collaboration case
&
Object of type \CC
&
$\mathit{event}
 \xrightarrow{\texttt{in\_collaboration}}
 \CC$
&
$\OC
 \xrightarrow{\texttt{part\_of}}
 \CC$;

$\Msg
 \xrightarrow{\texttt{exchanged\_in}}
 \CC$
\\

Participant
&
Object $p_p$ whose type is $\tau(p) \in \TPa$, the encoding as an object type of the participant identifier $p \in \mathcal{P}_L$: one type per participant (e.g., \textsf{Patient}, \textsf{Hospital})
&
$\mathit{event}
 \xrightarrow{\texttt{in\_participant}}
 O_{\textsf{Participant}}$
&
$\OC
 \xrightarrow{\texttt{for\_participant}}
 O_{\textsf{Participant}}$;

$\Msg
 \xrightarrow{\texttt{from}/\texttt{to}}
 O_{\textsf{Participant}}$
 (when defined in the normalized log)
\\

Orchestration case
&
Object of type \OC
&
$\mathit{event}
 \xrightarrow{\texttt{in\_orchestration}}
 \OC$
&
$\OC
 \xrightarrow{\texttt{part\_of}}
 \CC$;

$\OC
 \xrightarrow{\texttt{for\_participant}}
 O_{\textsf{Participant}}$
\\

Activity instance
&
OCED event preserving the source event attributes, including those also materialized as relations
&
$\mathit{event} \xrightarrow{\texttt{in\_collaboration}} \CC$; \newline
$\mathit{event} \xrightarrow{\texttt{in\_orchestration}} \OC$; \newline
$\mathit{event} \xrightarrow{\texttt{in\_participant}} O_{\textsf{Participant}}$ \newline
\textit{Communication events additionally carry:} \newline $\xrightarrow{\texttt{send}}$ \textit{or} $\xrightarrow{\texttt{receive}}$\textit{}
&
Not applicable
\\

Message interaction
&
Object of type \Msg{} (one per communication interaction)
&
$\mathit{send\ event}
 \xrightarrow{\texttt{send}}
 \Msg$;

$\mathit{receive\ event}
 \xrightarrow{\texttt{receive}}
 \Msg$
&
$\Msg
 \xrightarrow{\texttt{from}/\texttt{to}}
 O_{\textsf{Participant}}$
 (when defined in the normalized log);

$\Msg
 \xrightarrow{\texttt{exchanged\_in}}
 \CC$
\\

\bottomrule
\end{tabularx}
\end{widesidewaystable}

\begin{definition*}[\textbf{Mapping}]
Given a collaborative event log $L$, the \textbf{mapping} $\mu$ produces the object-centric event log, such that $\mu(L)$ is defined component-wise by rules M1--M8. $\mu$ operates on the full signature of Definition~\ref{def:app-r1} (\ref{app:formal}), of which the tuple of Section~\ref{sec:bg-collab} is an abbreviated projection: rule M8 additionally reads the residual attribute names $X_L$, values $V_L$, and assignment $D$, and rule M5 and criterion P1.2 rely on the order $\prec_0$, none of which the abbreviated tuple names.
\end{definition*}

\begin{description}
  \item[M1 (\CC).] For each collaboration case identifier $c$ occurring in the log, we create one object $\mathit{cc}_c$ of type \CC{} with attribute $\texttt{caseId}=c$.
  \item[M2 (Participants).] For each distinct participant $p$ that occurs as a \texttt{collab:\discretionary{}{}{}participant}, \texttt{collab:\discretionary{}{}{}fromParticipant}, or \texttt{collab:\discretionary{}{}{}toParticipant} value, we create an object $p_p$ whose object type is the participant identifier $p$ itself---formally, its injective encoding $\tau(p)$ as an object type (\ref{app:formal})---with the object attribute $\texttt{name}=p$. The participant identifiers occurring in the log are thus declared as object types and we write $\TPa$ for the resulting set of participant types, one per identifier. The scope is the entire log, so the same participant identifier refers to the same participant object throughout. 
  \item[M3 (\OC).] For each pair $(c,p)$ such that participant $p$ performs at least one event in the collaboration case $c$, we create one object $\mathit{oc}_{c,p}$ of type \OC{} with object attributes $\texttt{caseId}=c$ and $\texttt{participant}=p$. The orchestration-case object is functionally determined by its collaboration-case and participant objects. It does not add information to the source log: its event set can be equivalently reconstructed by selecting the events of $c$ performed by $p$. We nevertheless materialize it to make the participant-local execution a directly addressable object-centric viewpoint and an explicit scope to which local execution attributes may be attached, without repeatedly grouping by $(c,p)$; it also supports features such as synchronization among local executions, in line with the object-centric feature classes of \cite{adams2024improving}. The log-level participant object of rule~M2 connects activity instances across cases, whereas \OC{} keeps their local executions separate.
  \item[M4 (\Msg).] For each message sent or received in a collaboration case, we create a distinct object $m$ of type \Msg. The object represents an individual communication interaction recorded by a single participant (e.g., a send event on the sender's side or a receive event on the receiver's side), rather than a centralized, correlated message instance. It stores as object attributes whichever of the sender and the receiver is defined in the log, and is linked only to the event from which it is derived. Its own side (the sender of a send, the receiver of a receive) is always among them, since the normalization supplies it from \texttt{collab:\discretionary{}{}{}participant} where the source left it implicit; only the counterparty side may be absent. The source attribute \texttt{collab:elemType} is preserved, so no additional message-type attribute is required. In accordance with the correlation precondition of Section~\ref{sec:map-decisions}, the core mapping does not infer correspondences between send and receive interactions; consequently, no \Msg{} object's relations change over time, and none transition between lifecycle states.
  \item[M5 (Events).] Each source event $e$ that records an activity instance becomes a single OCED event with the same timestamp, event type (corresponding to the activity label), and preserved event attributes (per rule M8). Events are created in the per-case source order---by timestamp, with ties broken by their order of appearance in the source log (the order $\prec_L$ of \ref{app:formal})---and the event identifier encodes this order: it embeds the case identifier together with the zero-padded rank of $e$ within it, so that comparing identifiers (e.g., lexicographically, for a string-valued encoding) agrees with $\prec_L$ within each case. The case component makes $\mu_E$ injective over the whole log, not only within a single case.
  \item[M6 (E2O).] Every event is related to many objects via qualified relations: to its collaboration case via \emph{in\_collaboration}, to its orchestration case via \emph{in\_orchestration}, and to its participant via \emph{in\_participant}. The \emph{in\_participant} relation is redundant, since it is also obtainable via the path \emph{in\_orchestration} followed by \emph{for\_participant}. We keep the direct \emph{in\_participant} relation so that participant objects remain connected to events at the event-to-object level. In contrast, the \emph{in\_orchestration}--\emph{for\_participant} path keeps the participant distinct from its per-case execution. Send events are associated with their corresponding \Msg{} interaction object via the qualifier \emph{send}, whereas receive events are associated with their own \Msg{} interaction object via the qualifier \emph{receive}. In both cases, the event retains its structural relations and all preserved source attributes.
  \item[M7 (O2O).] Each orchestration case object is related to its collaboration case object by \emph{part\_of} and to its participant object by \emph{for\_participant}. Each message object is related to its sender and receiver by \emph{from} and \emph{to}, respectively, whenever that endpoint is defined in the normalized log---always for the interaction's own side, which the normalization above supplies, and for the counterparty side only where the source registers it---and unconditionally to the corresponding collaboration case object by \emph{exchanged\_in}. The E2O and O2O qualifier vocabularies are kept disjoint to simplify qualifier-based queries.
  \item[M8 (Attribute preservation).] All source activity instance attributes are preserved as event attributes, with their type determined by the carrying event's type, e.g., \texttt{participant}, \texttt{elemType}, \texttt{fromParticipant}, and \texttt{toParticipant} are retained even when the corresponding participant involvement or communication endpoints are also materialized by E2O or O2O relations; residual data attributes are preserved unchanged. This preservation is relative to the normalized log $L=\nu(L_0)$ (\ref{app:formal}): on a message event whose own-side endpoint the raw source $L_0$ left implicit, the retained \texttt{fromParticipant}/\texttt{toParticipant} value is the one the normalization backfills from \texttt{participant}, not a literal reproduction of an unrecorded source attribute.
\end{description}

\begin{example*}[\textbf{Mapping}]
We illustrate the mapping using the first three events of case~\texttt{case\_459} from the healthcare process, depicted in Table~\ref{table:relevantDataCase459}, which involves an interaction between a Patient and a Gynecologist. Figure~\ref{fig:example-ocel} shows the resulting object-centric instances and their relations. 
The mapping introduces two main structural changes. First, the Patient's send event and the Gynecologist's receive event are each represented by their own message interaction object, which preserves whichever of the sender and the receiver is defined for it in the normalized log---here both, the own side coming from \texttt{collab:participant}---without asserting a correspondence between the two interactions (rule~M4). Second, each event is related to the corresponding collaboration case through \emph{in\_collaboration}, to its orchestration case through \emph{in\_orchestration}, and directly to the respective participant through \emph{in\_participant}. Each orchestration case is, in turn, related to its collaboration case through \emph{part\_of} and to its participant through \emph{for\_participant}. In this way, the representation preserves the original collaboration-case trace through the order transported by $\mu_E$, while making the orchestration cases and messages explicit. Figure~\ref{fig:example-ocel} annotates each event with its source timestamp (rule M5) alongside the preserved event attributes of rule M8, and shows the object attributes fixed by rules M2--M4: \texttt{name} for participant objects, \texttt{caseId} and \texttt{participant} for orchestration-case objects, and \texttt{sender}/\texttt{receiver} for message objects.
\end{example*}

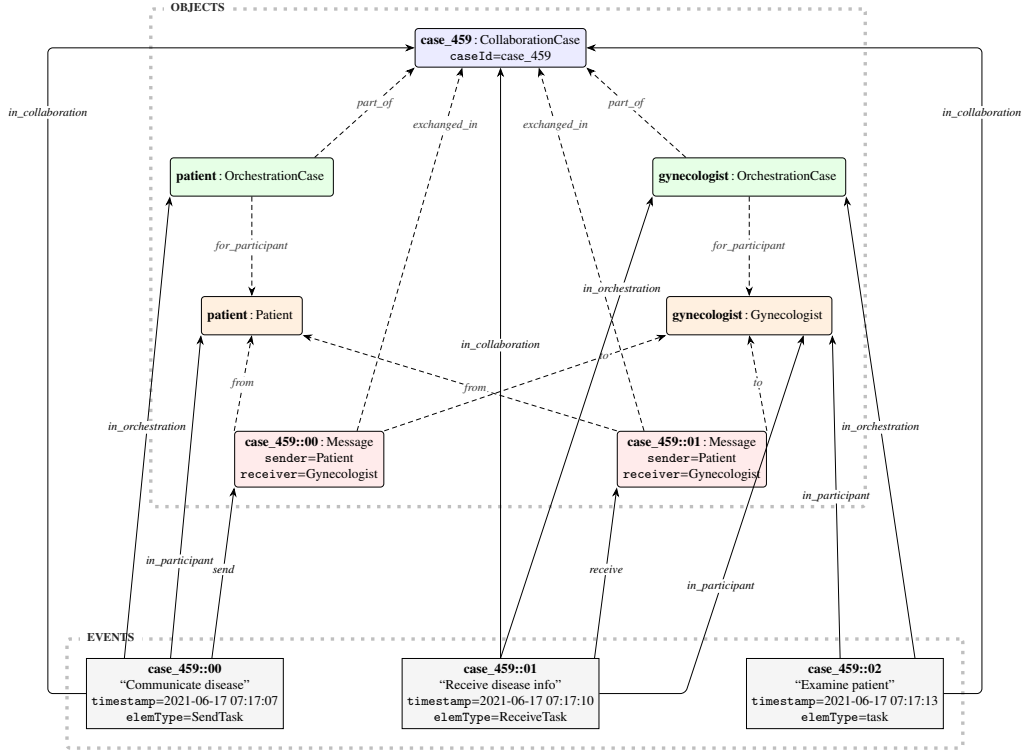
\begin{figure*}[!ht]
\centering
\resizebox{\linewidth}{!}{%
\begin{tikzpicture}[
  x=1mm,
  y=1mm,
  font=\footnotesize,
  >={Stealth[length=2mm]},
  obj/.style={
    draw,
    rounded corners=2pt,
    align=center,
    inner sep=3.5pt,
    minimum height=8mm,
    line width=0.5pt
  },
  cc/.style={obj, fill=blue!8},
  pa/.style={obj, fill=orange!12},
  oc/.style={obj, fill=green!10},
  msg/.style={obj, fill=red!8},
  ev/.style={
    draw,
    align=center,
    inner sep=3pt,
    fill=black!4,
    minimum height=7mm,
    line width=0.4pt
  },
  e2o/.style={
    ->,
    line width=0.55pt,
    rounded corners=3pt
  },
  o2o/.style={
    ->,
    densely dashed,
    line width=0.55pt,
    rounded corners=3pt
  },
  qlab/.style={
    inner sep=1pt,
    font=\scriptsize\itshape,
    fill=white
  },
  oqlab/.style={
    inner sep=1pt,
    font=\scriptsize\itshape,
    text=gray!45!black,
    fill=white
  }
]

% =========================================================
% OBJECT LAYER
% =========================================================
% Object identifiers are the ones the converter creates: cc::<case>,
% oc::<case>::<participant>, part::<participant>, msg::<event id>.

\node[cc] (cc) at (0,0) {
  \textbf{case\_459}\,:\,CollaborationCase\\
  \texttt{caseId}=case\_459
};

\node[oc] (ocp) at (-52,-27) {
  \textbf{patient}\,:\,OrchestrationCase
};

\node[oc] (ocg) at (52,-27) {
  \textbf{gynecologist}\,:\,OrchestrationCase
};

% Rule M2: the object type IS the participant identifier, so these two
% objects have distinct types, both elements of \TPa.
\node[pa] (pap) at (-52,-56) {
  \textbf{patient}\,:\,Patient
};

\node[pa] (pag) at (52,-56) {
  \textbf{gynecologist}\,:\,Gynecologist
};

% One Message object per communication event (M4); its identifier is
% derived from the identifier of the event that records the interaction.
\node[msg] (ms) at (-40,-86) {
  \textbf{case\_459::00}\,:\,Message\\
  \texttt{sender}=Patient\\
  \texttt{receiver}=Gynecologist
};

\node[msg] (mr) at (40,-86) {
  \textbf{case\_459::01}\,:\,Message\\
  \texttt{sender}=Patient\\
  \texttt{receiver}=Gynecologist
};

% =========================================================
% EVENT LAYER
% =========================================================
% Event identifiers embed the zero-padded rank of the event within its
% collaboration case (M5), so identifier order agrees with $\prec_L$.

\node[ev] (e1) at (-66,-135) {
  \textbf{case\_459::00}\\
  ``Communicate disease''\\
  \texttt{timestamp}=2021-06-17 07:17:07\\
  \texttt{elemType}=SendTask
};

\node[ev] (e2) at (0,-135) {
  \textbf{case\_459::01}\\
  ``Receive disease info''\\
  \texttt{timestamp}=2021-06-17 07:17:10\\
  \texttt{elemType}=ReceiveTask
};

\node[ev] (e3) at (72,-135) {
  \textbf{case\_459::02}\\
  ``Examine patient''\\
  \texttt{timestamp}=2021-06-17 07:17:13\\
  \texttt{elemType}=task
};

% =========================================================
% GROUP BOXES
% =========================================================

\begin{scope}[on background layer]
  \node[
    draw=gray!50,
    loosely dotted,
    rounded corners=2pt,
    line width=2pt,
    inner sep=4mm,
    fit=(cc)(ocp)(ocg)(pap)(pag)(ms)(mr)
  ] (objectsBox) {};

  \node[
    draw=gray!50,
    loosely dotted,
    rounded corners=2pt,
    line width=2pt,
    inner sep=4mm,
    fit=(e1)(e2)(e3)
  ] (eventsBox) {};
\end{scope}

\node[
  font=\scriptsize\bfseries,
  text=gray!30!black,
  fill=white,
  inner xsep=1.5mm,
  anchor=west
] at ([xshift=3mm]objectsBox.north west) {OBJECTS};

\node[
  font=\scriptsize\bfseries,
  text=gray!30!black,
  fill=white,
  inner xsep=1.5mm,
  anchor=west
] at ([xshift=3mm]eventsBox.north west) {EVENTS};

% =========================================================
% O2O RELATIONS
% =========================================================

\draw[o2o]
   ($(ocp.north east)+(-4mm,0mm)$)
  to node[oqlab, pos=0.6] {part\_of}
  (cc.south west);

\draw[o2o]
  ($(ocg.north west)+(7mm,0mm)$)
  to node[oqlab, pos=0.6] {part\_of}
  (cc.south east);

\draw[o2o]
  (ocp.south)
  to node[oqlab, pos=0.5] {for\_participant}
  (pap.north);

\draw[o2o]
  (ocg.south)
  to node[oqlab, pos=0.5] {for\_participant}
  (pag.north);

% Case membership of the two observations. Drawn before the endpoint
% relations so that the from/to labels mask these paths where they cross.
\draw[o2o]
  ($(ms.north)+(10mm,0mm)$)
  to node[oqlab, pos=0.84] {exchanged\_in}
  ($(cc.south)+(-8mm,0)$);

\draw[o2o]
  ($(mr.north)+(-10mm,0mm)$)
  to node[oqlab, pos=0.84] {exchanged\_in}
  ($(cc.south)+(8mm,0)$);

% Endpoint relations for the send observation msg::e::case_459::00.
\draw[o2o]
  (ms.north west)
  to node[oqlab, pos=0.5] {from}
  (pap.south);

\draw[o2o]
  (ms.north east)
  -- node[oqlab, pos=0.78] {to}
  (pag.south west);

% Endpoint relations for the receive observation msg::e::case_459::01.
\draw[o2o]
  (mr.north west)
  to node[oqlab, pos=0.45] {from}
  (pap.south east);

\draw[o2o]
  (mr.north east)
  to node[oqlab, pos=0.5] {to}
  (pag.south);

% =========================================================
% E2O RELATIONS: COMMUNICATION OBSERVATIONS
% =========================================================

\draw[e2o]
  ($(e1.north east)+(-15mm,0mm)$)
  to node[qlab, pos=0.52] {send}
  (ms.south west);

\draw[e2o]
  ($(e2.north east)+(-1mm,0mm)$)
  to node[qlab, pos=0.52] {receive}
  (mr.south west);

% =========================================================
% E2O RELATIONS: PARTICIPANT PROJECTIONS
% =========================================================

\draw[e2o]
  ($(e1.north west)+(8mm,0mm)$)
  to node[qlab, pos=0.5] {in\_orchestration}
  (ocp.south west);

\draw[e2o]
  (e2.north)
  -- node[qlab, pos=0.8] {in\_orchestration}
  (ocg.south west);

\draw[e2o]
  ($(e3.north east)+(-6mm,0mm)$)
  to node[qlab, pos=0.5] {in\_orchestration}
  (ocg.south east);

% =========================================================
% E2O RELATIONS: COLLABORATION CASE
% =========================================================

\draw[e2o]
  (e1.west)
  -- ++(-8,0)
  |- node[qlab, pos=0.45] {in\_collaboration}
  (cc.west);

\draw[e2o]
  (e2.north)
  -- node[qlab, pos=0.53] {in\_collaboration}
  (cc.south);

\draw[e2o]
  (e3.east)
  -- ++(8,0)
  |- node[qlab, pos=0.45] {in\_collaboration}
  (cc.east);

% =========================================================
% E2O RELATIONS: PARTICIPANTS
% =========================================================

\draw[e2o]
  ($(e1.north)+(-3mm,0mm)$)
  to node[qlab, pos=0.30] {in\_participant}
  (pap.south west);

\draw[e2o]
  (e2.east)
  -- ++(18,0)
  to node[qlab, pos=0.30] {in\_participant}
  ($(pag.south east)+(-6mm,0mm)$);

\draw[e2o]
  ($(e3.north)+(-1mm,0mm)$)
  to node[qlab, pos=0.5] {in\_participant}
  (pag.south east);

\end{tikzpicture}
}
\caption{Instance-level representation of the first three events of
collaboration case~\texttt{case\_459} (Table~\ref{table:relevantDataCase459}).
Dashed arrows represent qualified O2O relations, whereas solid arrows represent
qualified E2O relations. Some attributes are omitted for readability.}
\label{fig:example-ocel}
\end{figure*}

\subsection{Consistency criteria}\label{sec:cons-criteria}

The transformation ensures that six consistency criteria hold of its result, and they can be automatically checked. The criteria follow from the construction of the rules M1--M8; \ref{app:formal} provides the corresponding argument.
\begin{description}
  \item[P1.1 (Totality).] Every source event is mapped to exactly one event, preserving its timestamp, its event type, and every M8-preserved attribute value (decoded).
  \item[P1.2 (Per-case partition and order).] The set of events related to each $\mathit{cc}_c$ by the \emph{in\_collaboration} qualifier is exactly the image $\{\mu_E(e): \mathit{case}(e) = c\}$ of the events of case $c$; these sets partition the events of $\mu(L)$ as the cases partition those of $L$. Because the event identifier embeds the source order $\prec_L$ (M5), the events of $\mathit{cc}_c$ taken in identifier order are $\langle\mu_E(e_1),\ldots,\mu_E(e_{n_c})\rangle$; hence flattening on \CC{} and sorting by event identifier reconstructs the source trace $\sigma_c$ up to attribute encoding.
  \item[P1.3 (Message well-formedness).] Every \Msg{} object is related to exactly one communication event, either by a \emph{send} relation or by a \emph{receive} relation, but not both; conversely, every send (resp.\ receive) event is related, by \emph{send} (resp.\ \emph{receive}), to exactly one \Msg{} object. For each side of a \Msg{} object, the O2O relation (\emph{from} or \emph{to}), the object attribute (\texttt{sender} or \texttt{receiver}), and the preserved event attribute (\texttt{fromParticipant} or \texttt{toParticipant}) are defined simultaneously and agree whenever defined; a \Msg{} object whose counterparty endpoint the source never recorded simply lacks that relation and both attributes, rather than exposing a disagreement. For a send interaction, the event participant is the sender; for a receive interaction, the event participant is the receiver. No ordering or correspondence between distinct send and receive interactions is required. Moreover, every \Msg{} object is related, by exactly one \emph{exchanged\_in} relation, to the \CC{} object of the same collaboration case as its communication event.
  \item[P1.4 (OrchestrationCase coherence).] For every event, the associated \OC{} object is \emph{part\_of} the associated \CC{} object. These are, respectively, the objects reached through the event's \emph{in\_orchestration} and \emph{in\_collaboration} relations. Moreover, each \OC{} object is \emph{for\_participant} exactly one participant object, whose object type and \texttt{name} attribute encode the same participant identifier as its own \texttt{participant} attribute.
  \item[P1.5 (No orphan objects).] Every object participates in at least one E2O or O2O relation. For participant objects, this follows from how the participant set of the source log is defined: every participant either performs an event or is a \emph{from}/\emph{to} endpoint of some send or receive event, covered respectively by an E2O \emph{in\_participant} relation or an O2O \emph{from}/\emph{to} relation; no additional source-log precondition is needed.
  \item[P1.6 (Participant coherence).] For every event, the participant related to it through the \emph{in\_participant} qualifier coincides with the participant associated with it through \emph{for\_participant} with its \emph{in\_orchestration} object. Hence, the direct \emph{in\_participant} relation and the two-step \emph{in\_orchestration}--\emph{for\_participant} path identify the same participant.
\end{description}

% =====================================================================
\section{Object-Centric Definitions of Collaborative Prediction Tasks}
\label{sec:tasks}
% =====================================================================

In this section, we reformulate collaborative prediction tasks as object-centric prediction tasks over the object-centric representation introduced in Section~\ref{sec:mapping}. Throughout the reformulation, we deliberately keep \CC{} as the single scoping viewpoint for all fourteen tasks. This preserves the global execution prefixes of \cite{delgado2025predictcollab}, so the object-centric and source-level targets are evaluated over exactly the same observations.

We follow the notation of \ref{app:formal}: an extended collaborative event log (Definition~\ref{def:app-r1})
\begin{multline*}
L = (E_L, \mathcal{C}, \mathcal{P}, \mathcal{A}, X_L, V_L, \mathit{case}, \mathit{act}, \mathit{time}_L,
\mathit{part}, \mathit{elem}, \mathit{from}, \mathit{to}, D, \prec_0)
\end{multline*}
---the subscript $L$ marks a component as belonging to this one fixed log, e.g., $E_L$ is its event set and $\mathit{time}_L$ its timestamp function, not a family of distinct logs---and its mapping $\mu(L)$ to the single OCEL~2.0 log in the sense of \cite{berti2024ocel}. We write $\sigma_c = \langle e_1, \ldots, e_{n_c} \rangle$ for the trace of case $c$ and, as in \cite{delgado2025predictcollab}, $\mathit{hd}^k(\sigma_c) = \langle e_1, \ldots, e_k \rangle$ for its event prefix of length $k \in [1, n_c - 1]$. Since the prediction tasks are redefined over $\mu(L)$, they are stated over the corresponding sequence of $\mu(L)$ events, written $\bar{\sigma}_c$ and introduced in Section~\ref{sec:predictions}, rather than over $\sigma_c$ itself. We use $\bot$ as the designated outcome when no further target exists; this corresponds to the \emph{dummy} class of \cite{delgado2025predictcollab} for categorical tasks and to their sentinel constants for numeric ones. The Boolean outcome tasks (OB-P, OB-M) instead return $\mathsf{true}$ or $\mathsf{false}$, so $\bot$ never denotes a Boolean value. NV-PaT is the one numeric task that departs from this convention: following \cite{delgado2025predictcollab}, its fallback is the value $0$ rather than $\bot$ (\ref{app:tasks}).

\subsection{From global traces to viewpoints}\label{sec:tasks-viewpoint}

In the extended collaborative log, every prediction is defined over the global collaboration trace, and a target is scoped to a participant or to the messages by filtering the trace on event attributes \citep{delgado2025predictcollab}. In an object-centric log, a prediction is instead defined from a \emph{viewpoint}: the case notion is not fixed in advance but is chosen by selecting an object type, and the events of an execution are those related to one object of that type \citep{seidel2026manifesto}. We adopt \CC{} as the viewpoint for all tasks, allowing us to reconstruct the original execution, keeping the prefixes identical to those in \cite{delgado2025predictcollab}.

Each predicted label is stated over one category of entities, and that category need not reduce to a single existing object, nor coincide with the object type that scopes the prefix. 
We call this category the \emph{target anchor} of the task: the category of entities over which the label is denoted, counted, or quantified, not the object it identifies. The anchor is usually a single object type, but it need not be one, because rule~M2 gives each participant identifier its own type: labels about a participant are anchored on the whole family $\TPa$ of participant types rather than on any one of them. The anchor is \Msg{} for labels about messages (their kind, their number, or the time until the next one), $\TPa$ for labels about a participant, \OC{} for labels about a participant's local execution---including that execution's duration, NV-PaT---and \CC{} for labels about the global process, including its duration, NV-PrT and NV-TNE. The anchor is fixed by what the label denotes or ranges over, not by the events it quantifies: NE-NMPa and NV-NMPa restrict their quantification to one participant's orchestration case, yet they return a message kind and a message count, so both are anchored on \Msg.
Distinguishing the viewpoint from the anchor allows the same \CC-prefix to support labels for different objects without generating a separate source log for each target anchor; this viewpoint flexibility is identified as a distinctive trait of the object-centric setting in \cite{seidel2026manifesto}. In the object-centric predictive monitoring approaches reviewed in Section~\ref{sec:ppm}, the target coincides with the scoping object type, so the distinction does not arise. It arises here because the fourteen collaborative tasks state labels about participants, messages, orchestration cases, and the collaboration case itself under a single \CC-scoped prefix, rather than one prefix per target anchor. The anchor is a device for reading the reformulation rather than part of a task's definition: each task of \ref{app:tasks} is stated self-containedly, and neither those definitions nor Proposition~\ref{prop:p2-equivalence} refers to how a task is anchored.

This separation also clarifies what a prediction task is, independently of how it is computed. A task is the pair formed by the scoping viewpoint and a target function over the object-centric log; together, they fix the ground truth for every prefix of an execution as a property of the log alone. How a predictor converts the prefix into model inputs is a separate matter that affects the predictor's performance but not the target's definition. Accordingly, the definitions below are stated over $\mu(L)$ and the \CC{} viewpoint, and they do not refer to any encoding strategy.
Section~\ref{sec:impl-labels} describes how they are computed over the serialized log.

\subsection{Redefining predictions}\label{sec:predictions}

To define the labels, we first introduce well-defined accessor functions, as specified by the construction of the mapping (\ref{app:formal}). Let $\mathit{cc}_c$ be a \CC{} object and let
$\bar{\sigma}_c = \langle \varepsilon_1, \ldots, \varepsilon_{n_c} \rangle = \mu_E(\sigma_c)$
be the events related to $\mathit{cc}_c$ by the \emph{in\_collaboration} qualifier, ordered by event identifier $\prec_{\mathit{ev}}$ (\ref{app:formal}), i.e., taken in the per-case order $\prec_L$ of the source log (criterion P1.2). Because the accessors below read the objects and qualified relations of $\mu(L)$, they apply to the events $\varepsilon_i$ of $\bar{\sigma}_c$, not to the source events $e_i$ of $\sigma_c$; the prediction tasks are therefore stated over the object-centric prefix $\mathit{hd}^k(\bar{\sigma}_c) = \langle \varepsilon_1, \ldots, \varepsilon_k \rangle$. This is the same observation as $\mathit{hd}^k(\sigma_c)$---criterion P1.2 makes each prefix recoverable from the other---so the prefixes remain those of \cite{delgado2025predictcollab}, as claimed above.

For an event $\varepsilon \in E$ of $\mu(L)$, we define the following functions.
\begin{itemize}
  \item $\mathit{oc}(\varepsilon)$ and $\mathit{cc}(\varepsilon)$ are the unique \OC{} and \CC{} objects related to $\varepsilon$ by the \emph{in\_orchestration} and \emph{in\_collaboration} qualifiers, respectively (rule M6).
  \item $\mathit{pa}(\varepsilon)$ is the unique participant object directly related to $\varepsilon$ through the \emph{in\_participant} qualifier (rule M6). By criterion P1.6, it coincides with the participant reached through the two-step path \emph{in\_orchestration} followed by \emph{for\_participant}.
  \item $\mathit{snd}(\varepsilon)$ (resp.\ $\mathit{rcv}(\varepsilon)$) holds iff $(\varepsilon, \emph{send}, m) \in \mathit{E2O}$ (resp.\ $(\varepsilon, \emph{receive}, m) \in \mathit{E2O}$) for some $m$, and we write $\mathit{isMsg}(\varepsilon) = \mathit{snd}(\varepsilon) \vee \mathit{rcv}(\varepsilon)$. When $\mathit{isMsg}(\varepsilon)$ holds, $\mathit{msg}(\varepsilon)$ denotes that unique \Msg{} object (rule M6); it is undefined on events that are not message events.
  \item $\mathit{Msgs}(c) = \{\, m \in O : (m, \emph{exchanged\_in}, \mathit{cc}_c) \in \mathit{O2O} \,\}$ is the set of message objects of case $c$, and for $m \in \mathit{Msgs}(c)$, $\mathit{pos}(m)$ is the position in $\langle \varepsilon_1, \ldots, \varepsilon_{n_c} \rangle$ of the unique event related to $m$ by the \emph{send} or \emph{receive} qualifier.
\end{itemize}

Within these settings, \textbf{next-event predictions} take the prefix $\mathit{hd}^k(\bar{\sigma}_c)$ and predict a property of the next event. Over $\mu(L)$, the next event of the \CC{} prefix is $\varepsilon_{k+1}$, and the six next-event types of \cite{delgado2025predictcollab}, presented in Table~\ref{tab:predtypes}, differ in which property of that event, or of the next message event, is returned. 
Two modeling choices recur across the message-related types. First, when a task reads a message's endpoint participant, as NE-NPaM does, it does so via the \emph{from} (or \emph{to}) relation of the message object rather than via the participant of the send (or receive) event. Second, because the source logs lack a message-type attribute (rule~M4), the predicted message is identified by the event type of the send or receive event.
We present one representative type here and provide the formal definitions of all six next-event type predictions in \ref{app:tasks}. The \textbf{next-participant prediction (NE-NPaA)} is representative of this group: the predicted property of $\varepsilon_{k+1}$ is reached through the object structure of $\mu(L)$ rather than read from an event attribute, using the participant accessor $\mathit{pa}$ of Definition~\ref{def:accessors}.

\begin{definition*}[\textbf{Next-participant (NE-NPaA)}]
The next-participant prediction is the definition of a function $\Theta_{\mathrm{NPaA}}$ that takes the prefix $\mathit{hd}^k(\bar{\sigma}_c)$ and predicts the participant of the next event, that is,
\[
\Theta_{\mathrm{NPaA}}(\mathit{hd}^k(\bar{\sigma}_c)) = \mathit{pa}(\varepsilon_{k+1}).
\]
\end{definition*}

Concerning \textbf{numeric-value predictions}, they return a time or a count. Time-based types are defined from the preserved timestamps; count-based types are defined over the collaboration's message objects, exploiting the object identity introduced by the mapping. 
We present one representative type here and give the formal definitions of all six numeric value type predictions in \ref{app:tasks}.
The \textbf{time until next message (NV-TNM)} is representative of the time-based types.

\begin{definition*}[\textbf{Time until next message (NV-TNM)}]
The time-until-next-message prediction is the definition of a function $\Theta_{\mathrm{TNM}}^{d}$, parameterized by a message direction $d \in \{\mathit{snd}, \mathit{rcv}\}$, that takes the prefix $\mathit{hd}^k(\bar{\sigma}_c)$ and predicts the time from the current event to the next message event in direction $d$, that is,
\[
\Theta_{\mathrm{TNM}}^{d}(\mathit{hd}^k(\bar{\sigma}_c)) =
\begin{cases}
\mathit{time}(\varepsilon_j) \ominus \mathit{time}(\varepsilon_k) & \text{if } j = \min\{\, i : \substack{k < i \le n_c\\ d(\varepsilon_i)} \,\} \text{ exists,}\\
\bot & \text{otherwise.}
\end{cases}
\]
\end{definition*}
The two count-based types of prediction tasks, i.e., remaining messages of participant/process, are defined over the message objects $\mathit{Msgs}(c)$: a message object is counted as remaining at the cut-off $k$ when its communication event has not yet occurred, that is, when $\mathit{pos}(m) > k$. Because each send and each receive is a distinct message object (rule~M4), counting the objects of $\mathit{Msgs}(c)$ makes these predictions object-centric, as they quantify over the entities $\mathit{Msgs}(c)$ and promote an object-level quantity from a feature to a prediction target. 

Finally, \textbf{outcome-based predictions} return a Boolean indicating whether a participant will participate (\textbf{OB-P}) or whether a message will occur in the remainder of the execution (\textbf{OB-M}). We present one representative type here and give the formal definitions of both outcome-based type predictions in \ref{app:tasks}.

\begin{definition*}[\textbf{Participation of a participant (OB-P)}]
The prediction of whether a participant will participate is the definition of a function $\Theta_{\mathrm{OBP}}$ that takes the prefix $\mathit{hd}^k(\bar{\sigma}_c)$ and a participant $p \in \mathcal{P}_L$, and predicts whether $p$ acts after the cut-off, that is,
\[
\Theta_{\mathrm{OBP}}(\mathit{hd}^k(\bar{\sigma}_c), p) =
\begin{cases}
\mathsf{true} & \text{if } \exists\, i : k < i \le n_c \text{ and } \mathit{pa}(\varepsilon_i) = p_p,\\
\mathsf{false} & \text{otherwise.}
\end{cases}
\]
\end{definition*}

Equivalently, $\Theta_{\mathrm{OBP}}$ predicts whether at least one event beyond the cut-off is related to the orchestration case $\mathit{oc}_{c,p}$, which expresses the same outcome at the orchestration-case level; when $p$ never acts in $c$, that object is not materialized (rule M3) and no event is so related, so the outcome is $\mathsf{false}$ (\ref{app:tasks}). The companion type OB-M predicts whether a message of a given event kind occurs after the cut-off; as for the next-message types, the kind is identified by the event type of the communication event (rule~M4).

Beyond these fourteen tasks, the object-centric representation also makes it natural to state prediction targets that have no counterpart in the case-centric taxonomy, because they read relations between objects rather than the events of a single flattened case (Section~\ref{sec:disc-expressiveness}); the catalog reformulated here deliberately remains only on the fourteen types of \cite{delgado2025predictcollab}.

% ==============================================
\section{Implementation}
\label{sec:implementation}
% ==============================================

The representation in Section~\ref{sec:mapping} and the task definitions in Section~\ref{sec:tasks} are stated at a conceptual level, independent of any storage format or analysis tool. This section describes how both are realized: the serialization of the object-centric representation in OCEL~2.0, the computation of the task labels over the serialized log, and the pipeline that turns those labels into predictions. All code and materials are available at \cite{OCPPMcollab}. Figure~\ref{fig:pipelineframework} shows the pipeline at a high level.

\begin{figure}[htbp]
    \centering
    \includegraphics[width=\linewidth]{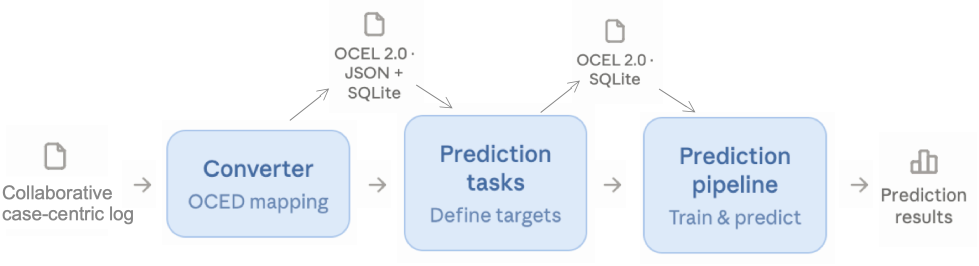}
    \caption{High level pipeline of the framework.}
    \label{fig:pipelineframework}
\end{figure}

\subsection{Serialization in OCEL~2.0}\label{sec:impl-serialization}

Since the OCED Core Model does not prescribe a serialization, we materialize its instances in OCEL~2.0. As noted in Section~\ref{sec:bg-ocpm}, OCEL~2.0 adds a per-type attribute schema and timestamped object attribute values on top of the core concepts. Neither carries meaning for the mapping nor imposes any constraint. Both therefore appear in the exported logs as serialization conventions rather than as part of the representation. Concretely, four elements of the exported logs are properties of the serialization rather than of the mapping: (i) event and object attributes are grouped per type; (ii) each static object attribute is encoded as a single value; (iii) each qualified relation is stored as a record carrying the qualifier string fixed by rules M6--M7, using the qualifier field OCEL~2.0 attaches to E2O and O2O relations; and (iv) the object type of a participant avoids the characters disallowed in type names.

The converter implements rules M1--M8 and checks criteria P1.1--P1.6; its check of P1.2 explicitly verifies that event-identifier order reproduces the source order, including the tie-break. Before export, an additional implementation-level check, P1.7, verifies that every object is reachable in the E2O table handed to PM4Py, preventing objects reachable only through O2O relations from being dropped by the writer; P1.7 is an export safeguard, not a seventh consistency criterion of the formal mapping. The converter exports each log in the two OCEL~2.0 serializations, JSON and SQLite, using PM4Py's OCEL~2.0 writers, and validates the JSON serialization against the OCEL~2.0 JSON schema. The SQLite serialization is the one consumed by the prediction pipeline of Section~\ref{sec:impl-pipeline}, which imports OCEL~2.0 directly. It is worth noting that this two-level organization keeps the representation independent of the analysis tool: the same core-level mapping could be serialized for and consumed by other OCED implementations.

\subsection{Computing the task labels}\label{sec:impl-labels}

The fourteen targets of Section~\ref{sec:tasks} are implemented as a standalone library that computes ground-truth labels and nothing else: it neither extracts features nor fits models, realizing in code the separation between a task and its predictor stated in Section~\ref{sec:tasks-viewpoint}. Three properties of this implementation follow directly from the definitions.

First, the accessors read the qualified relations of the serialized log. The functions $\mathit{oc}$, $\mathit{cc}$, $\mathit{pa}$, $\mathit{snd}$/$\mathit{rcv}$, $\mathit{msg}$, and $\mathit{Msgs}$ of Section~\ref{sec:predictions} are implemented as lookups over the E2O and O2O relations rather than over the event attributes that rule~M8 preserves. Every label is therefore computed from the object structure that the mapping introduces.

Second, the label functions are pure functions of the log. For a collaboration case and a cut-off $k$, the target is derived by looking forward from $\varepsilon_k$ over the remainder of the execution, with no model, encoding, or intermediate state involved, so the ground truth of a prefix depends on the log alone. Each parameterized task receives its parameter as an argument---a participant for NE-NMPa, NV-PaT, NV-NMPa, and OB-P; an activity for OB-M; a direction for NE-NPaM and NV-TNM---so that a task is realized as one function over its parameter.

Third, the vocabulary fixed by rules M1--M7, i.e., the object type names and the E2O and O2O qualifiers, is a parameter of the library, and the library depends on neither the converter nor the prediction pipeline. It can thus compute the fourteen targets for any object-centric log that carries the structure of Section~\ref{sec:mapping}, independently of the tool that produced it.

\subsection{Prediction pipeline}\label{sec:impl-pipeline}

To predict a target, the labels above must be paired with features extracted from the same prefixes. We use \texttt{ocpa}~$1.3.3$~\citep{adams2022ocpa}, which imports OCEL~2.0 natively and implements the object-centric feature extraction framework of \cite{adams2022framework}, as the basis for the structured inputs that each family of models requires: \texttt{ocpa} supplies the OCEL~2.0 importer, the event graph, and the tabular feature storage, while the per-case stacking of the sequential encoding and the ancestor subgraphs of the graph encoding are built on top of them by the pipeline. Table~\ref{tab:encodingsModels} summarizes these encodings, the models they support, and the libraries implementing them. These are the encoding--model pairs the pipeline implements; Section~\ref{sec:evaluation} evaluates a representative subset of them against the fourteen targets of Section~\ref{sec:tasks}.

\begin{widetable}
\setlength{\tabcolsep}{3pt}
\centering
\caption{Encoding approaches and compatible prediction models and libraries.}%
\label{tab:encodingsModels}
{\scriptsize\renewcommand{\arraystretch}{0.95}\setlength{\tabcolsep}{2pt}
\begin{tabular}{p{1.5cm} p{6.2cm} p{3.3cm} p{1.5cm}}
\toprule
\bf Encoding & \bf Input generation strategy & \bf ML/DL models & \bf Libraries \\
\midrule
Tabular  & Per-event vector, reading no object attribute: elapsed time; the number of distinct \OC{} and \Msg{} objects up to the cut; and one column per activity, counting its occurrences among the cut event's immediate predecessors. & Random Forest, XGBoost & scikit-learn \\
Sequential  & The same flat feature vectors as the tabular encoding, stacked in prefix order into per-case sequences.  & Transformers, LSTMs & PyTorch \\
Graph & Induced subgraphs of the OCPA event graph: the cut event plus its up to eight closest structural ancestor events, as interconnected nodes. & Graph Neural Networks (GNNs) & PyTorch \\
\bottomrule
\end{tabular}
}
\end{widetable}

Features and labels are produced by two separate derivation paths over the same imported log and are joined on the event identifier, which rule~M5 makes globally unique. The pipeline therefore never assumes that the feature extractor and the label functions traverse the log in the same way; it only assumes that both address the same events.

% ==================
\section{Empirical Validation}\label{sec:evaluation}
% ==================

The empirical validation is organized primarily around the first three research questions formulated in Section~\ref{sec:introduction}. We assess \textbf{RQ1} via the semantic correctness of the OCEL~2.0 transformation, \textbf{RQ2} via the label-equivalence proof (Proposition~\ref{prop:p2-equivalence}, under the source-level interpretations fixed there) and its empirical verification across execution prefixes, and \textbf{RQ3} via end-to-end execution of the reformulated tasks in a native object-centric predictive monitoring pipeline. RQ4 concerns the benefits and limitations of the representation and is discussed conceptually in Section~\ref{sec:discussion}. Code and materials are available at \cite{OCPPMcollab}.

% ==============================
\subsection{Datasets}\label{sec:eval:datasets}
% ==============================
We reuse the four extended collaborative event logs employed by \cite{delgado2025predictcollab}. The logs comprise both artificial and realistic collaboration scenarios and were selected for their variability in the number of participants and the number of messages exchanged. Reusing the logs as-is supports reproducibility and aligns the demonstration with the collaborative baseline at the source-data level. Table~\ref{table_logcaracteristics} summarizes descriptive statistics of the case-centric perspective of the logs (upper part).

\begin{widetable}
\setlength{\tabcolsep}{2pt}
\centering
\caption{Descriptive statistics of the event logs.}
\label{table_logcaracteristics}
{\scriptsize
  \begin{tabular}{lrrrr@{\hspace{10pt}}|@{\hspace{8pt}}r}
    \toprule
    \bf Statistic & \bf Healthcare & \bf Artificial~1 & \bf Artificial~5 & \bf Real~4 & \bf BPIC~2013 \\
    \midrule
    Cases & 100   & 100   & 100   & 100   & 7{,}554 \\
    Events                            & 1{,}450 & 800 & 2{,}360 & 1{,}800 & 69{,}584 \\
    Distinct activities               & 21    & 8     & 28    & 19    & 4 \\
    Max.\ events per case             & 18    & 8     & 25 & 18 & 135 \\
    Avg.\ events per case             & 14.5  & 8.0   & 23.6  & 18.0  & 9.21 \\
    Message types                     & 8     & 2     & 6     & 4     & 1 \\
    \midrule
    Distinct object types              & 7     & 5     & 7     & 6     & 28 \\
    Objects (total)                   & 1{,}528 & 702 & 1{,}450  & 1{,}203 & 25{,}527 \\
    \quad CollaborationCase       & 100   & 100   & 100   & 100   & 7{,}554 \\
    \quad OrchestrationCase                   & 350   & 200   & 330   & 300   & 9{,}846 \\
    \quad Participant                 & 4     & 2     & 4     & 3     & 25 \\
    \quad Message                     & 1{,}074   & 400   & 1{,}016   & 800   & 8{,}102 \\
    \midrule
    Event-to-object (E2O) relations   & 5{,}424 & 2{,}800 & 8{,}096 & 6{,}200 & 216{,}854 \\
    Object-to-object (O2O) relations  & 3{,}896 & 1{,}600 & 3{,}664 & 3{,}000 & 43{,}998 \\
    \bottomrule
  \end{tabular}
}
\end{widetable}

\noindent
The object-centric counts in the bottom part of Table~\ref{table_logcaracteristics} materialize the collaborative structure that the extended collaborative log leaves implicit. Note that the Message count includes both send and receive interactions (rule M4): a logical message exchange creates one Message object per observed side. The explicitness of this representation supports message-centered analysis but comes at the cost of a higher object count than the case-centric perspective. 
The same explicitness also duplicates information at the attribute level: the retained \texttt{participant}, \texttt{fromParticipant}, and \texttt{toParticipant} event attributes (rule M8) restate values already carried by relations. This retention is deliberate, keeping the mapping information-preserving relative to the normalized log $L=\nu(L_0)$ (\ref{app:formal}), not necessarily to the raw source $L_0$ where an own-side message endpoint was left implicit.

To complement these four logs with more complex real-world data, we also derive a fifth collaborative event log from the \emph{incidents} sub-log of the BPI Challenge 2013, which records the incident management process for the Volvo~IT VINST system~\citep{DBLP:conf/bpm/2013bpic}. The source log carries no native collaboration attributes, so obtaining it requires a domain reinterpretation and a preprocessing step external to the mapping itself: each IT organizational line recorded in the \texttt{organization involved} attribute is treated as a \emph{participant} and each incident as a \emph{collaboration instance}, and inter-participant hand-overs---only implicit in the source data as a change of responsible line between consecutive events---are recovered heuristically and flagged as such (a synthesized \texttt{SendTask} is paired with each detected \texttt{ReceiveTask}, declared via a residual \texttt{collab:synthesized} attribute rather than presented as directly observed data). Table~\ref{table_logcaracteristics} reports statistics alongside the four baseline logs (in the rightmost column).

% ==============================
\subsection{(RQ1) Semantic transformation to an object-centric representation}\label{sec:rq1experimentation}
% ==============================

The converter in Section~\ref{sec:impl-serialization} was executed on the five event logs described, providing \emph{semantic} validation via the mapping-specific consistency properties P1.1--P1.6 and \emph{structural} validation via conformance to the OCEL~2.0 JSON schema. Each check compares the constructed representation against an independently recomputed expectation from the source log, rather than only aggregate counts: P1.1 verifies the activity, timestamp, and residual (M8) attributes of every event individually, not only its structural attributes; P1.2 verifies, per case, that the independently recomputed set of mapped event identifiers matches exactly the one stored in the log, together with the \texttt{caseId} attribute of the corresponding \CC{} object, not merely its size; for P1.2, enumerating each case's events in event-identifier order reproduces $\prec_L$ exactly, including its tie-break by source order, as the order-preservation property of M5 requires; P1.4 and P1.6 verify, for every event, the totality of its \emph{in\_orchestration}/\emph{in\_participant} edges and the identity of the \OC{} object they resolve to, rather than only the agreement of whichever edges happen to be present; and P1.5 additionally verifies that no E2O/O2O relation targets an object that was never materialized. Consistency checks passed for every log under this stronger form. This is evidence that the implementation materializes the proposed mapping for the evaluated datasets. P1.3 additionally reports, without failing on it, the 26 (Healthcare) and 44 (Artificial5) messages whose counterparty endpoint the source never recorded. Table~\ref{table_logcaracteristics} summarizes the descriptive statistics of the resulting object-centric logs (bottom part).

% ==============================
\subsection{(RQ2) Verifying the label-preserving task reformulation}\label{sec:rq2experimentation}
% ==============================

RQ2 states a correctness question relative to the source-level readings fixed in \ref{app:tasks}. A collaboration case with a single event contributes no valid prefix ($k\in[1,n_c-1]$ is empty when $n_c=1$, \ref{app:formal}) and is therefore excluded from this and every other prediction task; this occurs once among the logs used, a single-event case in BPIC~2013. For the 14 reformulated tasks, target labels are independently computed for each valid execution prefix by two readers that follow different computation paths: R1 applies the source-level definitions $\Theta_\tau^{\mathit{src}}$ directly to the original collaborative log using the source accessors (Definition~\ref{def:app-r1}), without constructing an intermediate object-centric structure; R2 applies the object-centric definitions $\Theta_\tau$ over the generated OCEL~2.0 log using the accessors of Definition~\ref{def:accessors}. Their outputs are compared independently of any predictive model or encoding strategy. Categorical, Boolean, and count targets are compared for exact equality; temporal targets are compared both for exact equality and within a $1$-second tolerance declared in advance to absorb possible sub-second serialization slack.

As Proposition~\ref{prop:p2-equivalence} entails, R1 and R2 cannot disagree on a correct implementation: across the four baseline logs, 172 task--log--parameter combinations (279{,}400 evaluated target labels) had zero mismatches, and zero labels present in only one computation path were found. The same check was independently run at full scale on the reinterpreted, real-life BPI Challenge 2013 log: across 115 task--parameter combinations (7{,}133{,}450 evaluated target labels), again zero mismatches and zero one-sided labels were found. R1 and R2 are two independent code paths over the same accessors, not two independent implementations by different teams, so this result demonstrates internal consistency between the implementation and the stated formulas. The $1$-second tolerance played no role in this outcome: recomputing every temporal comparison at zero tolerance yields the same zero mismatches across all $287$ combinations, so no agreement reported here, temporal or otherwise, rests on a lenient numeric comparison that could mask a real disagreement.

% ==============================
\subsection{(RQ3) End-to-end feasibility on a native object-centric pipeline}\label{sec:rq3experimentation}
% ==============================

The experiments addressing RQ3 assess the feasibility of executing the reformulated tasks within the native object-centric PPM pipeline of Section~\ref{sec:impl-pipeline}. The evaluation considers the complete workflow---OCEL loading, viewpoint selection, prefix construction, feature extraction, target generation, model training, and prediction---and reports task coverage and execution outcomes together with descriptive predictive metrics. Partitioning is performed at the collaboration-case level and wrapped in grouped $5$-fold cross-validation using \texttt{CollaborationCase}; all prefixes of a collaboration case stay in the same fold to prevent information leakage, and features are computed only from the observable prefix. For NE-NPaM, NE-NMPr, NE-NMPa, and NV-TNM, prefixes whose target is the sentinel $\bot$ (\ref{app:tasks}) are excluded from training and evaluation throughout, since a model trained to output $\bot$ would conflate the substantive classes of the task with the unrelated condition that no further occurrence remains in the trace; the "number of valid labeled prefixes" reported in Table~\ref{tab:resultsAllSubset} already reflects this exclusion. The restriction is substantial on the largest log: on BPIC~2013 it removes between $70.2\%$ and $76.6\%$ of the cut points for these four tasks, and $53.0\%$ for NE-NPaM and NV-TNM on Artificial1.

% ==============================
\subsubsection{Selected models and hyperparameters for OCPM Prediction}\label{sec:eval:setup}
% ==============================

Of the encodings and models supported by the pipeline (Table~\ref{tab:encodingsModels}), we selected two models for tabular encodings: Random Forest and XGBoost, two models for sequential encodings: Transformer and LSTM, and one for the graph encoding: GNN, based on use and reported capabilities in process prediction literature. They are used to compare the execution of the reformulated tasks across representative encoding families. Although we cover different encoding approaches and corresponding prediction models, the objective is to determine whether the proposed definitions can be implemented reproducibly in practice, rather than to compare predictors or claim predictive superiority. In particular, applying the Transformer architecture to \texttt{ocpa}'s object-centric feature set stacked into per-case sequences (Table~\ref{tab:encodingsModels}), as opposed to a flattened single-object-type sequence \citep{gherissi2025embeddings} or a graph-based Transformer variant \citep{gherissi2025framework}, has, to our knowledge, not been previously reported (Section~\ref{sec:bg-ocpm}).

% ==============================
%\subsection{Experimental setup and hyperparameters}\label{sec:eval:setup}
% ==============================

All experiments reported here were executed on a single machine: Intel Core i5, 10~physical / 16 logical cores at 2.5 GHz base, 15.8 GB RAM, Windows 10. The three PyTorch-based predictors (LSTM, Transformer, GNN) were trained on CPU; a CUDA device was available and visible (NVIDIA GTX~1660, 6~GB) but was not used for the reported experiments.

All predictors use a single fixed hyperparameter configuration across logs, tasks, and folds, with a shared \texttt{random\_state}=3395 controlling stochastic components, including a seeded permutation of \CC\ assigned by rank modulo~5 that fixes case-to-fold assignment, persisted and reused across all tasks and predictors. RandomForest and XGBoost use conventional settings for tabular datasets of this scale, while LSTM and Transformer were calibrated on a small, non-exhaustive set of representative classification and regression tasks after an initial 20-epoch configuration underperformed. The final budgets are 60 epochs for LSTM, 35 for Transformer, and 50 for GNN, with fixed batch sizes and no early stopping or validation-based checkpoint selection. GNN additionally uses a hidden dimension of 64 and subgraphs capped at ($k=8$) events. The regression objective is not uniform across predictors: RandomForest, XGBoost, LSTM and Transformer minimize squared error, the latter two on a standardized target, while GNN minimizes a Huber objective ($\delta$=1) on the unstandardized target, which at these magnitudes is effectively absolute error. Section~\ref{sec:eval:summary} shows that this asymmetry, not the encoding, accounts for the temporal-target pattern on BPIC~2013.

\subsubsection{Experimentation results \& analysis}\label{sec:eval:summary}

The pipeline was executed on the full catalog of 14 task types, with single-case counterparts across all five logs. Unlike RQ2, which evaluates every applicable parameter value (Section~\ref{sec:rq2experimentation}), RQ3 instantiates each parameterized type with a single value per log, fixed as described next; the "14 tasks" of this section thus denotes 14 types, not the full parameter space. All 70 task-log combinations (14 tasks $\times$ 5 logs) completed the five folds end-to-end for each of the five predictors (350 task-log-predictor combinations in total). For the five parameterized task types, the target parameter was resolved per log by selecting the most frequent participant among all event actors for NE-NMPa, NV-PaT, NV-NMPa, and OB-P, and the most frequent message activity among all message-related activities for OB-M.
To provide a compact yet heterogeneous demonstration, we report a representative subset of six tasks that cover four target anchors and four problem types (Table~\ref{tab:rq3subset}).
These six tasks jointly exercise every target anchor and problem type introduced by the reformulation. Where the full catalog realizes the same anchor--problem-type combination more than once, the retained task is the one that offers a more informative descriptive contrast against its trivial baseline; this selection is descriptive and post hoc, made after inspecting the results, so this subset should not be read as independent evidence beyond what the full catalog in \ref{app:fullcatalog} already reports. Complete results are available in \ref{app:fullcatalog} and in the accompanying repository~\cite{OCPPMcollab}.

\begin{widetable}
\setlength{\tabcolsep}{4pt}
\centering
\caption{Representative subset for the end-to-end demonstration (RQ3).}%
\label{tab:rq3subset}
{\footnotesize
\begin{tabular}{lll}
\toprule
Task & Target anchor & Problem type \\
\midrule
NE-NPaA  & \texttt{Participant} & Multiclass classification \\
NE-NMPr  & \texttt{Message} & Multiclass classification \\
NV-PrT   & \texttt{CollaborationCase} & Regression (time) \\
NV-PaT   & \texttt{OrchestrationCase} & Regression (time) \\
NV-NMPr  & \texttt{Message} & Count regression \\
OB-M     & \texttt{Message} & Binary classification \\
\bottomrule
\end{tabular}
}
\end{widetable}

A trivial baseline (the majority class for classification and the median for regression) is included as a sanity check, not as a competing method. Reported metrics are $\mathrm{F1}_{\mathrm{macro}}$ for classification tasks and MAE for time and count regression tasks, shown as mean\,$\pm$\, standard deviation across the five-fold scores. 
Results are reported in Table \ref{tab:resultsAllSubset}, for each prediction task and log per model used, the target anchor, the number of valid labeled prefixes, and the descriptive metric against its trivial baseline. 

\begin{widesidewaystable}
\setlength{\tabcolsep}{1.5pt}
\centering
\caption{Experimentation results. Times are in seconds, except BPIC~2013 (\textsuperscript{*}), in days. Under two-decimal rounding: \textsuperscript{\dag}~at or near a perfect score or zero error; \textsuperscript{\ddag}~constant (degenerate) target, so the trivial baseline itself rounds to a perfect score; \textsuperscript{\S}~metric worse than that baseline. Per-row extremes at the same precision: best in \textbf{bold}, worst \underline{underlined}. These are not a ranking---models may tie, and on half of the rows the two marked models overlap within one standard deviation across folds.
}%
\label{tab:resultsAllSubset}
{\scriptsize\renewcommand{\baselinestretch}{1}\selectfont
\renewcommand{\arraystretch}{0.95}
\begin{tabular}{ll p{4.45cm} r @{\hspace{6pt}} l l l l l l}
\toprule
Task & Log & Anchor (parameter) & Samples & Baseline & RF & XGB & Transf. & LSTM & GNN \\
\midrule
\multirow{5}{*}{\begin{tabular}[c]{@{}l@{}}NE-NPaA\\ ($\mathrm{F1}_{\mathrm{macro}}$)\end{tabular}}
 & Healthcare & \texttt{Participant} & 1{,}350 & 0.17 & \underline{0.82}\,$\pm$\,0.03 & 0.86\,$\pm$\,0.02 & 0.87\,$\pm$\,0.02 & \textbf{0.88}\,$\pm$\,0.02 & 0.87\,$\pm$\,0.02 \\
 & Artificial1 & \texttt{Participant} & 700 & 0.34 & \underline{0.73}\,$\pm$\,0.04 & 0.74\,$\pm$\,0.02 & \textbf{0.79}\,$\pm$\,0.04 & 0.78\,$\pm$\,0.04 & 0.78\,$\pm$\,0.04 \\
 & Artificial5 & \texttt{Participant} & 2{,}260 & 0.18 & \underline{0.56}\,$\pm$\,0.01 & \textbf{0.59}\,$\pm$\,0.03 & 0.57\,$\pm$\,0.01 & 0.58\,$\pm$\,0.03 & \underline{0.56}\,$\pm$\,0.03 \\
 & Real4 & \texttt{Participant} & 1{,}700 & 0.23 & \underline{0.86}\,$\pm$\,0.02 & 0.89\,$\pm$\,0.01 & 0.90\,$\pm$\,0.02 & \textbf{0.91}\,$\pm$\,0.02 & 0.90\,$\pm$\,0.02 \\
 & BPIC~2013 & \texttt{Participant} & 62{,}030 & 0.04 & 0.07\,$\pm$\,0.00 & \underline{0.05}\,$\pm$\,0.00 & \textbf{0.12}\,$\pm$\,0.01 & \textbf{0.12}\,$\pm$\,0.02 & 0.09\,$\pm$\,0.01 \\
\midrule
\multirow{5}{*}{\begin{tabular}[c]{@{}l@{}}NE-NMPr\\ ($\mathrm{F1}_{\mathrm{macro}}$)\end{tabular}}
 & Healthcare & \texttt{Message} & 1{,}350 & 0.02 & \underline{0.70}\,$\pm$\,0.02 & \textbf{0.72}\,$\pm$\,0.03 & \textbf{0.72}\,$\pm$\,0.03 & \textbf{0.72}\,$\pm$\,0.06 & \underline{0.70}\,$\pm$\,0.06 \\
 & Artificial1 & \texttt{Message} & 549 & 0.11 & \textbf{0.58}\,$\pm$\,0.05 & 0.54\,$\pm$\,0.04 & \underline{0.52}\,$\pm$\,0.04 & \underline{0.52}\,$\pm$\,0.02 & 0.53\,$\pm$\,0.04 \\
 & Artificial5 & \texttt{Message} & 2{,}068 & 0.02 & \underline{0.64}\,$\pm$\,0.02 & 0.66\,$\pm$\,0.02 & 0.66\,$\pm$\,0.03 & \underline{0.64}\,$\pm$\,0.03 & \textbf{0.67}\,$\pm$\,0.02 \\
 & Real4 & \texttt{Message} & 1{,}600 & 0.07 & \textbf{1.00}\,$\pm$\,0.00\,\textsuperscript{\dag} & \textbf{1.00}\,$\pm$\,0.00\,\textsuperscript{\dag} & 0.99\,$\pm$\,0.00 & \underline{0.98}\,$\pm$\,0.01 & 0.99\,$\pm$\,0.00 \\
 & BPIC~2013 & \texttt{Message} & 18{,}481 & 1.00\,\textsuperscript{\ddag} & \textbf{1.00}\,$\pm$\,0.00\,\textsuperscript{\dag} & \textbf{1.00}\,$\pm$\,0.00\,\textsuperscript{\dag} & \textbf{1.00}\,$\pm$\,0.00\,\textsuperscript{\dag} & \textbf{1.00}\,$\pm$\,0.00\,\textsuperscript{\dag} & \textbf{1.00}\,$\pm$\,0.00\,\textsuperscript{\dag} \\
\midrule
\multirow{5}{*}{\begin{tabular}[c]{@{}l@{}}NV-PrT\\ (MAE)\end{tabular}}
 & Healthcare & \texttt{CollaborationCase} & 1{,}350 & 17.45 & 15.29\,$\pm$\,1.79 & 15.07\,$\pm$\,2.07 & \textbf{14.76}\,$\pm$\,1.52 & 14.78\,$\pm$\,2.14 & \underline{15.46}\,$\pm$\,2.46 \\
 & Artificial1 & \texttt{CollaborationCase} & 700 & 8.59 & \underline{4.65}\,$\pm$\,0.68 & 4.40\,$\pm$\,0.68 & 3.99\,$\pm$\,0.74 & \textbf{3.98}\,$\pm$\,0.67 & 4.05\,$\pm$\,0.76 \\
 & Artificial5 & \texttt{CollaborationCase} & 2{,}260 & 33.57 & \underline{11.15}\,$\pm$\,0.64 & 10.70\,$\pm$\,0.74 & \textbf{10.07}\,$\pm$\,0.98 & 10.53\,$\pm$\,0.78 & 10.41\,$\pm$\,1.05 \\
 & Real4 & \texttt{CollaborationCase} & 1{,}700 & 19.43 & \underline{7.62}\,$\pm$\,0.41 & 7.05\,$\pm$\,0.55 & 7.03\,$\pm$\,0.40 & \textbf{6.61}\,$\pm$\,0.46 & 6.67\,$\pm$\,0.65 \\
 & BPIC~2013\textsuperscript{*} & \texttt{CollaborationCase} & 62{,}030 & 13.39 & \underline{18.79}\,$\pm$\,1.75\,\textsuperscript{\S} & 16.73\,$\pm$\,1.80\,\textsuperscript{\S} & 14.60\,$\pm$\,1.81\,\textsuperscript{\S} & 16.37\,$\pm$\,2.20\,\textsuperscript{\S} & \textbf{13.37}\,$\pm$\,1.98 \\
\midrule
\multirow{5}{*}{\begin{tabular}[c]{@{}l@{}}NV-PaT\\ (MAE)\end{tabular}}
 & Healthcare & \begin{tabular}[c]{@{}l@{}}\texttt{OrchestrationCase}\\(Gynecologist)\end{tabular} & 1{,}350 & 12.85 & \underline{9.29}\,$\pm$\,0.72 & 9.01\,$\pm$\,0.92 & 8.40\,$\pm$\,0.75 & 8.33\,$\pm$\,1.13 & \textbf{8.32}\,$\pm$\,1.07 \\
 & Artificial1 & \begin{tabular}[c]{@{}l@{}}\texttt{OrchestrationCase}\\(PartyA)\end{tabular} & 700 & 8.57 & \underline{4.59}\,$\pm$\,0.72 & 4.38\,$\pm$\,0.71 & 3.95\,$\pm$\,0.73 & \textbf{3.93}\,$\pm$\,0.72 & \textbf{3.93}\,$\pm$\,0.75 \\
 & Artificial5 & \begin{tabular}[c]{@{}l@{}}\texttt{OrchestrationCase}\\(PartyB)\end{tabular} & 2{,}260 & 33.50 & \underline{10.59}\,$\pm$\,0.62 & 10.08\,$\pm$\,0.73 & 9.50\,$\pm$\,0.90 & 9.64\,$\pm$\,0.60 & \textbf{9.47}\,$\pm$\,1.00 \\
 & Real4 & \begin{tabular}[c]{@{}l@{}}\texttt{OrchestrationCase}\\(Zoo)\end{tabular} & 1{,}700 & 19.00 & \underline{6.79}\,$\pm$\,0.19 & 6.36\,$\pm$\,0.19 & 6.51\,$\pm$\,0.42 & \textbf{5.67}\,$\pm$\,0.41 & 5.81\,$\pm$\,0.37 \\
 & BPIC~2013\textsuperscript{*} & \begin{tabular}[c]{@{}l@{}}\texttt{OrchestrationCase}\\(Org line C)\end{tabular} & 62{,}030 & 7.86 & \underline{9.56}\,$\pm$\,0.85\,\textsuperscript{\S} & 8.53\,$\pm$\,0.87\,\textsuperscript{\S} & 7.97\,$\pm$\,1.06\,\textsuperscript{\S} & 8.55\,$\pm$\,0.83\,\textsuperscript{\S} & \textbf{7.37}\,$\pm$\,0.92 \\
\midrule
\multirow{5}{*}{\begin{tabular}[c]{@{}l@{}}NV-NMPr\\ (MAE)\end{tabular}}
 & Healthcare & \texttt{Message} & 1{,}350 & 2.82 & 2.23\,$\pm$\,0.18 & 2.23\,$\pm$\,0.20 & \textbf{2.19}\,$\pm$\,0.11 & 2.20\,$\pm$\,0.22 & \underline{2.36}\,$\pm$\,0.37 \\
 & Artificial1 & \texttt{Message} & 700 & 1.29 & \textbf{0.00}\,$\pm$\,0.00\,\textsuperscript{\dag} & 0.01\,$\pm$\,0.00 & \underline{0.18}\,$\pm$\,0.01 & 0.07\,$\pm$\,0.01 & 0.07\,$\pm$\,0.01 \\
 & Artificial5 & \texttt{Message} & 2{,}260 & 2.61 & 0.81\,$\pm$\,0.07 & \textbf{0.80}\,$\pm$\,0.06 & \underline{0.95}\,$\pm$\,0.03 & 0.94\,$\pm$\,0.03 & 0.91\,$\pm$\,0.05 \\
 & Real4 & \texttt{Message} & 1{,}700 & 1.76 & \textbf{0.00}\,$\pm$\,0.00\,\textsuperscript{\dag} & \textbf{0.00}\,$\pm$\,0.00\,\textsuperscript{\dag} & \underline{0.25}\,$\pm$\,0.01 & 0.08\,$\pm$\,0.00 & 0.20\,$\pm$\,0.01 \\
 & BPIC~2013 & \texttt{Message} & 62{,}030 & 1.39 & \underline{1.84}\,$\pm$\,0.11\,\textsuperscript{\S} & 1.72\,$\pm$\,0.11\,\textsuperscript{\S} & 1.45\,$\pm$\,0.12\,\textsuperscript{\S} & 1.63\,$\pm$\,0.13\,\textsuperscript{\S} & \textbf{1.42}\,$\pm$\,0.13\,\textsuperscript{\S} \\
\midrule
\multirow{5}{*}{\begin{tabular}[c]{@{}l@{}}OB-M\\ ($\mathrm{F1}_{\mathrm{macro}}$)\end{tabular}}
 & Healthcare & \begin{tabular}[c]{@{}l@{}}\texttt{Message}\\(Communicate disease)\end{tabular} & 1{,}350 & 1.00\,\textsuperscript{\ddag} & \textbf{1.00}\,$\pm$\,0.00\,\textsuperscript{\dag} & \textbf{1.00}\,$\pm$\,0.00\,\textsuperscript{\dag} & \textbf{1.00}\,$\pm$\,0.00\,\textsuperscript{\dag} & \textbf{1.00}\,$\pm$\,0.00\,\textsuperscript{\dag} & \textbf{1.00}\,$\pm$\,0.00\,\textsuperscript{\dag} \\
 & Artificial1 & \begin{tabular}[c]{@{}l@{}}\texttt{Message}\\(Activity CZ)\end{tabular} & 700 & 0.39 & \textbf{0.92}\,$\pm$\,0.02 & \textbf{0.92}\,$\pm$\,0.02 & \textbf{0.92}\,$\pm$\,0.00 & \underline{0.91}\,$\pm$\,0.01 & \textbf{0.92}\,$\pm$\,0.01 \\
 & Artificial5 & \begin{tabular}[c]{@{}l@{}}\texttt{Message}\\(Activity AA)\end{tabular} & 2{,}260 & 0.49 & \textbf{1.00}\,$\pm$\,0.00\,\textsuperscript{\dag} & \textbf{1.00}\,$\pm$\,0.00\,\textsuperscript{\dag} & 0.96\,$\pm$\,0.08 & \underline{0.91}\,$\pm$\,0.17 & \textbf{1.00}\,$\pm$\,0.00\,\textsuperscript{\dag} \\
 & Real4 & \begin{tabular}[c]{@{}l@{}}\texttt{Message}\\(Send info. to the ZooClub dept.)\end{tabular} & 1{,}700 & 0.47 & \textbf{1.00}\,$\pm$\,0.00\,\textsuperscript{\dag} & \textbf{1.00}\,$\pm$\,0.00\,\textsuperscript{\dag} & \textbf{1.00}\,$\pm$\,0.00\,\textsuperscript{\dag} & \textbf{1.00}\,$\pm$\,0.00\,\textsuperscript{\dag} & \textbf{1.00}\,$\pm$\,0.00\,\textsuperscript{\dag} \\
 & BPIC~2013 & \begin{tabular}[c]{@{}l@{}}\texttt{Message}\\(Queued)\end{tabular} & 62{,}030 & 0.41 & 0.65\,$\pm$\,0.00 & \textbf{0.71}\,$\pm$\,0.01 & \underline{0.64}\,$\pm$\,0.02 & \textbf{0.71}\,$\pm$\,0.01 & \textbf{0.71}\,$\pm$\,0.01 \\
\bottomrule
\end{tabular}\\[2pt]
}
\end{widesidewaystable}

Overall, the results show that the object-centric predictive paradigm successfully addresses the complexities of collaborative business processes. In line with the state of the art in the OCPM literature, the native representation of event-to-object (E2O) and object-to-object (O2O) relations enables holistic monitoring of the collaborative environment, in which object-centric predictions effectively capture the concurrent realities of multiple participants interacting simultaneously.

Beyond executability, we read the descriptive metrics as evidence of feature informativeness only where the target is both variable and non-deterministic. Cells reaching a perfect score or near-zero error are treated as deterministic reconstructions, where the observable prefix almost fully determines the label. Two cells of the reported subset carry a constant target and are degenerate (\textsuperscript{\ddag}): OB-M on Healthcare and NE-NMPr on BPIC~2013; over the full 14-task catalog, there are four such cells, the other two being NE-NPaM on Artificial1 and NE-NMPa on BPIC~2013. Degeneracy is thus a property of the (log, task) pair rather than of the task, and no task is degenerate on both log groups. By the same criterion, the cells in which a model's metric is worse than its trivial baseline (marked~\textsuperscript{\S}) provide no evidence of feature informativeness under the fixed-hyperparameter pipeline used here. There are 13 such cells in the reported subset, and 26 of the 30 non-degenerate numeric cells on BPIC~2013 over the full 14-task catalog. Across the full catalog, 31 of the 330 non-degenerate task--log--predictor entries are worse than their trivial baseline: those 26 numeric entries on BPIC~2013, four on the baseline logs (all of them on NV-TNE), and a single classification entry. No classification entry on BPIC~2013 falls below its baseline. Unlike the four baseline logs, on this real, order-of-magnitude larger log, the untuned pipeline does not uniformly outperform the trivial baseline: four of the five predictors are worse than it on every numeric target, and the fifth (GNN) is better on four of the six, by margins under $0.6\%$ on three of them. All these observations remain descriptive and do not constitute an algorithmic comparison.

The evaluation confirms that structural and categorical predictions—such as identifying the Next participant to act (NE-NPaA) or anticipating the occurrence of a message (OB-M)—yield robust results in most object-centric environments. Temporal targets behave differently, but the pattern follows from the training objectives rather than from the encodings. The reported metric is MAE and the trivial baseline is the training-fold median, the MAE-optimal constant, whereas the optimal constant for squared error is the mean. A predictor that extracts no signal falls back to its own optimal constant, and on BPIC~2013 the mean constant costs 1.25--1.60$\times$ the median constant in MAE across the six numeric targets; a squared-error learner therefore lands above the baseline, and an absolute-error learner ties it, by construction. The measurements match: the GNN is within 0.15--0.54\% of the median constant on NV-PrT, NV-TNE and NV-TNM, and XGBoost within 0.23--0.53\% of the mean constant on NV-PrT and NV-PaT. No predictor extracts temporal signal from this log, the single exception being the GNN on NV-PaT ($-6.3$\% against its baseline); the apparent contrast between flattened and graph-native encodings on temporal targets is an artifact of the objectives and is not evidence about the representation. What the predictors have to work with is also narrow: the \texttt{ocpa} feature table is seven columns wide on BPIC~2013 (Table~\ref{tab:encodingsModels}; the per-activity columns number four on this log), so the absence of temporal signal is a statement about the available features rather than about the object-centric representation.

Additionally, comparing results for equivalent predictive tasks between object-centric logs and case-centric collaborative logs from Predict-Collab \cite{delgado2025predictcollab} yields several descriptive observations on the trade-offs of both paradigms; architecture, preprocessing, and folds are not harmonized between the two studies, so none of these observations constitutes a controlled comparison. Notably, Predict-Collab relied exclusively on the Process Transformer architecture. Therefore, our primary baseline for comparison is the object-centric Transformer model, while also highlighting how other architectures (such as tabular or graph-based models) behave in this new environment.

Next participant to act (NE-NPaA): The object-centric approach reports a higher score classifying the next acting participant than the case-centric baseline. While collaborative case-centric logs capture temporal parallelism by interleaving concurrent activities in chronological order, they force the model to implicitly infer structural dependencies among distributed participants. In contrast, object-centric sequential features explicitly encode the multi-object bindings at each step. This means that object-centric models (such as LSTMs and Transformers) do not merely process a timeline of events; they explicitly receive structural context indicating which objects are correlated in the current interaction. This pattern is consistent with, though not proof of, explicit structural binding aiding participant prediction.

Next message in the process (NE-NMPr): Conversely, for predicting the next message interaction, the flattened case-centric approach reports a better score on the metrics. This discrepancy highlights a fundamental difference in the nature of the target variables. While participant activities often involve complex, concurrent states that benefit from structural context, collaborative message exchanges typically follow strictly linear, protocol-driven sequences (e.g., a rigid request-response communication). The case-centric Process Transformer excels at memorizing these 1D sequences from a flattened trace, much like next-word prediction.

Remaining Time Predictions (NV-PrT \& NV-PaT): When estimating the remaining time for both the overall process and individual participants, the performance between the object-centric approach and Predict-Collab is similar in magnitude. The object-centric Transformer and LSTM models exhibited error margins close to those of the Predict-Collab Process Transformer. The object-centric approach nonetheless avoids constructing separate per-orchestration-case traces, since a single OCEL log covers every participant; each predictor is still fit to one task and parameter at a time, not jointly across participants.

Predictions Without Direct Equivalents (NV-NMPr \& OB-M): Finally, the object-centric evaluation included predicting the number of remaining messages (NV-NMPr) and forecasting the occurrence of a target message (OB-M), neither of which has a direct numerical counterpart in the Predict-Collab experimental results. On the four baseline logs, the object-centric models scored well on their respective metrics (low MAE for NV-NMPr, high macro-F1 for OB-M), with the tabular models among the strongest; several of those entries reach a perfect score or near-zero error and are therefore read as deterministic reconstructions rather than as evidence of feature informativeness. On BPIC~2013 the picture differs: all five predictors are worse than the trivial baseline on NV-NMPr, Random Forest by the widest margin, while OB-M stays above its baseline for all five.

% ==============================
\subsubsection{Execution cost}\label{sec:eval:cost}
% ==============================

Every run was additionally executed with per-stage profiling enabled, recording wall-clock time and peak resident memory for each of the five pipeline stages (OCEL loading, feature extraction, target generation, model fitting, prediction) on every fold. Because feasibility is the claim under test, the cost of obtaining these results is part of the result, but not taken as an algorithmic comparison of predictors

Results for the full task catalog are reported in Table~\ref{tab:rq3cost}. All predictors are evaluated on the same population of non-degenerate labeled prefixes, including the cells that the sequence models would otherwise short-circuit. \emph{Fit} reports total fitting time, \emph{s/1k} normalizes the mean per-fold fitting time by workload size, \emph{Predict} reports mean per-fold wall-clock inference time, and \emph{Peak} records the maximum resident set size observed during the pipeline.

\begin{widetable}
\centering
\caption{RQ3 execution cost over the full 14-task catalog, by predictor and log group.}%
\label{tab:rq3cost}
{\scriptsize \setlength{\tabcolsep}{3pt}
\begin{tabular}{l rrrr rrrr}
\toprule
 & \multicolumn{4}{c}{PredictCollab logs (78{,}600 fitted prefixes)} & \multicolumn{4}{c}{BPIC~2013 (653{,}298 fitted prefixes)} \\
\cmidrule(lr){2-5}\cmidrule(lr){6-9}
Predictor & Fit (s) & s/1k & Predict (ms) & Peak (MB) & Fit (s) & s/1k & Predict (ms) & Peak (MB) \\
\midrule
Random Forest & 41 & 0.10 & 30.2 & 899 & 84 & 0.03 & 104.5 & 6{,}411 \\
XGBoost & 62 & 0.16 & 4.4 & 964 & 103 & 0.03 & 44.6 & 2{,}024 \\
Transformer & 886 & 2.25 & 3.1 & 1{,}034 & 13{,}513 & 4.14 & 282.8 & 2{,}192 \\
LSTM & 292 & 0.74 & 1.2 & 962 & 10{,}753 & 3.29 & 26.1 & 1{,}997 \\
GNN & 5{,}389 & 13.71 & 204.3 & 1{,}074 & 43{,}994 & 13.46 & 6{,}837.9 & 2{,}748 \\
\midrule
\multicolumn{9}{p{\dimexpr\textwidth-2\tabcolsep\relax}}{%
\footnotesize\emph{Note:} Model-independent preprocessing required 5.9--6.3~s per predictor run for the baseline logs and 39.9--86.7~s for BPIC~2013.} \\
\bottomrule
\end{tabular}
}
\end{widetable}

The results reveal distinct computational profiles. GNN incurs the highest fitting cost, whereas Random~Forest exhibits the highest peak memory requirement. More importantly, GNN is the only predictor considered here whose fitting cost scales consistently with the number of labeled prefixes: its normalized cost remains stable across the baseline and BPIC~2013 logs. Within this experimental setting, prefix count therefore provides a useful basis for estimating the feasibility of the graph-native encoding. This interpretation remains conditional, however, because sequence length, relational density, and hardware were not varied independently.

The tabular predictors follow a different pattern. In particular, XGBoost shows substantial variation across tasks, with the most demanding cases generally associated with multiclass targets of higher label cardinality. This suggests that class structure should be considered when allocating its computational budget, although the experiment does not isolate cardinality from row count, feature density, and other task characteristics.

Finally, for the less computationally demanding learners, model-independent pipeline stages account for most of the overall runtime. Since these stages perform the same work regardless of the downstream predictor, differences between runs primarily reflect execution conditions rather than changes in workload. Their much greater stability during uncontested baseline runs supports this interpretation.
% ====================
\section{Benefits and limitations of the object-centric representation}\label{sec:discussion}
% ====================

We discuss the benefits of the object-centric representation relative to the extended collaborative (case-centric) representation of \cite{delgado2025predictcollab}, focusing on target expressiveness, explicit collaborative structure and traceability, and viewpoint flexibility. Drawing on RQ3, we then examine its structural and computational costs, including its behavior at scale, and discuss the limitations of the representation and the empirical evidence.

\subsection{Expressiveness of prediction targets}\label{sec:disc-expressiveness}

A first benefit is expressiveness. Because a participant object is defined at the log level, it relates the orchestration cases that one participant enacts across different collaboration cases, and a prediction target may quantify over that relation. The number of collaborations in which a participant is concurrently engaged when it next acts is one such target of interest for bottleneck and delay analysis in multi-organization settings. It has no counterpart in the case-centric taxonomy of \cite{delgado2025predictcollab}: under any single-case notion, the events of the other collaborations either disappear from the trace or lose the per-collaboration boundary that the count requires, so the quantity is not locally readable from a single flattened case's own trace: computing it would require a separate pass over the rest of the log to recover the per-collaboration boundaries the flattened trace discards, rather than reading a relation the representation already makes explicit. What the counted entity denotes depends on the source data: on the evaluated logs, whose participant identifiers are role- or type-level labels (Section~\ref{sec:map-decisions}), such a target would quantify the concurrent engagement of a role, whereas a log supplying per-instance organizational identifiers would yield that of a specific organization. The catalog reformulated here remains the fourteen types; targets of this kind delimit what the representation additionally enables, and their formalization and empirical assessment are left for future work.

\subsection{Explicit collaborative structure and traceability}\label{sec:disc-structure}

A second benefit is that the collaborative structure, which the extended collaborative log keeps implicit in event attributes, becomes explicit and directly navigable. Participants, orchestration cases, communication interactions, and global collaboration cases are first-class objects related by qualified relations rather than values attached to the events of a single case. As a consequence, participant- and message-centered analyses, which the case-centric representation obtains only through additional projection and preprocessing steps, are here read off the representation by traversing relations. This explicitness is also measurable: the descriptive structural counts quantify the collaborative structure that the object-centric representation materializes, and the case-centric log leaves implicit. The same explicitness supports traceability, as the paths illustrated in Figure~\ref{fig:example-ocel} connect an event to its participant and a communication event to its message without the need for attribute filtering.

\subsection{Flexibility of analysis viewpoints}\label{sec:disc-viewpoint}

A third benefit is viewpoint flexibility. In the object-centric representation, the case notion is not fixed in advance: the analysis viewpoint is chosen by selecting an object type, and the category of objects that a predicted label refers to (the target anchor) need not coincide with the scoping viewpoint. The same collaboration-case prefix can therefore support labels anchored on participants, messages, or orchestration cases without generating a separate source log for each target anchor. This decoupling allows the current catalog to retain a common viewpoint while using heterogeneous anchors, and it also leaves open prediction tasks based on alternative viewpoints. Existing object-centric PPM approaches do not need this decoupling, since their target and scoping object types coincide by construction (Section~\ref{sec:tasks-viewpoint}).

The present reformulation intentionally exercises only one point in this design space: \CC{} is fixed as the scoping viewpoint for all fourteen tasks so that their prefixes remain identical to those of \cite{delgado2025predictcollab}. The representation nevertheless admits alternative viewpoints, which would define different prediction problems. An \OC{} viewpoint would yield participant-local prefixes and support predictions about the continuation of one local execution, whereas a participant-type viewpoint could organize events across collaboration cases and support cross-case workload predictions. By contrast, a \Msg{} viewpoint has limited sequential value here because each message interaction is associated with a single event; it would become more useful if a correlated exchange or message lifecycle were represented as a multi-event object. Formalizing and evaluating tasks under these alternative viewpoints is left for future work.

\subsection{Operational cost and behavior at scale}\label{sec:disc-cost}

The benefits above come at a cost that RQ3 profiling makes measurable, and that cost has three components rather than one.

The first is structural. Materializing every communication interaction as its own object multiplies the number of objects and relations relative to a single trace, as the object and relation counts in Table~\ref{table_logcaracteristics} make apparent, and correspondingly enlarges the serialized log.

The second is the cost of consuming that structure. The stages that are model-independent by construction amount to roughly $6$~s per predictor run over the four baseline logs and to between $40$ and $87$~s on BPIC~2013, a spread that reflects machine conditions rather than workload. For the neural encodings, this is invisible against training, which accounts for over $98\%$ of their time. For the tabular ones, it is of the same order as training itself: on BPIC~2013, those $45$~s stand against $84$~s of fitting for Random Forest and $103$~s for XGBoost. The practical consequence for an object-centric pipeline is that, whenever the learner is cheap, the representation, not the model, is a first-order cost. This motivates profiling and optimizing the model-independent stages in that regime, but their aggregate time does not identify OCEL loading, rather than another stage, as the bottleneck.

The third is how these costs extrapolate, and here the encodings differ enough that a measurement taken on a small log does not transfer. The graph-native encoding is the only one whose per-prefix cost is stable across the two log groups --- $13.71$~s per $1{,}000$ labeled prefixes on the baseline logs against $13.46$~s on BPIC~2013, across an $8.5\times$ difference in scale --- which follows from its configuration, a fixed epoch budget over fixed-size subgraphs, one per labeled prefix. Among the encodings and logs examined here, it is the only one whose runtime tracks prefix count alone; this is an observation from these experiments, not a general predictability claim, since sequence length, relational density, and hardware were not varied independently across the five logs. The tabular predictors become cheaper per prefix on the larger log as their fixed budgets are amortized, whereas the sequence encodings become more expensive. Memory inverts outright: Random Forest is the lightest of the five on the baseline logs ($899$~MB peak) and by a wide margin the heaviest on BPIC~2013 ($6{,}411$~MB, about $40\%$ of the memory of the machine used and more than twice the graph-native encoding's peak). A plausible explanation is that it materializes the full feature table and its entire ensemble at once where the neural encodings stream mini-batches, but the reported figure is the whole-run peak resident set, which does not by itself identify which stage or which of the two (table or ensemble) drove it, nor separate that from incremental memory growth. This is the one cost dimension that neither the predictive metrics nor the wall-clock times reveal, and the one most likely to bind first as a collaborative log grows.

\subsection{Limitations}\label{sec:disc-limitations}

The message-synchronization and matching phenomena that the message objects expose are only lightly exercised: the representation asserts no correspondence between a send and a receive (rule~M4), so targets that presuppose such a pairing lie outside the catalog, and what a correlated representation would contribute to synchronization prediction remains open. Object attributes are moreover static, so a per-message lifecycle state is never carried as a changing attribute value. What remains available is an aggregate, prefix-level view: comparing the number of \emph{send} and \emph{receive} relations of a collaboration case up to a given prefix approximates a backlog of messages sent but not yet matched by a receive only under explicit assumptions the representation does not itself guarantee --- a one-to-one correlation between sends and receives, eventual completeness (every send is matched), no lost interactions, and a send always preceding its matching receive --- and even then without identifying the exchange each pending message belongs to, since a receive on an unrelated exchange offsets the count equally.

Second, individual performers remain event attributes rather than objects with an identity of their own (rule~M8), and only one evaluated log records an actor identifier at all; resource-level prediction is therefore definable over the source data but is neither anchored in the catalog nor evaluated here.

Third, an \OC{} object is identified by the collaboration case and the participant that delimits it (rule~M3). When a participant enacts its local process more than once within a single collaboration case, that pair does not separate the executions; the object aggregates them.

Fourth, the explicitness that benefits analysis carries a structural and operational cost, examined in Section~\ref{sec:disc-cost}: more objects and relations, a larger serialized log, and a preprocessing stage that becomes a first-order cost whenever the predictor itself is cheap.

Fifth, a limitation that is empirical rather than structural concerns what the representation has been shown to support predictively. On the four baseline logs, the predictors exceed their trivial baseline on the reformulated numeric targets in all but four task--log--predictor entries, all of them on NV-TNE; on the real-life log, none of them does so reliably (Section~\ref{sec:eval:summary}). Two readings are compatible with our data, and this study cannot separate them because log scale and log provenance vary together in it: the numeric targets of a real collaborative process may carry variance that the feature set extracted by the native object-centric pipeline does not explain, or the generative structure of the study logs may make their numeric targets unusually predictable. What the result does not support is attributing the gap to the object-centric representation itself, since the labels are provably source-level and the same features are extracted without flattening across all five logs. It does delimit the claim: the reformulation is executable and, under the interpretations of \ref{app:tasks}, label-preserving on every log, and demonstrably informative for numeric targets only on the curated ones.

None of these limitations affects the structural correctness of the mapping (P1.1--P1.6) or the reformulated tasks under the interpretations of \ref{app:tasks}; they delimit the scope within which the representation's benefits have been established.
The first three share a shape and are consequences of the preconditions the mapping deliberately does not impose, rather than of the rules themselves. The correspondence between a send and its reception, the individual performer, and a native identifier for a participant's local execution are each recorded by some source logs as ordinary event attributes, which rule~M8 preserves. Promoting them to a relation between existing message objects, to objects with their own identity, or to a refinement of each \OC{} into per-instantiation sub-objects would make them addressable through qualified relations rather than only readable as values, without disturbing the objects and relations that rules M1--M8 already create. Formalizing that refinement and assessing how message correlation, actor identity, and local-execution identity contribute to collaborative prediction are natural continuations of this work.

% =====================================================================
\section{Threats to Validity}\label{sec:threats}
% =====================================================================

We discuss the threats that may compromise the truthfulness of our findings, following the classification of \cite{wohlin2012experimentation} into conclusion, internal, construct, and external validity, and the measures taken to mitigate them.

\paragraph{Conclusion validity}
Each of the four baseline logs contains one hundred collaboration cases, which is a small sample for estimating per-task predictive performance under class imbalance; the fifth log, the reinterpreted BPI Challenge 2013, is an order of magnitude larger and partly offsets this concern, but it is a single log from a single domain and does not restore statistical power for the study as a whole. We therefore treat the RQ3 metrics, on all five logs, as descriptive evidence of operational feasibility and feature informativeness, not as estimates to be compared across representations or ranked; each predictor is run with a single fixed hyperparameter configuration and no systematic tuning, and trivial baselines are included only as a sanity check rather than as competing methods. To make the variability explicit rather than obscure it, results are reported as the mean and standard deviation across folds.

A further limitation concerns how the fixed hyperparameters were set: for the two predictors with no early-stopping mechanism (LSTM, Transformer), the epoch budget, and, for LSTM, the batch size and learning rate, were adjusted from a conventional starting point after observing marked underfitting on a small, non-exhaustive set of task--log combinations, rather than through a systematic procedure such as grid search or nested cross-validation. The reported LSTM and Transformer metrics should accordingly be read as the outcome of this exploratory calibration, not of a systematically tuned configuration; this does not affect the feasibility conclusions of RQ3, which do not rely on any predictor attaining its best possible performance. Relatedly, the reported dispersion is the fold-to-fold variability of a single run with a fixed seed on a single machine, not the variability across repeated independent cross-validation runs, so the reported standard deviations should not be interpreted as confidence intervals for experimental repetitions.

\paragraph{Internal validity}
The main internal threat is information leakage from the future of an execution into the prefix used to predict it. We mitigate it by partitioning at the collaboration-case level, wrapping the evaluation in grouped five-fold cross-validation on \CC so that all prefixes of a collaboration case remain in the same fold, and computing features only from the observable prefix. To guard against a misalignment between extracted feature rows and their intended targets, the pipeline includes a remaining-time oracle check that aborts on any schema mismatch.

A second internal threat concerns the two tasks that take a value parameter (NV-PaT and OB-M). Their parameter is resolved once per log by selecting the most frequent participant or message activity across the entire log and \emph{before} the fold split, and is then held fixed across the five folds. This resolution does not leak label or feature information into a test fold, since it only fixes which parameterized target is evaluated; it is not itself a per-fold decision. It does, however, define the task from the distribution of the whole log, test folds included, so the reported metrics are those of the best-populated parameter value rather than of a value drawn at random: they characterize the parameterized target actually evaluated, and should not be read as an expectation over parameter values. We state it as a design choice rather than as part of the cross-validation procedure.

\paragraph{Construct validity}
The constructs of interest are the semantic correctness of the mapping and the fidelity of the reformulated task labels. Both are operationalized structurally: the targets are defined over the object-centric relations of $\mu(L)$, and label fidelity is assessed by two readers that follow different computation paths, one applying the source-level definitions to the original log and one traversing the generated OCEL~2.0 relations. Using independent computation paths reduces the risk that a shared implementation bias inflates their agreement, although, as noted in Section~\ref{sec:rq2experimentation}, the two readers are independent code paths rather than independent implementations by different teams.

A construct-specific threat to the catalog is that two of its fourteen task types do not induce distinct targets in the evaluated data. NE-NEPr labels the activity of the next event and NE-NEPa labels the pair (activity, actor) of that same event; on the four baseline logs, no activity is ever performed by more than one actor, so the pair is a bijective relabeling of the activity, and both tasks induce the identical partition of the label space. Their scores are consequently identical for every predictor on three of the four logs and differ little on the fourth, where the residual gap reflects the label ordering seen by the stochastic components rather than a different partition. This is a property of the evaluated data, not of the reformulation, which keeps the two types distinct by definition.

Finally, BPI Challenge 2013 is not a natively collaborative log: obtaining a collaborative reading of it requires a domain reinterpretation in which organizational lines are treated as participants and incidents as collaboration instances, and, critically, inter-participant hand-overs are recovered heuristically from a change of responsible line between consecutive events, with the corresponding send interactions synthesized rather than observed (Section~\ref{sec:eval:datasets}). The synthesized elements are flagged as such in the resulting log, but the log's collaborative structure is an interpretation of the source data, not a record of it. This matters most where our conclusions are negative: the numeric results reported on this log are conditioned on that reinterpretation being an adequate reading of the underlying process.

\paragraph{External validity}
Four collaborative event logs are publicly available in \cite{delgado2025predictcollab}, comprising both synthetic and realistic scenarios but covering a narrow range of domains, which limits the generalizability of the quantitative observations to other collaborative processes. To partly address this threat, we added the reinterpreted BPI Challenge 2013 log, which extends the RQ1 applicability evidence and, at full scale, the RQ2 label-equivalence evidence to real-life data, and extends the RQ3 evidence to both the six-task representative subset and the full fourteen-task catalog.

On that log, the RQ3 fixed-hyperparameter pipeline behaves differently from the four baseline logs, and the difference is specific rather than uniform. All five predictors exceed the trivial baseline on all six non-degenerate classification task types (30 task--predictor entries). Across the six regression task types (also 30 entries), by contrast, four of the five predictors are worse than the trivial baseline on every task, by $21.5\%$ to $42.8\%$ of mean relative MAE, and the fifth (GNN) is better on four tasks but by margins under $0.6\%$ on three of them. We report this descriptively rather than draw a comparative conclusion.

Two properties of the design limit how far that observation can be generalized. First, scale and provenance are confounded: BPIC~2013 is simultaneously an order of magnitude larger, real rather than curated, and a single log contributed against four, so its differences from the baseline logs should not be taken as a controlled variation of log size. It would require a second real-life collaborative log, which lies outside the scope of this evaluation. Second, the converter assumes the extended collaborative log schema of Predict-Collab, so applying the mapping to logs produced by other collaboration-mining pipelines would require aligning their attributes with the ones that rules M1--M8 consume; the two-level design mitigates this at the representation level, since the same core-level mapping could be serialized for other OCED implementations, of which OCEL~2.0 is the one used here.

Finally, to avoid any conflict of interest that could affect the reported findings, the study is deliberately non-comparative and makes no superiority claim over the collaborative baseline, consistent with the declaration in the front matter.

% ====================
\section{Conclusions}\label{sec:conclusion}
% ====================

This paper investigated the connection between collaborative and object-centric predictive process monitoring. We defined a semantic mapping from extended collaborative event logs to an object-centric representation that conforms to the OCED core model and is serialized in OCEL~2.0, in which the global collaboration case, the orchestration cases, the participants, and the communication interactions are explicitly represented as related objects. We reformulated fourteen collaborative prediction tasks as object-centric tasks over this representation and provided a reproducible converter and prediction pipeline over three encoding approaches and five corresponding predictive models.

The empirical validation was conducted on five collaborative event logs: the four publicly available logs of the collaborative baseline \citep{delgado2025predictcollab}, and a fifth one derived from the BPI Challenge 2013 incident-management log through a documented collaborative reinterpretation, which is an order of magnitude larger and the only real-life log of the study. On these logs, the converter yields representations satisfying the consistency criteria P1.1--P1.6, exported as schema-valid OCEL~2.0 logs (RQ1); the reformulated tasks reproduce the source labels on every evaluated prefix, with zero mismatches; and the full catalog of fourteen tasks executes end-to-end for the five predictors on the five logs, so that all $350$ task--log--predictor combinations completed their five folds without failure (RQ3).

Overall, the representation makes the collaborative structure explicit and directly addressable, supports flexible analytical viewpoints, and facilitates stating targets that naturally arise from the object-centric structure and fall outside the original single-case taxonomy, at the cost of a larger relational structure and dependence on object-centric tooling.

Several directions remain open. Relating the two interactions of a single logical message, giving the individual performer an identity of its own, and separating the repeated instantiations of a participant's local process would make new prediction targets addressable within the representation; formalizing and evaluating them is the immediate next step in this work. The numeric result on the real-life log points to two further directions that the present evaluation cannot separate, since log scale and log provenance vary together in it: enriching the native object-centric feature set that the pipeline currently extracts, and evaluating a second real-life collaborative log. Systematic hyperparameter search, deliberately excluded here in favor of a single fixed configuration, is a prerequisite for either. Further work includes validating the framework on collaborative logs from other domains and produced by other collaboration-mining pipelines, and studying how the object-centric representation compares with the case-centric one in terms of predictive performance.

\appendix
% ==================
\section{Formal definitions of the mapping}\label{app:formal}
% ==================

This appendix provides the formal counterparts to the mapping rules introduced in Sect.~\ref{sec:mapping}. It formalizes the transformation $\mu$ (rules M1--M8) and establishes the consistency criteria P1.1--P1.6.

\begin{definition}\label{def:app-r1}
An extended \textbf{collaborative event log} (adapted from \cite{delgado2025predictcollab}) is a tuple
$L = (E_L, \mathcal{C}, \mathcal{P}, \mathcal{A}, X_L, V_L, \mathit{case}, \mathit{act}, \mathit{time}_L, \mathit{part}, \mathit{elem}, \mathit{from}, \mathit{to}, D, \prec_0)$
where $E_L$ is a finite set of events; $\mathcal{C}$, $\mathcal{P}$, and $\mathcal{A}$ are sets of collaboration-case identifiers, participant identifiers, and activity labels; $X_L$ is a set of residual attribute names and $V_L$ a set of values; $\mathit{case}\colon E_L \to \mathcal{C}$, $\mathit{act}\colon E_L \to \mathcal{A}$, $\mathit{time}_L\colon E_L \to \mathbb{T}$, $\mathit{part}\colon E_L \to \mathcal{P}$, and $\mathit{elem}\colon E_L \to \{\textit{task}, \textit{SendTask}, \textit{ReceiveTask}\}$ are total functions; $\mathit{from}, \mathit{to}\colon E_L \rightharpoonup \mathcal{P}$ are partial functions defined only when $\mathit{elem}(e) \neq \textit{task}$, and possibly undefined even then, when the source format does not record that endpoint for $e$ (see the Normalization paragraph below); $D\colon E_L \to (X_L \rightharpoonup V_L)$ assigns to each event $e$ a partial function giving the values of its residual data attributes, i.e., the data attributes $d_1, \ldots, d_m$ of \cite{delgado2025predictcollab} other than $\mathit{part}$, $\mathit{elem}$, $\mathit{from}$, and $\mathit{to}$ (for $x\in X_L$, $D(e)(x)$ is defined exactly when $e$ carries a value for $x$, i.e., when $x\in\mathrm{dom}(D(e))$; rule M8 preserves $D(e)$ together with the structural source attributes); and $\prec_0$ is a fixed total order on $E_L$ given by the order of appearance of the events in the merged collaborative log. These accessors correspond, respectively, to the functions $\pi_C$, $\pi_A$, $\pi_T$, $\pi_P$, and $\pi_{ET}$ of the collaborative event log of \cite{delgado2025predictcollab}; $\mathit{from}$ and $\mathit{to}$ correspond to its \texttt{fromParticipant} and \texttt{toParticipant} attributes. We write $\mathcal{C}_L = \mathit{case}(E_L)$, $\mathcal{A}_L = \mathit{act}(E_L)$, $\mathcal{P}_L = \mathit{part}(E_L) \cup \mathrm{ran}(\mathit{from}) \cup \mathrm{ran}(\mathit{to})$, $S_L = \{e \in E_L : \mathit{elem}(e) = \textit{SendTask}\}$, and $R_L = \{e \in E_L : \mathit{elem}(e) = \textit{ReceiveTask}\}$. The log is well formed when (i) $\mathit{elem}(e)=\textit{SendTask}$ implies that $\mathit{from}(e)$ is defined and equals $\mathit{part}(e)$, and (ii) $\mathit{elem}(e)=\textit{ReceiveTask}$ implies that $\mathit{to}(e)$ is defined and equals $\mathit{part}(e)$; the equalities are thus read strictly, so an event whose own-side endpoint is undefined violates well-formedness (the normalization $\nu$ below restores it). In particular, on a well-formed log, only the counterparty side ($\mathit{to}$ on $S_L$, $\mathit{from}$ on $R_L$) may remain undefined. (A third condition requiring every $p\in\mathcal{P}_L$ to occur as a send endpoint is unnecessary and is not imposed: $\mathcal{P}_L=\mathit{part}(E_L)\cup\mathrm{ran}(\mathit{from})\cup\mathrm{ran}(\mathit{to})$ already guarantees, by construction, that every $p\in\mathcal{P}_L$ is the participant of some event or the $\mathit{from}$/$\mathit{to}$ endpoint of some $e\in S_L\cup R_L$; see the proof of P1.5.) The timestamp domain $\mathbb{T}$ is assumed to be totally ordered and equipped with a difference operation $t_1 - t_2$ taking values in a domain $\Delta\mathbb{T}$ of \emph{durations}, in general distinct from $\mathbb{T}$ itself, since a duration need not be a valid instant; the order is what makes $\prec_L$ below a trace order, and the difference is what the time-based targets of \ref{app:tasks} return. The Representation of source domains paragraph below fixes the encoding of $\mathbb{T}$ into $\mathbb{U}_{\mathit{time}}$, the encoding of $\Delta\mathbb{T}$ into $\mathbb{U}_{\mathit{val}}$, and the compatibility they must satisfy. The element type $\mathit{elem}(e)$ already distinguishes \textit{task}, \textit{SendTask}, and \textit{ReceiveTask}, so no separate message-type function is needed. Let $\prec_L$ be the per-case order defined by $e\prec_L e'$ iff $\mathit{case}(e)=\mathit{case}(e')$ and either $\mathit{time}_L(e)<\mathit{time}_L(e')$, or $\mathit{time}_L(e)=\mathit{time}_L(e')$ and $e\prec_0 e'$; that is, events are ordered by timestamp and ties are broken by their order of appearance. Because $e\prec_L e'$ requires $\mathit{case}(e)=\mathit{case}(e')$, events of different cases are $\prec_L$-incomparable; this tie-breaking makes $\prec_L$ a strict total order on each per-case set $E_c=\{e\in E_L:\mathit{case}(e)=c\}$, not on all of $E_L$, and guarantees that the trace order of each case is preserved before and after the mapping. The trace of collaboration case $c$ is $\sigma_c = \langle e_1, \ldots, e_{n_c} \rangle$, the sequence of $\{e \in E_L : \mathit{case}(e) = c\}$ ordered by $\prec_L$; since $\mathcal{C}_L=\mathit{case}(E_L)$, every $c\in\mathcal{C}_L$ has $n_c\ge 1$.

\paragraph{Convention: cases and prefix ranges.} Throughout \ref{app:formal} and \ref{app:tasks}, quantification over "every collaboration case $c$" ranges over $c\in\mathcal{C}_L$ unless stated otherwise, and an interval of prefix lengths such as $[1,n_c-1]$ denotes the \emph{integer} interval $\{1,2,\ldots,n_c-1\}$, empty when $n_c-1<1$. Since $n_c\ge 1$ for every $c\in\mathcal{C}_L$, the only case in which this range is empty is $n_c=1$: a case consisting of a single event has no valid prefix length under $k\in[1,n_c-1]$, so none of the fourteen tasks of \ref{app:tasks} is ever evaluated on it.
\end{definition}

\paragraph{Normalization.} The source format may leave a send or receive event's own endpoint implicit, recording only the counterparty: $\mathit{from}(e)$ undefined for some $e\in S_L$, or $\mathit{to}(e)$ undefined for some $e\in R_L$, relying on \texttt{collab:participant} to supply the missing side. Given a log $L_0$ satisfying every requirement of Definition~\ref{def:app-r1} except that $\mathit{from}_0,\mathit{to}_0$ may be undefined on such events, define $\nu(L_0)$ as the log obtained by replacing $\mathit{from}_0,\mathit{to}_0$ with
\[
\begin{aligned}
\mathit{from}(e) &= \begin{cases}
\mathit{part}(e) & \text{if } e\in S_L \text{ and } \mathit{from}_0(e) \text{ is undefined},\\
\mathit{from}_0(e) & \text{otherwise},
\end{cases}\\[4pt]
\mathit{to}(e) &= \begin{cases}
\mathit{part}(e) & \text{if } e\in R_L \text{ and } \mathit{to}_0(e) \text{ is undefined},\\
\mathit{to}_0(e) & \text{otherwise}.
\end{cases}
\end{aligned}
\]
This backfills only the event's \emph{own} side, from \texttt{collab:participant}, which the source format always carries. The \emph{counterparty} side --- $\mathit{to}_0(e)$ for $e\in S_L$, $\mathit{from}_0(e)$ for $e\in R_L$ --- is taken as given by the source and is not guessed: when the source omits it, $\mathit{from}(e)$ or $\mathit{to}(e)$ stays undefined after $\nu$, so $\mathit{from},\mathit{to}$ remain genuinely partial on $S_L\cup R_L$ in that case. This occurs in the logs used in this work; Section~\ref{sec:rq1experimentation} reports the affected events per log. For such an event, $\mu$ still creates its \Msg{} object (rule M4) but omits the corresponding \emph{from}/\emph{to} relation and \texttt{sender}/\texttt{receiver} attribute (\ref{def:app-mapping}) rather than inferring a value; criterion P1.3 reports this completeness gap explicitly rather than treating it as a disagreement. Well-formedness conditions (i) and (ii) only constrain the \emph{own} side. Note that $\nu$ backfills a \emph{missing} own side but deliberately does not overwrite a recorded one, so (i) and (ii) hold after $\nu$ exactly when any own-side endpoint the source \emph{does} record already agrees with \texttt{collab:participant}: a source event carrying $\mathit{from}_0(e) \neq \mathit{part}(e)$ for a send (resp.\ $\mathit{to}_0(e) \neq \mathit{part}(e)$ for a receive) remains non-well-formed after $\nu$ and is surfaced by criterion P1.3 as a participant/sender (resp.\ participant/receiver) disagreement, rather than silently corrected. In the logs used in this work, this premise holds vacuously --- no send records its own $\mathit{from}_0$ and no receive records its own $\mathit{to}_0$; the own side is always left implicit --- so (i) and (ii) hold after $\nu$, regardless of whether the counterparty side is defined. The mapping $\mu$ below is understood to be applied to $\nu(L_0)$ rather than to a possibly non-well-formed $L_0$ directly; the converter performs this normalization while reading the source log (\texttt{collab\_xes\_to\_ocel.py}, function \texttt{\_sorted\_case\_events}).

\begin{definition}\label{def:app-ocel}
Let $\mathbb{U}_{\mathit{ev}}$, $\mathbb{U}_{\mathit{etype}}$, $\mathbb{U}_{\mathit{obj}}$, $\mathbb{U}_{\mathit{otype}}$, $\mathbb{U}_{\mathit{attr}}$, $\mathbb{U}_{\mathit{val}}$, $\mathbb{U}_{\mathit{time}}$ (totally ordered, with smallest element $0$ and largest element $\infty$, i.e., $0\le t\le\infty$ for every $t\in\mathbb{U}_{\mathit{time}}$), and $\mathbb{U}_{\mathit{qual}}$ be the pairwise disjoint universes of events, event types, objects, object types, attribute names, attribute values, timestamps, and qualifiers of Definition~1 of \cite{berti2024ocel}. An \textbf{object-centric event log} is a tuple
\[
\begin{aligned}
\mathcal{L} = (&E, O, EA, OA, \mathit{evtype}, \mathit{time}, \mathit{objtype},
                 \mathit{eatype}, \mathit{oatype},\\
               &\mathit{eaval}, \mathit{oaval}, \mathit{E2O}, \mathit{O2O}),
\end{aligned}
\]

\noindent where:
\begin{itemize}
    \item $E \subseteq \mathbb{U}_{\mathit{ev}}$ is the set of events;
    \item $O \subseteq \mathbb{U}_{\mathit{obj}}$ is the set of objects;
    \item $EA \subseteq \mathbb{U}_{\mathit{attr}}$ is the set of event attributes;
    \item $OA \subseteq \mathbb{U}_{\mathit{attr}}$ is the set of object attributes;
    \item $\mathit{evtype} \colon E \rightarrow \mathbb{U}_{\mathit{etype}}$ assigns event types to events;
    \item $\mathit{time} \colon E \rightarrow \mathbb{U}_{\mathit{time}}$ assigns timestamps to events;
    \item $\mathit{objtype} \colon O \rightarrow \mathbb{U}_{\mathit{otype}}$ assigns types to objects;
    \item $\mathit{eatype} \colon EA \rightarrow \mathbb{U}_{\mathit{etype}}$ assigns event types to event attributes;
    \item $\mathit{oatype} \colon OA \rightarrow \mathbb{U}_{\mathit{otype}}$ assigns object types to object attributes;
    \item $\mathit{eaval} \colon E \times EA \nrightarrow \mathbb{U}_{\mathit{val}}$ assigns values to event attributes (partial function);
    \item $\mathit{oaval} \colon O \times OA \times \mathbb{U}_{\mathit{time}} \nrightarrow \mathbb{U}_{\mathit{val}}$ assigns values to object attributes (partial function);
    \item $\mathit{E2O} \subseteq E \times \mathbb{U}_{\mathit{qual}} \times O$ is the set of qualified event-to-object relations;
    \item $\mathit{O2O} \subseteq O \times \mathbb{U}_{\mathit{qual}} \times O$ is the set of qualified object-to-object relations.
\end{itemize}

\noindent such that:
\begin{itemize}
    \item $\mathrm{dom}(\mathit{eaval}) \subseteq \{(e, \mathit{ea}) \in E \times EA \mid \mathit{evtype}(e) = \mathit{eatype}(\mathit{ea})\}$
    \item $\mathrm{dom}(\mathit{oaval}) \subseteq \{(o, \mathit{oa}, t) \in O \times OA \times \mathbb{U}_{\mathit{time}} \mid \mathit{objtype}(o) = \mathit{oatype}(\mathit{oa})\}$.
\end{itemize}
As noted in \cite{berti2024ocel}, these constraints make the attribute sets of distinct event types and of distinct object types disjoint, and a single value assignment at timestamp $0$ encodes a static (unchanging) attribute value.
\end{definition}

\paragraph{Per-type attribute convention.} Since $\mathit{eatype}$ is a function and attribute sets are disjoint across event types, an attribute that is conceptually shared by all event types is represented formally as one attribute per event type. Thus, $\texttt{collab:participant}_a$, $\texttt{collab:elemType}_a$, $\texttt{collab:fromParticipant}_a$, and $\texttt{collab:toParticipant}_a$ denote the corresponding source attributes for event type $a \in \mathcal{A}_L$. Analogously, $x_a$ denotes a residual source attribute $x$ carried by events of type $a$ (rule M8). The convention applies verbatim to object attributes through $\mathit{oatype}$: $\texttt{caseId}_{\CC}$ and $\texttt{caseId}_{\OC}$ denote homonymous attributes of distinct object types, and, because rule~M2 gives each participant identifier an object type of its own, the participant name is likewise one attribute per participant type, written $\texttt{name}_p$ for $p \in \mathcal{P}_L$. We use the unqualified name when the type is clear from context. The relational format of \cite{berti2024ocel} follows this convention by using separate tables for different event and object types.

\paragraph{Representation of source domains.} We assume fixed injective encodings of source activity labels into $\mathbb{U}_{\mathit{etype}}$, source timestamps into $\mathbb{U}_{\mathit{time}}$, and source identifiers and attribute values (including $V_L$ of Definition~\ref{def:app-r1}) into $\mathbb{U}_{\mathit{val}}$; the residual attribute \emph{names} $X_L$ are not values but attribute names, encoded per event type by a fixed injective map $(x,a) \mapsto x_a$ of $X_L \times \mathcal{A}_L$ into $\mathbb{U}_{\mathit{attr}}$ (the per-type attribute convention above), with range disjoint from the reserved names that populate $EA$ after the M8 renaming --- $\texttt{participant}_a$, $\texttt{elemType}_a$, $\texttt{fromParticipant}_a$, and $\texttt{toParticipant}_a$ --- as well as from their \texttt{collab:}-prefixed source counterparts; this disjointness is what keeps $\mathit{eaval}$ single-valued in Definition~\ref{def:app-mapping} even when a residual source attribute is itself named, e.g., \texttt{participant}. To avoid clutter, the encoding functions are left implicit. For source timestamps, we require more than injectivity. We assume, beyond Definition~\ref{def:app-ocel} --- which fixes for $\mathbb{U}_{\mathit{time}}$ no operation besides a smallest element $0$, and gives no universe of durations at all --- a typed subtraction $\ominus\colon\mathbb{U}_{\mathit{time}}\times\mathbb{U}_{\mathit{time}}\to\mathbb{U}_{\mathit{val}}$, landing in the OCEL attribute-value universe rather than in $\mathbb{U}_{\mathit{time}}$ itself, since a duration is not an instant and Definition~\ref{def:app-ocel} provides no dedicated universe for it; $\mathbb{U}_{\mathit{val}}$ is the standard's own carrier for typed non-instant values, and is disjoint from $\mathbb{U}_{\mathit{time}}$, so this reuses the metamodel rather than extending it. The encoding $\iota\colon\mathbb{T}\to\mathbb{U}_{\mathit{time}}$ must then be an embedding compatible with order and subtraction, i.e., for all $t_1,t_2,t_3,t_4\in\mathbb{T}$ arising in $L$, $\iota(t_1)\le\iota(t_2)$ iff $t_1\le t_2$, and
\[
\iota(t_1)\ominus\iota(t_2)=\iota(t_3)\ominus\iota(t_4)
\quad\Longleftrightarrow\quad
t_1-t_2=t_3-t_4.
\]
Concretely, let $\Delta\mathbb{T} = \{t_1 - t_2 : t_1, t_2 \in \mathbb{T} \text{ arising in } L\}$ be the domain of \emph{durations} obtained by subtracting two source timestamps -- a domain distinct from $\mathbb{T}$ itself, since a duration is not, in general, a valid instant. Subtraction-compatibility states that $\iota$ induces a map $\iota_\Delta(t_1 - t_2) = \iota(t_1)\ominus\iota(t_2)$ on $\Delta\mathbb{T}$ --- well defined by the right-to-left implication and injective by the left-to-right implication --- and hence invertible on its image $\iota_\Delta(\Delta\mathbb{T}) \subseteq \mathbb{U}_{\mathit{val}}$; we write $\mathrm{dec}_{\Delta\mathbb{T}}\colon \iota_\Delta(\Delta\mathbb{T}) \to \Delta\mathbb{T}$ for that inverse. This range need not be disjoint from the value encoding's range fixed above: which inverse to apply to a task's output --- $\mathrm{dec}_{\Delta\mathbb{T}}$, $\mathrm{dec}_{\mathcal{A}}$, $\mathrm{dec}_{\mathsf{Pa}}$, or the identity --- is selected by that task's declared return type (Table~\ref{tab:p2-bytask}, \ref{app:tasks}), never by inspecting the value itself, so no ambiguity arises from the two encodings sharing a codomain. This is needed so that time differences computed on $\mu(L)$ (rules NV-PrT, NV-PaT, NV-TNE, NV-TNM of \ref{app:tasks}) agree with the corresponding source-level differences once decoded via $\mathrm{dec}_{\Delta\mathbb{T}}$, not merely happen to be distinguishable (Proposition~\ref{prop:p2-equivalence}); no task in \ref{app:tasks} returns an absolute timestamp (a member of $\mathbb{U}_{\mathit{time}}$ decoded via the pointwise inverse of $\iota$), so $\mathrm{dec}_{\Delta\mathbb{T}}$, typed over differences rather than points, is the only timestamp-decoding function \ref{app:tasks} actually applies to a task output. The symbols \CC, \OC, and \Msg{} denote three distinct members of $\mathbb{U}_{\mathit{otype}}$, and all attribute and qualifier names used below denote members of their corresponding universes. Rule M2 turns participant identifiers into object types as well, so we additionally fix an injective encoding $\rho\colon\mathcal{P}\to\mathbb{U}_{\mathit{otype}}$ whose range avoids those three symbols, and write $\TPa=\rho(\mathcal{P}_L)$ for the set of participant types the log declares, one per identifier occurring in it. A participant identifier is therefore encoded twice, as the object type $\rho(p)$ and as the attribute value stored by $\texttt{name}_p$, which are members of the disjoint universes $\mathbb{U}_{\mathit{otype}}$ and $\mathbb{U}_{\mathit{val}}$; the assertion that a type and a name attribute ``both equal'' a participant identifier (criterion~P1.4) means that both are the image of the same $p\in\mathcal{P}_L$ under the respective injective encoding. The symbol \textsf{Participant}{} names the role that these objects play, not an object type: under rule M2, there is no single \textsf{Participant}{} type, and the set $O_{\textsf{Participant}}$ of participant objects introduced below is typed by $\TPa$.

\paragraph{Encoding convention.} The equations of rules M1--M8 below, the well-definedness sketch, and the accessor-invariance lemma of \ref{app:tasks} write source-domain terms such as $\mathit{act}(e)$, $\mathit{time}_L(e)$, $\mathit{part}(e)$, $\mathit{elem}(e)$, and $D(e)(x)$ as if they already belonged to the corresponding OCEL universe ($\mathbb{U}_{\mathit{etype}}$, $\mathbb{U}_{\mathit{time}}$, $\mathbb{U}_{\mathit{val}}$). Each such occurrence denotes the image of that source element under the fixed injective encoding declared above --- $\iota$ for timestamps, the activity-label and value encodings, and $\rho$ for participant identifiers where a type rather than a value is required --- never the source element itself; the encodings being injective, this identification is harmless and every equality or set membership stated in these terms transports back to the source domain via the corresponding inverse. We adopt this convention throughout \ref{app:formal} and \ref{app:tasks} to avoid subscripting every occurrence with its encoding function.

\paragraph{Identifier creation} The universes of Definition~\ref{def:app-ocel} are pairwise disjoint. In particular, $\mathbb{U}_{\mathit{obj}} \cap \mathbb{U}_{\mathit{val}} = \emptyset$, so a source identifier stored as an attribute value cannot itself be used as an object. We therefore fix injective naming functions with pairwise disjoint ranges in $\mathbb{U}_{\mathit{obj}}$:
$c \mapsto \mathit{cc}_c$, \;
$p \mapsto p_p$, \;
$(c,p) \mapsto \mathit{oc}_{c,p}$, \;
$e \mapsto m_e$ (for $e \in S_L \cup R_L$), \;
and an injection $\mu_E\colon E_L \to \mathbb{U}_{\mathit{ev}}$. For the object naming functions, any concrete injective scheme is sufficient; the converter uses type prefixes separated by \texttt{::}, such as \texttt{cc::case\_44}. For $\mu_E$ we require more: fix a strict total order $\prec_{\mathit{ev}}$ on $\mu_E(E_L)\subseteq\mathbb{U}_{\mathit{ev}}$ (for a string-valued encoding, its lexicographic order), and require $\mu_E$ to be \emph{order-preserving} from $(E_L,\prec_L)$ into $(\mathbb{U}_{\mathit{ev}},\prec_{\mathit{ev}})$, i.e., $e\prec_L e'$ implies $\mu_E(e)\prec_{\mathit{ev}}\mu_E(e')$ for all $e,e'\in E_L$. (We deliberately do not call $\mu_E$ an order-embedding: the converse implication necessarily fails across cases, whose events are $\prec_L$-incomparable yet $\prec_{\mathit{ev}}$-comparable. Within each case, however, $\prec_L$ is total, so order preservation together with injectivity yields the full biconditional there --- which is all that criterion P1.2 uses.) Injectivity of $\mu_E$ alone does not give order preservation: creating per-case indices $0,1,\ldots$ in $\prec_L$ order and formatting them without fixed-width zero padding is \emph{not} order-preserving under lexicographic order, since \texttt{"...::10"} precedes \texttt{"...::9"}; the converter pads each index to the width needed for its case.

\begin{definition}\label{def:app-mapping}
Given $L$ as in Definition~\ref{def:app-r1}, $\mu$ produces the object-centric event log $\mu(L)$ defined component-wise by rules M1--M8, i.e.,
$\mu(L) = (E, O, EA, OA,$ $
\mathit{evtype}, \mathit{time}, \mathit{objtype}, \mathit{eatype}, \mathit{oatype},
\mathit{eaval}, \mathit{oaval}, \mathit{E2O}, \mathit{O2O})$

\paragraph{Objects and object types (M1--M4)}
\begin{description}[leftmargin=3.2em,style=nextline]
  \item[M1.] $O_{\CC} = \{\, \mathit{cc}_c : c \in \mathcal{C}_L \,\}$. Every such object satisfies $\mathit{objtype}(\mathit{cc}_c) = \CC$.
  \item[M2.] $O_{\textsf{Participant}} = \{\, p_p : p \in \mathcal{P}_L \,\}$. For every $p \in \mathcal{P}_L$, $\mathit{objtype}(p_p) = \rho(p)$, with log-level scope (one object per identifier across the whole log). The object type of a participant object is thus the participant identifier itself, and the types declared by the log are $\TPa = \{\rho(p) : p \in \mathcal{P}_L\}$, one per identifier. The set $O_{\textsf{Participant}}$ collects the participant objects; it is not indexed by a single object type.
  \item[M3.] $O_{\OC} = \{\, \mathit{oc}_{c,p} : \exists e \in E_L.\ \mathit{case}(e) = c \wedge \mathit{part}(e) = p \,\}$, with $\mathit{objtype}(\mathit{oc}_{c,p}) = \OC$. The object $\mathit{oc}_{c,p}$ materializes the orchestration case that participant $p$ executes within collaboration case $c$; the pair $(c,p)$ identifies it, as the source format guarantees no local case identifier (Section~\ref{sec:map-decisions}).
  \item[M4.] $O_{\Msg} = \{\, m_e : e \in S_L \cup R_L \,\}$, with $\mathit{objtype}(m_e) = \Msg$. Each communication event, whether a send or a receive, yields its own message object; the core mapping introduces no correspondence between distinct send and receive interactions.
\end{description}
Then $O = O_{\CC} \cup O_{\textsf{Participant}} \cup O_{\OC} \cup O_{\Msg}$.

\paragraph{Events (M5)} $E = \mu_E(E_L)$, and for every $e \in E_L$:
\[
\mathit{evtype}(\mu_E(e)) = \mathit{act}(e), \qquad \mathit{time}(\mu_E(e)) = \mathit{time}_L(e).
\]

\paragraph{Event attributes (M5, M8)} Let $X_a = \bigcup_{e\in E_L,\ \mathit{act}(e)=a} \mathrm{dom}(D(e))$ denote, for each $a \in \mathcal{A}_L$, the set of residual source attribute names carried by at least one event $e$ of activity $a$ (Definition~\ref{def:app-r1}), and let $\mathcal{A}^{\mathit{msg}}_L=\{\mathit{act}(e):e\in S_L\cup R_L\}$. Then
\[
\begin{split}
EA = {} & \{\, \texttt{participant}_a,\ \texttt{elemType}_a : a \in \mathcal{A}_L \,\}\\
& \cup \{\, \texttt{fromParticipant}_a,\ \texttt{toParticipant}_a : a \in \mathcal{A}^{\mathit{msg}}_L \,\}\\
& \cup \{\, x_a : a \in \mathcal{A}_L,\ x \in X_a \,\},
\end{split}
\]
where each attribute subscripted by $a$ is assigned event type $a$ by $\mathit{eatype}$ (following the M8 renaming convention: the retained event attributes drop the \texttt{collab:} prefix of the corresponding source attribute --- \texttt{collab:participant}, \texttt{collab:elemType}, \texttt{collab:fromParticipant}, \texttt{collab:toParticipant} --- of the Per-type attribute convention paragraph above). The partial function $\mathit{eaval}$ is defined by
\begin{align*}
\mathit{eaval}\bigl(\mu_E(e),\texttt{participant}_{\mathit{act}(e)}\bigr) &= \mathit{part}(e),\\
\mathit{eaval}\bigl(\mu_E(e),\texttt{elemType}_{\mathit{act}(e)}\bigr) &= \mathit{elem}(e),
\end{align*}
for every $e\in E_L$. When $\mathit{elem}(e)\neq\textit{task}$, whenever the corresponding right-hand side is defined, also by
\begin{align*}
\mathit{eaval}\bigl(\mu_E(e),\texttt{fromParticipant}_{\mathit{act}(e)}\bigr) &= \mathit{from}(e),\\
\mathit{eaval}\bigl(\mu_E(e),\texttt{toParticipant}_{\mathit{act}(e)}\bigr) &= \mathit{to}(e),
\end{align*}
and by
\[
\mathit{eaval}\bigl(\mu_E(e),x_{\mathit{act}(e)}\bigr)=D(e)(x)
\]
whenever $x\in\mathrm{dom}(D(e))$; $\mathit{eaval}$ is undefined on \texttt{fromParticipant}$_{\mathit{act}(e)}$ (resp.\ \texttt{toParticipant}$_{\mathit{act}(e)}$) precisely when $\mathit{from}(e)$ (resp.\ $\mathit{to}(e)$) is, i.e., when the source leaves that endpoint unrecorded (Normalization paragraph above). Thus, structural attributes are retained even when their semantics are also materialized by E2O or O2O relations, and M8 drops no attribute of $L=\nu(L_0)$: preservation here is relative to the normalized log, not to the raw source $L_0$. On the own-side \texttt{fromParticipant}/\texttt{toParticipant} attribute in particular, M8 preserves exactly what $\nu$ already supplies (Normalization paragraph above) --- the source value $\mathit{from}_0(e)$/$\mathit{to}_0(e)$ when the raw log records it, or the value $\nu$ backfills from \texttt{collab:participant} when $L_0$ leaves it implicit --- so this is not a claim that $L_0$'s own-side attribute is always literally reproduced.

\paragraph{Object attributes (M1--M4)}
\[
\begin{split}
OA = {} & \{\texttt{caseId}_{\CC}\} \cup \{\texttt{caseId}_{\OC},\ \texttt{participant}\}\\
& \cup \{\texttt{sender},\ \texttt{receiver}\} \cup \{\, \texttt{name}_p : p \in \mathcal{P}_L \,\}
\end{split}
\]
where $\mathit{oatype}$ assigns the type \CC{} to $\texttt{caseId}_{\CC}$, the type \OC{} to $\texttt{caseId}_{\OC}$ and $\texttt{participant}$, the type \Msg{} to $\texttt{sender}$ and $\texttt{receiver}$, and the type $\rho(p)$ to $\texttt{name}_p$ for each $p \in \mathcal{P}_L$, one name attribute per participant type, as the disjointness of attribute sets across object types requires (Per-type attribute convention above). A message object carries no message-type attribute, since $\mathit{elem}$ already classifies the related events. All values are assigned by $\mathit{oaval}$ exactly once, at the reference timestamp $0 \in \mathbb{U}_{\mathit{time}}$ (the encoding of static attribute values in \cite{berti2024ocel}):
\begin{align*}
\mathit{oaval}(\mathit{cc}_c, \texttt{caseId}_{\CC}, 0) &= c,\\
\mathit{oaval}(p_p, \texttt{name}_p, 0) &= p,\\
\mathit{oaval}(\mathit{oc}_{c,p}, \texttt{caseId}_{\OC}, 0) &= c,\\
\mathit{oaval}(\mathit{oc}_{c,p}, \texttt{participant}, 0) &= p,\\
\mathit{oaval}(m_e, \texttt{sender}, 0) &= \mathit{from}(e) \quad \text{whenever } \mathit{from}(e) \text{ is defined},\\
\mathit{oaval}(m_e, \texttt{receiver}, 0) &= \mathit{to}(e) \quad \text{whenever } \mathit{to}(e) \text{ is defined},
\end{align*}
for $c \in \mathcal{C}_L$, $p \in \mathcal{P}_L$, every pair $(c,p)$ with $\mathit{oc}_{c,p} \in O_{\OC}$, and $e \in S_L \cup R_L$; $\mathit{oaval}$ is undefined elsewhere (in particular, on the \texttt{sender}/\texttt{receiver} attribute of $m_e$ when the source leaves that endpoint unrecorded, Normalization paragraph above). The \texttt{sender}/\texttt{receiver} attributes deliberately duplicate the \emph{from}/\emph{to} O2O relations of M7 whenever both are defined; criterion P1.3 checks their agreement and reports any $m_e$ for which one is defined and not the other as a genuine inconsistency, as opposed to an endpoint the source never recorded.

\paragraph{E2O relations (M6)} With qualifiers
\[
\begin{gathered}
Q_{\mathit{E2O}} = \{\emph{in\_collaboration}, \emph{in\_orchestration}, \emph{in\_participant},\\ \emph{send}, \emph{receive}\} \subseteq \mathbb{U}_{\mathit{qual}}:
\end{gathered}
\]
\begin{align*}
\mathit{E2O} = {} & \{\, (\mu_E(e), \emph{in\_collaboration}, \mathit{cc}_{\mathit{case}(e)}) : e \in E_L \,\}\\
\cup {} & \{\, (\mu_E(e), \emph{in\_orchestration}, \mathit{oc}_{\mathit{case}(e),\mathit{part}(e)}) : e \in E_L \,\}\\
\cup {} & \{\, (\mu_E(e), \emph{in\_participant}, p_{\mathit{part}(e)}) : e \in E_L \,\}\\
\cup {} & \{\, (\mu_E(e), \emph{send}, m_e) : e \in S_L \,\}\\
\cup {} & \{\, (\mu_E(e), \emph{receive}, m_e) : e \in R_L \,\}.
\end{align*}
The \emph{in\_participant} relation makes the participant directly accessible from each event. Criterion~P1.6 guarantees that it agrees with the participant reached through the two-step \emph{in\_orchestration}--\emph{for\_participant} path.

\paragraph{O2O relations (M7)}
With qualifiers $Q_{\mathit{O2O}} = \{\emph{part\_of}, \emph{for\_participant},\\ \emph{from}, \emph{to}, \emph{exchanged\_in}\} \subseteq \mathbb{U}_{\mathit{qual}}$:
\begin{align*}
\mathit{O2O} = {} & \{\, (\mathit{oc}_{c,p}, \emph{part\_of}, \mathit{cc}_c) : \mathit{oc}_{c,p} \in O_{\OC} \,\}\\
\cup {} & \{\, (\mathit{oc}_{c,p}, \emph{for\_participant}, p_p) : \mathit{oc}_{c,p} \in O_{\OC} \,\}\\
\cup {} & \{\, (m_e, \emph{from}, p_{\mathit{from}(e)}) : e \in S_L \cup R_L,\ \mathit{from}(e) \text{ defined} \,\}\\
\cup {} & \{\, (m_e, \emph{to}, p_{\mathit{to}(e)}) : e \in S_L \cup R_L,\ \mathit{to}(e) \text{ defined} \,\}\\
\cup {} & \{\, (m_e, \emph{exchanged\_in}, \mathit{cc}_{\mathit{case}(e)}) : e \in S_L \cup R_L \,\}.
\end{align*}
The standard places no constraint on $\mathbb{U}_{\mathit{qual}}$ beyond being a set of strings; keeping the E2O and O2O qualifier vocabularies disjoint is a presentation choice of this mapping, not a requirement of OCEL~2.0.
\end{definition}

We now state that the construction is well-defined and satisfies the consistency criteria P1.1--P1.6 in Section~\ref{sec:mapping}.

\begin{lemma}[Well-definedness]\label{lem:app-welldef}
For every extended collaborative log $L$, well formed or not, the structure $\mu(L)$ of Definition~\ref{def:app-mapping} is an object-centric event log in the sense of Definition~\ref{def:app-ocel}. Hence, $\mu$ produces logs that conform to the OCEL~2.0 metamodel.
\end{lemma}

We sketch the argument. The naming functions are injective and have pairwise disjoint ranges, so $\mathit{objtype}$ is a well-defined total function --- on participant objects because $\rho$ is defined on all of $\mathcal{P}$ --- and source identifiers, which are stored as attribute values through the fixed encodings above, are never reused as objects ($\mathbb{U}_{\mathit{obj}} \cap \mathbb{U}_{\mathit{val}} = \emptyset$); the same holds of the participant identifiers that M2 additionally encodes as object types, since $\mathbb{U}_{\mathit{otype}}$ is disjoint from both universes. The functions $\mathit{evtype}$ and $\mathit{time}$ are total because $\mathit{act}$ and $\mathit{time}_L$ are total and $\mu_E$ is a bijection onto $E$. Both $\mathit{eaval}$ and $\mathit{oaval}$ assign at most one value per argument --- for $\mathit{eaval}$, the defining clauses target pairwise distinct attributes because the residual names $x_a$ are disjoint from the reserved names $\texttt{participant}_a$, $\texttt{elemType}_a$, $\texttt{fromParticipant}_a$, and $\texttt{toParticipant}_a$ (Representation of source domains above) --- and pair each event or object with an attribute of its own type, so the typing constraints of Definition~\ref{def:app-ocel} hold; for the participant objects this is the equality $\mathit{oatype}(\texttt{name}_p) = \rho(p) = \mathit{objtype}(p_p)$, which is why the name attribute is indexed by $p$ rather than shared. Moreover, $\mathit{oaval}$'s partiality on $m_e$'s \texttt{sender}/\texttt{receiver} attribute exactly tracks that of $\mathit{from}(e)$/$\mathit{to}(e)$ (Normalization paragraph above), which need not both be defined for $e \in S_L\cup R_L$. Finally, every triple in $\mathit{E2O}$ and $\mathit{O2O}$ has its components in $E$, $\mathbb{U}_{\mathit{qual}}$, and $O$, and every target object is created by M1--M4; in particular, the \emph{in\_participant} triples of M6 target the participant objects of M2, which are defined for every $\mathit{part}(e)$.

\paragraph{OCED Core conformance.} Lemma~\ref{lem:app-welldef} establishes conformance to OCEL~2.0's metamodel (Definition~\ref{def:app-ocel}), which varies from and extends the OCED-MM Base Model (\ref{sec:bg-ocpm}); \citep{OCEDstandard} names as the main differences a per-type attribute schema ($\mathit{eatype}$, $\mathit{oatype}$) and timestamped, rather than static, object attribute values, and observes that dropping both yields a reduced OCEL~2.0 metamodel similar to the OCED Core Model up to relationship naming ($\mathit{E2O}$/\emph{observes}, $\mathit{O2O}$/object relations). Neither addition does any work in $\mu(L)$: the per-type attribute convention (Representation of source domains above) is a disjointness device for keeping $\mathit{eaval}$ and $\mathit{oaval}$ single-valued, not a constraint the construction otherwise relies on, and $\mathit{oaval}$ is assigned only at the reference timestamp $0$ throughout M1--M4, so $\mu(L)$'s object attributes are already static and never exercise the timestamped-history capability. Forgetting $\mathit{eatype}$, $\mathit{oatype}$, and the timestamp argument of $\mathit{oaval}$ therefore leaves $\mu(L)$'s events, objects, attribute values, and $E2O$/$O2O$ relations intact --- i.e., it reduces $\mu(L)$ to an instance of that reduced metamodel, and hence of the OCED Core Model up to the naming of relationships. This reduction only ever \emph{drops} structure Definition~\ref{def:app-ocel} requires beyond the reduced metamodel, so the well-typing established by Lemma~\ref{lem:app-welldef} carries over a fortiori, and none of P1.1--P1.6 below reads $\mathit{eatype}$, $\mathit{oatype}$, or a non-zero timestamp in $\mathit{oaval}$; both therefore hold of the reduced instance as well. This is what grounds the OCED-conformance claimed of the mapping, as distinct from its OCEL~2.0 serialization, which additionally exercises the per-type schema for machine-checkable typing (\ref{app:formal}, throughout).

\begin{proposition}[Consistency]\label{prop:app-p1}
For every \emph{well-formed} extended collaborative log $L$ (Definition~\ref{def:app-r1}), the log $\mu(L)$ satisfies the consistency criteria P1.1--P1.6.
\end{proposition}

Well-formedness is used only by P1.3; P1.1, P1.2, P1.4, P1.5, and P1.6 hold for every $L$, well formed or not. The proposition is stated for well-formed logs uniformly for simplicity, and, by the Normalization paragraph above, well-formedness of $\mathit{from}/\mathit{to}$ is not an extra burden on the source format: applying $\mu$ to $\nu(L_0)$ discharges it by construction --- regardless of whether counterparty sides are recorded --- provided any own-side endpoint the source does record agrees with \texttt{collab:participant}, a premise that holds vacuously in the logs used in this work (Normalization paragraph above).

The criteria follow from the construction. P1.1 holds because $\mu_E$ is a bijection that preserves timestamps and event types (M5), and because every M8-preserved attribute value decodes to exactly the source value it is assigned from --- $\mathit{part}(e)$, $\mathit{elem}(e)$, $\mathit{from}(e)$/$\mathit{to}(e)$ when defined, and $D(e)(x)$ for $x\in\mathrm{dom}(D(e))$ --- since M8's defining equations assign each directly, under the Encoding convention above, with no lossy intermediate step. For P1.2, the \emph{in\_collaboration} relations of $\mathit{cc}_c$ are exactly $\{(\mu_E(e),\emph{in\_collaboration},\mathit{cc}_c):\mathit{case}(e)=c\}$, so its event image is the set of events of collaboration case $c$; if $\sigma_c=\langle e_1,\ldots,e_{n_c}\rangle$, these events taken in $\prec_{\mathit{ev}}$ order are $\langle\mu_E(e_1),\ldots,\mu_E(e_{n_c})\rangle$, because $\mu_E$ is order-preserving on $\prec_L$ (Identifier creation) and $\prec_L$ is total on the events of case $c$; hence $\sigma_c$ is reconstructed from the event-identifier order alone, without using timestamps to break ties, \emph{provided} the events of $\mathit{cc}_c$ are enumerated according to $\prec_{\mathit{ev}}$ (e.g., by an explicit sort or query ordering on the event identifier) --- Definition~\ref{def:app-ocel} gives $E$ no intrinsic enumeration order, so this is a requirement on how $\mu(L)$ is queried, not a free consequence of $\mu_E$ being order-preserving. For P1.3, each $m_e$ is related to exactly one communication event: to $\mu_E(e)$ by \emph{send} if $e\in S_L$, or by \emph{receive} if $e\in R_L$, and to no other event, since $m_e$ is created for the single event $e$. Conversely, each communication event $\mu_E(e)$, $e\in S_L\cup R_L$, is related by \emph{send}/\emph{receive} to exactly one \Msg{} object, namely $m_e$: the \emph{send} and \emph{receive} triples of M6 are indexed by $e$ itself, each $e$ contributes exactly one such triple, and $e\mapsto m_e$ and $\mu_E$ are injective. For each side of $m_e$, the O2O relation (M7), the object attribute (\texttt{sender}/\texttt{receiver}), and the preserved event attribute (\texttt{fromParticipant}/\texttt{toParticipant}, M8) are all conditioned on the same source value $\mathit{from}(e)$ (resp.\ $\mathit{to}(e)$) being defined, so the three carriers of one side are defined simultaneously --- all three or none --- and use the same source value whenever defined (Normalization paragraph above); hence no disagreement can arise from an endpoint the source never recorded. Source well-formedness gives $\mathit{part}(e)=\mathit{from}(e)$ when $e$ is a send and $\mathit{part}(e)=\mathit{to}(e)$ when $e$ is a receive, so the event participant is the sender of a send interaction and the receiver of a receive interaction. Finally, each $m_e$ carries exactly one \emph{exchanged\_in} triple, to $\mathit{cc}_{\mathit{case}(e)}$, since M7's \emph{exchanged\_in} triples are indexed uniquely by $e\in S_L\cup R_L$ and $e\mapsto m_e$ is injective; this is the same collaboration case as $\mu_E(e)$'s \emph{in\_collaboration} object (M6), so the two agree. For P1.4, the \emph{in\_orchestration} object of $\mu_E(e)$ is $\mathit{oc}_{\mathit{case}(e),\mathit{part}(e)}$, which M7 relates by \emph{part\_of} to $\mathit{cc}_{\mathit{case}(e)}$, its \emph{in\_collaboration} object, and by \emph{for\_participant} to $p_{\mathit{part}(e)}$ and to no other object, since M7 contributes one \emph{for\_participant} triple per \OC{} object; that object has type $\rho(\mathit{part}(e))$ and carries $\mathit{oaval}(p_{\mathit{part}(e)},\texttt{name}_{\mathit{part}(e)},0)=\mathit{part}(e)$, while the orchestration case carries $\mathit{oaval}(\mathit{oc}_{\mathit{case}(e),\mathit{part}(e)},\texttt{participant},0)=\mathit{part}(e)$. Its object type and its name attribute are therefore the images of one and the same source identifier under $\rho$ and under the value encoding, which is what the equality asserted by P1.4 means (Representation of source domains above). For P1.5, collaboration cases and orchestration cases are targeted by E2O relations, and every $m_e$ carries its \emph{send} or \emph{receive} relation. By definition, $\mathcal{P}_L=\mathit{part}(E_L)\cup\mathrm{ran}(\mathit{from})\cup\mathrm{ran}(\mathit{to})$ (Definition~\ref{def:app-r1}), so every $p\in\mathcal{P}_L$ either performs an event, in which case $p_p$ is targeted by an E2O \emph{in\_participant} relation, or is $\mathit{from}(e)$ or $\mathit{to}(e)$ for some $e\in S_L\cup R_L$, in which case $p_p$ is targeted by the corresponding O2O \emph{from} or \emph{to} relation of M7, which ranges over $S_L\cup R_L$, not only $S_L$. Hence, no object is orphaned, and this argument uses no hypothesis on $L$ beyond Definition~\ref{def:app-r1}. For P1.6, M6 relates $\mu_E(e)$ directly to $p_{\mathit{part}(e)}$ through \emph{in\_participant}, while M7 relates its \emph{in\_orchestration} object $\mathit{oc}_{\mathit{case}(e),\mathit{part}(e)}$ to the same participant through \emph{for\_participant}; hence both access paths coincide.

\paragraph{Type-independent access to participants.} Rule M2 makes $\TPa$ depend on the log, but the four qualifiers that reach participant objects do not. In $\mu(L)$, the relations qualified by \emph{in\_participant}, \emph{for\_participant}, \emph{from}, and \emph{to} target participant objects and nothing else (M6, M7), and, by the argument for P1.5, every participant object is the target of at least one of them. Hence
\begin{align*}
O_{\textsf{Participant}} = {} & \{\, o : (\varepsilon, \emph{in\_participant}, o) \in \mathit{E2O} \,\}\\
\cup {} & \{\, o : (o', q, o) \in \mathit{O2O},\ q \in \{\emph{for\_participant}, \emph{from}, \emph{to}\} \,\},
\end{align*}
so participant objects are addressable without naming any member of $\TPa$. Criteria P1.1--P1.3, P1.5, and P1.6, together with the accessors used by the prediction targets of Section~\ref{sec:tasks}, read objects and qualifiers, never object types, and are therefore insensitive to whether participants are typed one per identifier (M2) or gathered under a single type carrying the identifier as an attribute value. Criterion~P1.4 is the exception, and deliberately so: it asserts that an \OC{} object's \texttt{participant} attribute and the participant object it is \emph{for\_participant}-related to encode the same source identifier, and, under M2, one of the two carriers of that identifier \emph{is} the object type $\rho(\mathit{part}(e))$, which the argument above reads. Under a single \textsf{Participant}{} type P1.4 would not fail but would change form, reducing to a statement about the \texttt{name} attribute alone, since the type would no longer encode the identifier.

% ==================
\section{Formal definitions of the collaborative prediction tasks}\label{app:tasks}
% ==================
This appendix provides the formal counterparts to the object-centric prediction tasks of Sect.~\ref{sec:tasks}. All definitions are stated over $\mu(L)$ (Definition~\ref{def:app-mapping}), for a \emph{well-formed} source log $L$ (Definition~\ref{def:app-r1}; the normalization $\nu$ of \ref{app:formal} discharges well-formedness under the premise stated there), at the \CC{} viewpoint, using the accessor functions of Definition~\ref{def:accessors}. For $\sigma_c=\langle e_1,\ldots,e_{n_c}\rangle$, the collaboration-case event sequence is the transported sequence
\[
\bar{\sigma}_c=\langle\varepsilon_1,\ldots,\varepsilon_{n_c}\rangle=\mu_E(\sigma_c)=\langle\mu_E(e_1),\ldots,\mu_E(e_{n_c})\rangle,
\]
i.e., the events related to $\mathit{cc}_c$ by the \emph{in\_collaboration} qualifier taken in event-identifier order, which coincides with $\prec_L$ by the order preservation of $\mu_E$ (Identifier-creation paragraph of \ref{app:formal}); here $\mu_E$ is applied to a sequence componentwise. The distinction matters for the domain of the definitions below: $\sigma_c$ is a sequence of source events of $E_L$, whereas the accessors of Definition~\ref{def:accessors} are defined on events of $\mu(L)$, so every $\Theta_\tau$ below takes the object-centric prefix
\[
\mathit{hd}^k(\bar{\sigma}_c)=\langle\varepsilon_1,\ldots,\varepsilon_k\rangle=\mu_E(\mathit{hd}^k(\sigma_c)),
\]
for a prefix length $k \in [1,n_c-1]$, and never $\mathit{hd}^k(\sigma_c)$ itself. The two are interchangeable as \emph{observations}---each determines the other, since $\mu_E$ is injective and order-preserving---which is what criterion~P1.2 states, and which is why the source-level counterparts $\Theta_\tau^{\mathit{src}}$ of Proposition~\ref{prop:p2-equivalence} may be stated over $\mathit{hd}^k(\sigma_c)$ while these definitions are stated over $\mathit{hd}^k(\bar{\sigma}_c)$. No label refers to an encoding strategy, and $\bot$ denotes the designated outcome when no further target exists, as in Sect.~\ref{sec:tasks}, with two exceptions carried over from \cite{delgado2025predictcollab}: the Boolean tasks OB-P and OB-M return $\mathsf{true}$ or $\mathsf{false}$ and use no sentinel value at all, and NV-PaT returns $0$ rather than $\bot$ when no further target exists (Definition~\ref{def:pat}). Some definitions below are additionally parameterized by a participant $p \in \mathcal{P}_L$ or an activity $\hat a \in \mathcal{A}_L$; under the Encoding convention of \ref{app:formal}, such a parameter denotes the corresponding OCEL-native object or type ($p_p$, or $\hat a$'s image under the activity-label encoding) that the definition's body actually operates on, not the raw source identifier. Quantifying the parameter over $\mathcal{P}_L$/$\mathcal{A}_L$ is legitimate precisely because every element of these source-level sets has such a corresponding object or type by construction (M2, M5), so no parameter value can fail to transport.

\begin{definition}[Accessors over $\mu(L)$]\label{def:accessors}
For an event $\varepsilon \in E$ of $\mu(L)$, we define the following functions.
\begin{itemize}
  \item $\mathit{oc}(\varepsilon)$ and $\mathit{cc}(\varepsilon)$ are the unique \OC{} and \CC{} objects related to $\varepsilon$ by the \emph{in\_orchestration} and \emph{in\_collaboration} qualifiers, respectively (rule M6).
  \item $\mathit{pa}(\varepsilon)$ is the unique participant object directly related to $\varepsilon$ through the \emph{in\_participant} qualifier. By criterion~P1.6, it is also the unique object $o$ such that $(\mathit{oc}(\varepsilon), \emph{for\_participant}, o) \in \mathit{O2O}$.
  \item The direction predicates are defined by
  \[
  \begin{aligned}
  \mathit{snd}(\varepsilon) &\Longleftrightarrow \exists m.\ (\varepsilon, \emph{send}, m) \in \mathit{E2O},\\
  \mathit{rcv}(\varepsilon) &\Longleftrightarrow \exists m.\ (\varepsilon, \emph{receive}, m) \in \mathit{E2O}.
  \end{aligned}
  \]
  We write $\mathit{isMsg}(\varepsilon)=\mathit{snd}(\varepsilon)\vee\mathit{rcv}(\varepsilon)$. By M6, every send event is related to its own \Msg{} object by \emph{send} and every receive event to its own \Msg{} object by \emph{receive} (equivalently, $\mathit{snd}(\varepsilon)$ holds iff the preserved $\texttt{elemType}$ value of $\varepsilon$ is \textit{SendTask}, by rule~M8); in either case, $\mathit{msg}(\varepsilon)$ denotes that message object, which is defined for every message event.
  \item $\mathit{Msgs}(c) = \{\, m \in O : (m, \emph{exchanged\_in}, \mathit{cc}_c) \in \mathit{O2O} \,\}$ is the set of message objects of collaboration case $c$, and for $m \in \mathit{Msgs}(c)$, $\mathit{pos}(m)$ is the position in $\langle \varepsilon_1, \ldots, \varepsilon_{n_c} \rangle$ of the unique event related to $m$ by the \emph{send} or \emph{receive} qualifier.
\end{itemize}
\end{definition}

\paragraph{Orchestration case domain.} The naming function $(c,p) \mapsto \mathit{oc}_{c,p}$ of \ref{app:formal} (Identifier-creation paragraph) is total on $\mathcal{C} \times \mathcal{P}$, whereas rule M3 materializes $\mathit{oc}_{c,p}$ as an object of $\mu(L)$ only for the pairs $(c,p)$ such that $p$ performs at least one event in case $c$. The tasks below that take a participant parameter (NE-NMPa, NV-PaT, NV-NMPa, OB-P) quantify $p$ over the full participant set $\mathcal{P}_L$, which may include participants that never act in a given case $c$. We do not restrict the domain of $p$ to the case's own actors, and no term below is left without a referent when we do not: $\mathit{oc}_{c,p}$ names an element of $\mathbb{U}_{\mathit{obj}}$ for \emph{every} pair, so the equality $\mathit{oc}(\varepsilon_i) = \mathit{oc}_{c,p}$ is well formed whether or not the pair is materialized. What the non-materialized case removes is not the referent but its membership in $O_{\OC}$: the equality is then false for every $i$, since $\mathit{oc}(\varepsilon_i) \in O_{\OC}$ by rule M6 while $\mathit{oc}_{c,p} \notin O_{\OC}$ by M3. The orchestration case of such a pair is therefore read as empty rather than undefined. This is the convention already forced by the values fixed for such cases: $I_p=\emptyset$ in NV-PaT (Definition~\ref{def:pat}) yields $0$, not $\bot$; the message counts of NV-NMPa (Definition~\ref{def:nmpacount}) and the existential of OB-P (Definition~\ref{def:obp}) both reduce to their empty-set values, $0$ and $\mathsf{false}$ respectively, over vacuous filtering conditions; and NE-NMPa (Definition~\ref{def:nmpa}) falls to its $\bot$ branch because no $\varepsilon_i$ satisfies $\mathit{oc}(\varepsilon_i)=\mathit{oc}_{c,p}$. No task definition below needs a separate case split for this situation.

\begin{definition}[Next event in the process, NE-NEPr]\label{def:nepr}
The next-event prediction in the process is the definition of a function $\Theta_{\mathrm{NEPr}}$ that takes the prefix $\mathit{hd}^k(\bar{\sigma}_c)$ and predicts the event type of the next event, that is,
\[
\Theta_{\mathrm{NEPr}}(\mathit{hd}^k(\bar{\sigma}_c)) = \mathit{evtype}(\varepsilon_{k+1}).
\]
\end{definition}

\begin{definition}[Next participant to act, NE-NPaA]\label{def:npaa}
The next-participant prediction is the definition of a function $\Theta_{\mathrm{NPaA}}$ that takes the prefix $\mathit{hd}^k(\bar{\sigma}_c)$ and predicts the participant of the next event, that is,
\[
\Theta_{\mathrm{NPaA}}(\mathit{hd}^k(\bar{\sigma}_c)) = \mathit{pa}(\varepsilon_{k+1}).
\]
\end{definition}

\begin{definition}[Next event and participant, NE-NEPa]\label{def:nepa}
The prediction of the next event together with the participant that performs it is the definition of a function $\Theta_{\mathrm{NEPa}}$ that takes the prefix $\mathit{hd}^k(\bar{\sigma}_c)$ and predicts the pair
\[
\Theta_{\mathrm{NEPa}}(\mathit{hd}^k(\bar{\sigma}_c)) = \bigl(\mathit{evtype}(\varepsilon_{k+1}),\ \mathit{pa}(\varepsilon_{k+1})\bigr).
\]
\cite{delgado2025predictcollab}'s Def.~5 ($\Theta_{ap}$) packages this same pair as a single string, the concatenation $\pi_{A\_P}(e_{k+1}) = a\_p$ (its Def.~2) -- a presentation choice suited to their tool's single-column model input, not part of the prediction's semantic content. \ref{app:tasks} states the target as the underlying pair; the equivalence of Proposition~\ref{prop:p2-equivalence} is with that pair, not with its concatenated string encoding (Table~\ref{tab:p2-bytask}).
\end{definition}

\begin{definition}[Next participant to exchange a message, NE-NPaM]\label{def:npam}
Following the send/receive taxonomy of \cite{delgado2025predictcollab} (type NE-NPaM predicts the next participant to \emph{send or receive} a message), this task is parameterized by a message direction $d \in \{\mathit{snd}, \mathit{rcv}\}$ and predicts the participant at the corresponding endpoint of the next message event in that direction. Let $q_{\mathit{snd}} = \emph{from}$ and $q_{\mathit{rcv}} = \emph{to}$ be the endpoint relation associated with each direction, and let $j = \min\{\, k < i \le n_c : d(\varepsilon_i) \,\}$. Then
\[
\Theta_{\mathrm{NPaM}}^{d}(\mathit{hd}^k(\bar{\sigma}_c)) =
\begin{cases}
o & \text{with } (\mathit{msg}(\varepsilon_j), q_d, o) \in \mathit{O2O}, \text{ if } j \text{ exists,}\\
\bot & \text{otherwise.}
\end{cases}
\]
The endpoint participant is read via the message object's \emph{from} relation (for a send) or \emph{to} relation (for a receive), rather than via the participant of the communication event. The direction-side endpoint is the message's \emph{own} side, which well-formedness makes total --- $\mathit{from}$ is defined on every send and $\mathit{to}$ on every receive, conditions (i)--(ii) of Definition~\ref{def:app-r1} --- and each endpoint relation of $\mathit{msg}(\varepsilon_j)$ has at most one target, so $o$ exists and is unique whenever $j$ exists. The two coincide by criterion P1.3, but reading the endpoint relation expresses the prediction at the message-object level, consistent with the message-centered viewpoint that the mapping makes explicit. The evaluation instantiates the send direction $d = \mathit{snd}$ (reading \emph{from}), which is implemented by Predict-Collab and keeps the prefixes identical to it; the receive direction $d = \mathit{rcv}$ is symmetric.

\cite{delgado2025predictcollab}'s Def.~6 ($\Theta_{pm}$) is formally send-only -- it locates the next event of \emph{sending} a message, not send-or-receive, despite Table 2's "send/receive" wording for this type -- and returns the concatenation $\pi_{ET\_P}(e_{k+j}) = et\_p$ (its Def.~2, e.g., "Sendtask\_gynecologist"), falling back to the literal string \textit{dummy} rather than a designated sentinel. $\Theta_{\mathrm{NPaM}}^{\mathit{snd}}$ coincides with $\Theta_{pm}$ up to two presentation choices: it returns the bare participant $o$ rather than the $et\_p$ concatenation, since $et$ is already fixed to \textit{SendTask} once the direction is fixed to $\mathit{snd}$ and so carries no additional information; and it uses $\bot$ uniformly in place of \textit{dummy}, consistent with every other task. The receive instantiation $d = \mathit{rcv}$ is an explicit generalization beyond $\Theta_{pm}$'s stated (send-only) scope, licensed by Table 2's own "send/receive" description of the type, not by Def.~6 itself.
\end{definition}

\begin{definition}[Next message in the process, NE-NMPr]\label{def:nmpr}
The prediction of the next message in the process is defined by a function $\Theta_{\mathrm{NMPr}}$ that takes the prefix $\mathit{hd}^k(\bar{\sigma}_c)$ and predicts the event type of the next message. Let $j = \min\{\, k < i \le n_c : \mathit{isMsg}(\varepsilon_i) \,\}$. Then
\[
\Theta_{\mathrm{NMPr}}(\mathit{hd}^k(\bar{\sigma}_c)) =
\begin{cases}
\mathit{evtype}(\varepsilon_j) & \text{if } j \text{ exists,}\\
\bot & \text{otherwise.}
\end{cases}
\]
Because the source logs carry no message-type attribute, the predicted message is identified by the event type of its message event, that is, by the send or receive activity label, in accordance with rule M4.
\end{definition}

\begin{definition}[Next message of a participant, NE-NMPa]\label{def:nmpa}
The prediction of the next message of a participant is the definition of a function $\Theta_{\mathrm{NMPa}}$ that takes the prefix $\mathit{hd}^k(\bar{\sigma}_c)$ and a participant $p \in \mathcal{P}_L$, and predicts the event type of the next message event in the orchestration case of $p$. Let $j = \min\{\, k < i \le n_c : \mathit{isMsg}(\varepsilon_i) \wedge \mathit{oc}(\varepsilon_i) = \mathit{oc}_{c,p} \,\}$. Then
\[
\Theta_{\mathrm{NMPa}}(\mathit{hd}^k(\bar{\sigma}_c), p) =
\begin{cases}
\mathit{evtype}(\varepsilon_j) & \text{if } j \text{ exists,}\\
\bot & \text{otherwise.}
\end{cases}
\]
\end{definition}

\begin{definition}[Process remaining time, NV-PrT]\label{def:prt}
The remaining-time prediction of the process is the definition of a function $\Theta_{\mathrm{PrT}}$ that takes the prefix $\mathit{hd}^k(\bar{\sigma}_c)$ and predicts the time from the current event to the end of the execution, that is,
\[
\Theta_{\mathrm{PrT}}(\mathit{hd}^k(\bar{\sigma}_c)) = \mathit{time}(\varepsilon_{n_c}) \ominus \mathit{time}(\varepsilon_k).
\]
\end{definition}

\begin{definition}[Participant remaining time, NV-PaT]\label{def:pat}
The remaining-time prediction of a participant is the definition of a function $\Theta_{\mathrm{PaT}}$ that takes the prefix $\mathit{hd}^k(\bar{\sigma}_c)$ and a participant $p \in \mathcal{P}_L$, and predicts the time until the last event in the orchestration case of $p$. Let $I_p = \{\, i \in [1, n_c] : \mathit{oc}(\varepsilon_i) = \mathit{oc}_{c,p} \,\}$ and, when $I_p \neq \emptyset$, let $z = \max I_p$. Then
\[
\Theta_{\mathrm{PaT}}(\mathit{hd}^k(\bar{\sigma}_c), p) =
\begin{cases}
\mathit{time}(\varepsilon_z) \ominus \mathit{time}(\varepsilon_k) & \text{if } I_p \neq \emptyset \text{ and } z > k,\\
\iota_\Delta(0) & \text{otherwise.}
\end{cases}
\]
Following \cite{delgado2025predictcollab} (Def.\ 9, $\Theta_{\mathrm{rtp}}$), the fallback value is the zero duration, not $\bot$: the source definition and its reference implementation both return $0$ whenever $p$ has no later event in $c$, without distinguishing $p$ acting earlier in $c$ from $p$ not acting in $c$ at all. On the object-centric side, this fallback is $\iota_\Delta(0)$, the image under the duration encoding of \ref{app:formal} of the source zero $0=t-t\in\Delta\mathbb{T}$; $\iota_\Delta(0)$ is well defined independently of which $t\in\mathbb{T}$ witnesses it, since subtraction-compatibility (\ref{app:formal}, "Representation of source domains") gives $\iota(t_1)\ominus\iota(t_1)=\iota(t_2)\ominus\iota(t_2)$ whenever $t_1-t_1=t_2-t_2\,(=0)$, and it lies in $\iota_\Delta(\Delta\mathbb{T})\subseteq\mathbb{U}_{\mathit{val}}$, the same codomain as the other branch, so $\mathrm{dec}_{\Delta\mathbb{T}}$ applies to it uniformly and $\mathrm{dec}_{\Delta\mathbb{T}}(\iota_\Delta(0))=0$. Def.~9 itself presupposes a non-empty subsequence $\hat\sigma$ of $p$'s events (it is stated as $\langle e_a,\ldots,e_y,e_z\rangle$, implicitly requiring $e_a$ to exist), so it does not, strictly speaking, cover a participant $p$ that never acts in $c$ at all; $I_p=\emptyset$ above is accordingly an explicit extension of Def.~9's domain, not a literal instance of it, resolved to $0$ for consistency with Def.~9's own fallback value and with Predict-Collab's reference implementation, which returns $0$ uniformly regardless of whether $p$ acted earlier in $c$ or never acted in $c$ at all.
\end{definition}

\begin{definition}[Time until next event, NV-TNE]\label{def:tne}
The prediction of the time until the next event is the definition of a function $\Theta_{\mathrm{TNE}}$ that takes the prefix $\mathit{hd}^k(\bar{\sigma}_c)$ and predicts the time from the current event to the next one, that is,
\[
\Theta_{\mathrm{TNE}}(\mathit{hd}^k(\bar{\sigma}_c)) = \mathit{time}(\varepsilon_{k+1}) \ominus \mathit{time}(\varepsilon_k).
\]
\cite{delgado2025predictcollab}'s Def.~10 ($\Theta_t$) states this target as the absolute timestamp $\pi_T(e_{k+1})$, with no subtraction. Read literally, that formula returns an absolute timestamp, whereas the type is \emph{named} as a duration ("time until next event") and its own worked example reports durations (its Table~11: a few seconds within a 31-second case, not epoch timestamps). The published type is therefore internally inconsistent between its name, its worked example and its formula, and we resolve it toward the duration $\pi_T(e_{k+1}) - \pi_T(e_k)$, matching $\Theta_{\mathrm{TNE}}$ above, rather than adjudicate which of the three the authors intended. Proposition~\ref{prop:p2-equivalence} is accordingly established against this resolved reading, which is why NV-TNE is one of the five documented adaptations rather than a literal restatement.
\end{definition}

\begin{definition}[Time until next message, NV-TNM]\label{def:tnm}
Like NE-NPaM and following the send/receive taxonomy of \cite{delgado2025predictcollab} (type NV-TNM is the time until the next message to \emph{send or receive}), this task is parameterized by a message direction $d \in \{\mathit{snd}, \mathit{rcv}\}$ and predicts the time until the next message event in that direction. Let $j = \min\{\, k < i \le n_c : d(\varepsilon_i) \,\}$. Then
\[
\Theta_{\mathrm{TNM}}^{d}(\mathit{hd}^k(\bar{\sigma}_c)) =
\begin{cases}
\mathit{time}(\varepsilon_j) \ominus \mathit{time}(\varepsilon_k) & \text{if } j \text{ exists,}\\
\bot & \text{otherwise.}
\end{cases}
\]
As for NE-NPaM, the evaluation instantiates the send direction $d = \mathit{snd}$, matching the direction implemented by Predict-Collab; the receive instantiation is symmetric.

\cite{delgado2025predictcollab}'s Def.~11 ($\Theta_{tm}$) is, like Def.~6, formally send-only despite Table 2's "send/receive" wording, and states the fallback as the literal value $1$ rather than a designated sentinel, with the existence condition written as $k+s<n$ for the position $s$ of the next sending event. Read literally, this excludes a next send located at the very last position of the trace (index $n$): if the only remaining send is $e_n$ itself, Def.~11 falls to the sentinel branch, whereas $\Theta_{\mathrm{TNM}}^{d}$ above returns a value whenever $j$ exists, including $j=n_c$. We do not carry over this restriction, since nothing in Def.~11's prose or Table 2's "time until next message" description motivates treating a send at the last position differently from one earlier in the trace, and we read it as an artifact of the reference implementation's loop bounds rather than a deliberate exclusion. $\Theta_{\mathrm{TNM}}^{\mathit{snd}}$ coincides with $\Theta_{tm}$ up to that boundary case and its sentinel: it uses $\bot$ uniformly, in place of the value $1$, which is not a time unit consistent with the rest of the type (a duration of exactly one second, minute, hour, or day, depending on the selected \texttt{TimeF}, would be indistinguishable from this "no next message" case under Def.~11's own convention). The receive instantiation $d = \mathit{rcv}$ is, as for NE-NPaM, an explicit generalization beyond $\Theta_{tm}$'s stated scope.
\end{definition}

The two count-based types are defined over $\mathit{Msgs}(c)$. A message object is counted as remaining at the cut-off $k$ when its communication event has not yet occurred, that is, when $\mathit{pos}(m) > k$.

\begin{definition}[Remaining messages of the process, NV-NMPr]\label{def:nmprcount}
The prediction of the number of remaining messages of the process is the definition of a function $\Theta_{\mathrm{NMPr}\#}$ that takes the prefix $\mathit{hd}^k(\bar{\sigma}_c)$ and predicts the number of message objects of the collaboration case whose communication event lies beyond the cut-off, that is,
\[
\Theta_{\mathrm{NMPr}\#}(\mathit{hd}^k(\bar{\sigma}_c)) = \bigl|\{\, m \in \mathit{Msgs}(c) : \mathit{pos}(m) > k \,\}\bigr|.
\]
\end{definition}

\begin{definition}[Remaining messages of a participant, NV-NMPa]\label{def:nmpacount}
The prediction of the number of remaining messages of a participant is the definition of a function $\Theta_{\mathrm{NMPa}\#}$ that takes the prefix $\mathit{hd}^k(\bar{\sigma}_c)$ and a participant $p \in \mathcal{P}_L$, and predicts the number of message objects of $p$'s orchestration case whose communication event lies beyond the cut-off, that is,
\[
\Theta_{\mathrm{NMPa}\#}(\mathit{hd}^k(\bar{\sigma}_c), p) = \bigl|\{\, m \in \mathit{Msgs}(c) : \mathit{oc}(\varepsilon_{\mathit{pos}(m)}) = \mathit{oc}_{c,p} \wedge \mathit{pos}(m) > k \,\}\bigr|.
\]
\end{definition}

These two types follow the remaining-message reading of \cite{delgado2025predictcollab}; the total-message variant is obtained by removing the $\mathit{pos}(m) > k$ condition, and send-only or receive-only variants by additionally requiring $\mathit{snd}(\varepsilon_{\mathit{pos}(m)})$ or $\mathit{rcv}(\varepsilon_{\mathit{pos}(m)})$. Quantifying over $\mathit{Msgs}(c)$ makes these predictions object-centric as a matter of formulation rather than of computation: rule~M4 puts the message objects of a collaboration case in bijection with its communication events, so counting either yields the same value, and the complementary set of message events already observed at the cut-off is available as a feature for any task.

\begin{definition}[Participation of a participant, OB-P]\label{def:obp}
The prediction of whether a participant will participate is the definition of a function $\Theta_{\mathrm{OBP}}$ that takes the prefix $\mathit{hd}^k(\bar{\sigma}_c)$ and a participant $p \in \mathcal{P}_L$, and predicts whether $p$ acts after the cut-off, that is,
\[
\Theta_{\mathrm{OBP}}(\mathit{hd}^k(\bar{\sigma}_c), p) =
\begin{cases}
\mathsf{true} & \text{if } \exists\, i : k < i \le n_c \text{ and } \mathit{pa}(\varepsilon_i) = p_p,\\
\mathsf{false} & \text{otherwise.}
\end{cases}
\]
Equivalently, $\Theta_{\mathrm{OBP}}$ predicts whether at least one event beyond the cut-off is related to $\mathit{oc}_{c,p}$ by \emph{in\_orchestration}, which expresses the same outcome at the orchestration-case level. The two readings agree on a $p$ that never acts in $c$ as well: by criterion P1.6 the participant reached from an event coincides with the one its orchestration case is \emph{for\_participant}-related to, and the non-materialized $\mathit{oc}_{c,p}$ of the Orchestration-case-domain paragraph above is related to no event, so both branches yield $\mathsf{false}$.
\end{definition}

\begin{definition}[Occurrence of a particular message, OB-M]\label{def:obm}
The prediction of whether a particular message will occur is the definition of a function $\Theta_{\mathrm{OBM}}$ that takes the prefix $\mathit{hd}^k(\bar{\sigma}_c)$ and a message-activity label $\hat{a} \in \mathcal{A}_L$, and predicts whether a message of that kind, sent or received, occurs after the cut-off, that is,
\[
\Theta_{\mathrm{OBM}}(\mathit{hd}^k(\bar{\sigma}_c), \hat{a}) =
\begin{cases}
\mathsf{true} & \text{if } \exists\, m \in \mathit{Msgs}(c) \text{ such that}\\
& \mathit{evtype}(\varepsilon_{\mathit{pos}(m)}) = \hat{a} \wedge \mathit{pos}(m) > k,\\
\mathsf{false} & \text{otherwise,}
\end{cases}
\]
where $\varepsilon_{\mathit{pos}(m)}$ is the communication event of $m$, and $\hat{a}$ on the right of $=$ denotes its image under the activity-label encoding of \ref{app:formal} ("Representation of source domains"), so that the comparison is between two elements of $\mathbb{U}_{\mathit{etype}}$, not across the disjoint universes $\mathbb{U}_{\mathit{etype}}$ and $\mathcal{A}_L$. As for NE-NMPr, the message type is determined by that event's event type, since rule~M4 stores no message-type attribute. The parameter ranges over all of $\mathcal{A}_L$ rather than over the message-activity labels $\mathcal{A}_L^{\mathit{msg}} = \mathit{act}(S_L \cup R_L)$ alone. The wider domain is harmless but not vacuous: on an $\hat{a} \in \mathcal{A}_L \setminus \mathcal{A}_L^{\mathit{msg}}$---an ordinary task label, which no message object's communication event carries---the existential is never satisfied and $\Theta_{\mathrm{OBM}}$ returns $\mathsf{false}$ at every cut-off. Restricting the parameter to $\mathcal{A}_L^{\mathit{msg}}$ would remove exactly these constantly-$\mathsf{false}$ instances and change nothing else; we keep $\mathcal{A}_L$ so that the parameter domain is the one the source log already provides, without a derived subset the extended collaborative log does not name.
\end{definition}

We now state that the accessors of Definition~\ref{def:accessors} agree on every event with the corresponding source-level reading on $L$, and use this to establish label equivalence against separately stated source-level task definitions, not against a restatement of \ref{app:tasks} on $L$. Those source-level definitions are themselves interpretations fixed here---literal for three types, adapted for five, and reconstructed from a one-line description for six---and the equivalence is stated relative to them, as the proposition's name records.

\begin{lemma}[Accessor invariance]\label{lem:accessor-invariance}
For every $e \in E_L$, writing $\varepsilon = \mu_E(e)$:
\begin{enumerate}[label=(\roman*)]
  \item $\mathit{evtype}(\varepsilon) = \mathit{act}(e)$ and $\mathit{time}(\varepsilon) = \mathit{time}_L(e)$;
  \item $\mathit{cc}(\varepsilon) = \mathit{cc}_{\mathit{case}(e)}$, $\mathit{oc}(\varepsilon) = \mathit{oc}_{\mathit{case}(e),\mathit{part}(e)}$, and $\mathit{pa}(\varepsilon) = p_{\mathit{part}(e)}$;
  \item $\mathit{snd}(\varepsilon) \Leftrightarrow e \in S_L$, and $\mathit{rcv}(\varepsilon) \Leftrightarrow e \in R_L$; when either holds, $\mathit{msg}(\varepsilon) = m_{e}$, with $\bigl(m_e,\emph{from},p_{\mathit{from}(e)}\bigr) \in \mathit{O2O}$ whenever $\mathit{from}(e)$ is defined and $\bigl(m_e,\emph{to},p_{\mathit{to}(e)}\bigr) \in \mathit{O2O}$ whenever $\mathit{to}(e)$ is defined (\ref{app:formal}, Normalization);
  \item for every collaboration case $c$, $\mathit{Msgs}(c) = \{\, m_{e'} : e' \in S_L \cup R_L,\ \mathit{case}(e') = c \,\}$, and for $m_{e'} \in \mathit{Msgs}(c)$, $\varepsilon_{\mathit{pos}(m_{e'})} = \mu_E(e')$.
\end{enumerate}
\end{lemma}

\begin{proof}
(i) is M5, restated. For (ii), M6 relates $\varepsilon$ to $\mathit{cc}_{\mathit{case}(e)}$ through \emph{in\_collaboration}, to $\mathit{oc}_{\mathit{case}(e),\mathit{part}(e)}$ through \emph{in\_orchestration}, and directly to $p_{\mathit{part}(e)}$ through \emph{in\_participant}. These relations are unique by construction; moreover, M7 maps the same orchestration case to $p_{\mathit{part}(e)}$ via \emph{for\_participant}, consistent with P1.6. For (iii), $O_{\Msg}$ is indexed by $S_L \cup R_L$ (M4), and M6 relates $\mu_E(e)$ to $m_e$ by \emph{send} exactly when $e \in S_L$ and by \emph{receive} exactly when $e \in R_L$; the \emph{from}/\emph{to} relations are given directly by M7, conditioned there on $\mathit{from}(e)$/$\mathit{to}(e)$ being defined (\ref{app:formal}, Normalization). For (iv), $\mathit{Msgs}(c)$ is the \emph{exchanged\_in}-image of $\mathit{cc}_c$, which M7 sets to $\{m_{e'} : e' \in S_L \cup R_L,\ \mathit{case}(e') = c\}$; $\mathit{pos}(m_{e'})$ is, by definition, the position of the unique event related to $m_{e'}$ by \emph{send} or \emph{receive}, which is $\mu_E(e')$ by (iii), with uniqueness following from the injectivity of $e'\mapsto m_{e'}$ and $\mu_E$ (Lemma~\ref{lem:app-welldef}). Finally, criterion~P1.2 and the order-preservation requirement on $\mu_E$ give $\bar{\sigma}_c=\mu_E(\sigma_c)=\langle\mu_E(e_1),\ldots,\mu_E(e_{n_c})\rangle$ when the events of $\mathit{cc}_c$ are enumerated in $\prec_{\mathit{ev}}$ order (\ref{app:tasks}); hence $\mu_E(e')$ occupies in $\bar{\sigma}_c$ exactly the position that $e'$ occupies in $\sigma_c$, proving $\varepsilon_{\mathit{pos}(m_{e'})}=\mu_E(e')$.
\end{proof}

\paragraph{Decoding.} Formally, each $\Theta_\tau$ takes values in the disjoint sum $\{p_p : p\in\mathcal{P}_L\} \uplus \mathit{evtype}(E) \uplus \iota_\Delta(\Delta\mathbb{T}) \uplus \mathbb{N} \uplus \mathbb{B} \uplus \{\bot\}$, or, for NE-NEPa alone, a pair drawn from two of these summands; $\bot$ is a fixed symbol distinct from every element of the other summands, so no genuine encoded value --- in particular, no duration in $\iota_\Delta(\Delta\mathbb{T}) \subseteq \mathbb{U}_{\mathit{val}}$ --- is ever mistaken for the no-target sentinel. $\mathrm{dec}$ below selects the summand-specific inverse by the task's declared return type (Table~\ref{tab:p2-bytask}), never by inspecting the value itself; this is well-defined precisely because the summands are disjoint, so the same value cannot admit two different decodings.

The naming functions $p \mapsto p_p$ of Definition~\ref{def:app-mapping} and the source-domain encodings of \ref{app:formal} (activity labels into $\mathbb{U}_{\mathit{etype}}$, timestamps into $\mathbb{U}_{\mathit{time}}$) are injective, hence invertible on their image. We write $\mathrm{dec}_{\mathsf{Pa}}\colon \{p_p : p \in \mathcal{P}_L\} \to \mathcal{P}_L$ for the inverse of $p \mapsto p_p$, $\mathrm{dec}_{\mathcal{A}}\colon \mathit{evtype}(E) \to \mathcal{A}_L$ for the inverse of the activity-label encoding, and $\mathrm{dec}_{\Delta\mathbb{T}}\colon \iota_\Delta(\Delta\mathbb{T}) \to \Delta\mathbb{T}$ for the inverse of the induced duration encoding $\iota_\Delta$ of \ref{app:formal} ("Representation of source domains"), where $\iota_\Delta(\Delta\mathbb{T}) \subseteq \mathbb{U}_{\mathit{val}}$; every time-based task returns a difference $\mathit{time}(\varepsilon_a) \ominus \mathit{time}(\varepsilon_b) \in \iota_\Delta(\Delta\mathbb{T})$, never an absolute timestamp, so $\mathrm{dec}_{\Delta\mathbb{T}}$ -- typed over durations, valued in $\mathbb{U}_{\mathit{val}}$, not over instants in $\mathbb{U}_{\mathit{time}}$ -- is the only timestamp-related decoding function applied below; no pointwise inverse of $\iota$ on $\mathbb{U}_{\mathit{time}}$ itself is needed or used. Booleans, natural numbers, and the sentinel $\bot$ are values of $\Theta_\tau$ itself, not attributes stored by $\mu$, so they are not re-encoded and need no decoding; we set $\mathrm{dec}(\bot) = \bot$ and extend $\mathrm{dec}$ componentwise to the pair returned by NE-NEPa. Applied to a task output, $\mathrm{dec}$ denotes whichever of $\mathrm{dec}_{\mathsf{Pa}}$, $\mathrm{dec}_{\mathcal{A}}$, or $\mathrm{dec}_{\Delta\mathbb{T}}$ matches that task's return type (Table~\ref{tab:p2-bytask}); it is the identity on Boolean and integer outputs.

\begin{proposition}[Label equivalence under the stated source-level interpretations, P2]\label{prop:p2-equivalence}
Let $L$ be a well-formed extended collaborative log (Definition~\ref{def:app-r1}). For each of the fourteen task types $\tau$ of \cite{delgado2025predictcollab} (Table~\ref{tab:p2-bytask}), let $\Theta_\tau$ be its object-centric definition (\ref{app:tasks}) and let $\Theta_\tau^{\mathit{src}}$ be the source-level definition fixed for it here---displayed below for the eleven types that need one, and, for the three restated unchanged from \cite{delgado2025predictcollab} (NE-NEPr, NE-NPaA, NV-PrT), identified by the citation in the last column of Table~\ref{tab:p2-bytask}---over the same applicable parameters: $p \in \mathcal{P}_L$ for NE-NMPa, NV-PaT, NV-NMPa, and OB-P; $\hat{a} \in \mathcal{A}_L$ for OB-M; $d \in \{\mathit{snd}, \mathit{rcv}\}$ for NE-NPaM and NV-TNM. Then, for every collaboration case $c$, every prefix $k \in [1, n_c{-}1]$, and every applicable parameter value,
\[
\mathrm{dec}\bigl(\Theta_\tau(\mathit{hd}^k(\bar{\sigma}_c), \ldots)\bigr) = \Theta_\tau^{\mathit{src}}(\mathit{hd}^k(\sigma_c), \ldots),
\]
where $\bar{\sigma}_c = \mu_E(\sigma_c)$ is the transported event sequence of \ref{app:tasks}, so that the two sides are applied to the two representations of the same prefix: the object-centric one on the left and the source one on the right. For the tasks with a no-target or degenerate branch --- NE-NPaM, NE-NMPr, NE-NMPa, and NV-TNM, whose fallback is the sentinel $\bot$, and NV-PaT, whose degenerate branch falls back to $0$ --- both sides take the same fallback value under exactly the same failing condition, and both sides take the same value in every other case, including the Boolean outcomes of OB-P and OB-M.
\end{proposition}

The qualification in the name is not decorative, and it delimits what P2 claims. Each $\Theta_\tau^{\mathit{src}}$ is a source-level definition \emph{fixed in this appendix}, and the fourteen fall into three tiers with respect to \cite{delgado2025predictcollab}: for three types (NE-NEPr, NE-NPaA, NV-PrT) it is a published formula restated unchanged in the notation of Definition~\ref{def:app-r1}; for five (marked $\dagger$ in Table~\ref{tab:p2-bytask}) it is a documented adaptation of a published formula, each departure being stated and justified under ``Departures from a literal restatement'' below; and for the remaining six \cite{delgado2025predictcollab} supplies no formula and no reference implementation at all, so $\Theta_\tau^{\mathit{src}}$ is reconstructed from the one-line description of its Table~2. P2 is therefore a claim about the interpretations displayed here, not about \cite{delgado2025predictcollab}'s formulas read in the abstract: for the last six types in particular, a different reading of the same one-line description would yield a different $\Theta_\tau^{\mathit{src}}$, and the equivalence would have to be re-established against it. What P2 rules out is a discrepancy introduced \emph{by the mapping}---the object-centric targets compute what their source-level readings compute---which is exactly the question RQ2 asks.

The well-formedness hypothesis in the statement plays the same role as in Proposition~\ref{prop:app-p1} (P1.1--P1.6, \ref{app:formal}): Lemma~\ref{lem:accessor-invariance}(iii) only guarantees the \emph{from}/\emph{to} O2O relations where $\mathit{from}(e)$/$\mathit{to}(e)$ are defined. This restriction has no bearing on any of the fourteen tasks, however: none reads a message's \emph{counterparty} endpoint. NE-NPaM reads only the \emph{own}-side relation for its instantiated direction ($q_{\mathit{snd}}=\emph{from}$ on a send event, $q_{\mathit{rcv}}=\emph{to}$ on a receive event; \ref{app:tasks}), which is always defined by well-formedness conditions (i)/(ii) regardless of whether the counterparty side is (\ref{app:formal}, Normalization); no other task reads \emph{from}/\emph{to} at all.

For eight of the fourteen types, $\Theta_\tau^{\mathit{src}}$ is grounded in the corresponding numbered definition of \cite{delgado2025predictcollab}'s own Appendix A (Table~\ref{tab:p2-bytask} gives the correspondence): for three of them (NE-NEPr, NE-NPaA, NV-PrT) it is that definition restated unchanged with the source notation of Definition~\ref{def:app-r1}, displayed below, while for the other five (marked $\dagger$) it is a documented adaptation of it, stated as an explicit formula in the paragraph ``Departures from a literal restatement'' below. Writing $\sigma_c = \langle e_1,\ldots,e_{n_c}\rangle$ ordered by $\prec_L$ as above, these three read
\[
\begin{aligned}
\Theta_{\mathrm{NEPr}}^{\mathit{src}}(\mathit{hd}^k(\sigma_c)) &= \mathit{act}(e_{k+1}); &&\text{(\cite{delgado2025predictcollab}, Def.~3, } \Theta_a\text{)}\\[4pt]
\Theta_{\mathrm{NPaA}}^{\mathit{src}}(\mathit{hd}^k(\sigma_c)) &= \mathit{part}(e_{k+1}); &&\text{(\cite{delgado2025predictcollab}, Def.~4, } \Theta_p\text{)}\\[4pt]
\Theta_{\mathrm{PrT}}^{\mathit{src}}(\mathit{hd}^k(\sigma_c)) &= \mathit{time}_L(e_{n_c}) - \mathit{time}_L(e_k). &&\text{(\cite{delgado2025predictcollab}, Def.~8, } \Theta_{rt}\text{)}
\end{aligned}
\]
Def.~8 additionally defines three auxiliary quantities, $fv_{t1}$, $fv_{t2}$, $fv_{t3}$ (the time since the last message sent, since the penultimate one, and since the case started), listed alongside $\Theta_{rt}$ as further arguments of the predictive model; they are input features for the model of \cite{delgado2025predictcollab}, not part of what the target itself computes, which reads only $\pi_T(e_n) - \pi_T(e_k)$ (its own notation) --- exactly $\Theta_{\mathrm{PrT}}^{\mathit{src}}$ above, under the accessor correspondence of Definition~\ref{def:app-r1}. For the remaining six, \cite{delgado2025predictcollab} states only the one-line description of its Table 2, with no accompanying formula or reference implementation (confirmed by inspecting both the published paper and the Predict-Collab source); $\Theta_\tau^{\mathit{src}}$ for these six is stated below directly from that description, over the source event sequence $\langle e_1,\ldots,e_{n_c}\rangle = \sigma_c$ ordered by $\prec_L$ (Definition~\ref{def:app-r1}), using only the source accessors $\mathit{case}$, $\mathit{act}$, $\mathit{time}_L$, $\mathit{part}$, and $\mathit{elem}$:
{\footnotesize
\[
\begin{aligned}
\Theta_{\mathrm{NMPr}}^{\mathit{src}}(\mathit{hd}^k(\sigma_c)) &=
\begin{cases}
\mathit{act}(e_j) & \text{if } j = \min\{\, k < i \le n_c : \mathit{elem}(e_i) \neq \textit{task} \,\} \text{ exists,}\\
\bot & \text{otherwise;}
\end{cases}\\[4pt]
\Theta_{\mathrm{NMPa}}^{\mathit{src}}(\mathit{hd}^k(\sigma_c), p) &=
\begin{cases}
\mathit{act}(e_j) & \text{if } j = \min\{\, i : \substack{k < i \le n_c,\ \mathit{elem}(e_i) \neq \textit{task}\\ \mathit{part}(e_i) = p} \,\} \text{ exists,}\\
\bot & \text{otherwise;}
\end{cases}\\[4pt]
\Theta_{\mathrm{NMPr}\#}^{\mathit{src}}(\mathit{hd}^k(\sigma_c)) &= \bigl|\{\, k < i \le n_c : \mathit{elem}(e_i) \neq \textit{task} \,\}\bigr|;\\[4pt]
\Theta_{\mathrm{NMPa}\#}^{\mathit{src}}(\mathit{hd}^k(\sigma_c), p) &= \bigl|\{\, i : \substack{k < i \le n_c,\ \mathit{elem}(e_i) \neq \textit{task}\\ \mathit{part}(e_i) = p} \,\}\bigr|;\\[4pt]
\Theta_{\mathrm{OBP}}^{\mathit{src}}(\mathit{hd}^k(\sigma_c), p) &= \bigl[\, \exists\, i : k < i \le n_c \text{ and } \mathit{part}(e_i) = p \,\bigr];\\[4pt]
\Theta_{\mathrm{OBM}}^{\mathit{src}}(\mathit{hd}^k(\sigma_c), \hat{a}) &= \bigl[\, \exists\, i : \substack{k < i \le n_c,\ \mathit{elem}(e_i) \neq \textit{task}\\ \mathit{act}(e_i) = \hat{a}} \,\bigr],
\end{aligned}
\]
}
where $[\,\cdot\,]$ denotes the Boolean value of the enclosed condition, and $\mathit{elem}(e_i) \neq \textit{task}$ (i.e., $e_i \in S_L \cup R_L$) is the source-level counterpart of $\mathit{isMsg}$, covering both directions, consistent with the "(send/receive)" and "sent/received" wording of Table 2 of \cite{delgado2025predictcollab}.

\paragraph{Scope: the \emph{remaining} reading only.} Table~2 of \cite{delgado2025predictcollab} names two readings for four of these types: "remaining/total" for NV-NMPa and NV-NMPr, and "remaining/duration" for NV-PaT and NV-PrT. P2 covers only the \emph{remaining} reading in all four cases: $\Theta_{\mathrm{NMPr\#}}^{\mathit{src}}$ and $\Theta_{\mathrm{NMPa\#}}^{\mathit{src}}$ above count events with $i>k$, and $\Theta_{\mathrm{PrT}}^{\mathit{src}}$, $\Theta_{\mathrm{PaT}}^{\mathit{src}}$ (Def.~8/9, \ref{app:tasks}) read a duration anchored at $e_k$. The \emph{total}/\emph{duration} reading is a legitimate target in its own right: its label is constant within a case, so it does not vary with $k$, but it remains predictable from each prefix, as total-cycle-time prediction is in case-centric PPM. Restricting P2 to the \emph{remaining} reading is therefore a scope decision, and it follows the source's own formalization: Def.~8 and Def.~9 resolve Table~2's ambiguity toward the prefix-dependent reading, and NV-NMPr/NV-NMPa are formalized here by the same resolution. The object-centric side matches this scope: $\Theta_{\mathrm{NMPr}\#}$, $\Theta_{\mathrm{NMPa}\#}$, $\Theta_{\mathrm{PrT}}$, $\Theta_{\mathrm{PaT}}$ (\ref{app:tasks}) are likewise stated only for the remaining reading, so P2 relates the same reading on both sides throughout.

\begin{proof}
The \CC{}-viewpoint prefix $\mathit{hd}^k(\bar{\sigma}_c) = \langle \varepsilon_1,\ldots,\varepsilon_k\rangle$ is $\mu_E(\mathit{hd}^k(\sigma_c))$, indexed in $\prec_L$ order, by criterion~P1.2 (\ref{app:formal}); this already accounts for the order dependency, via the order-preservation precondition placed on $\mu_E$ in \ref{app:formal} ("Identifier creation") together with the requirement that the events of $\mathit{cc}_c$ be enumerated in that order, so no separate order argument is needed here. In particular, the two sides index the same positions: $\varepsilon_i = \mu_E(e_i)$ for every $i \in [1,n_c]$. Throughout, Lemma~\ref{lem:accessor-invariance} gives that the selection conditions evaluate identically on $\varepsilon_i$ and on $e_i$: $\mathit{isMsg}(\varepsilon_i) \Leftrightarrow \mathit{elem}(e_i)\neq\textit{task}$ and $\mathit{snd}(\varepsilon_i) \Leftrightarrow e_i \in S_L$ (resp.\ $\mathit{rcv}$, $R_L$) by (iii), $\mathit{oc}(\varepsilon_i)=\mathit{oc}_{c,p} \Leftrightarrow \mathit{part}(e_i)=p$ by (ii), $\mathit{pa}(\varepsilon_i)=p_p \Leftrightarrow \mathit{part}(e_i)=p$ by (ii) and injectivity of $p\mapsto p_p$, and $\mathit{evtype}(\varepsilon_i)=\hat{a} \Leftrightarrow \mathit{act}(e_i)=\hat{a}$ by (i) and injectivity of the activity-label encoding. How the two sides reach their value from these conditions differs by task, in three ways.

\emph{Fixed index.} NE-NEPr, NE-NPaA, NE-NEPa, NV-TNE, and NV-PrT perform no search: they read position $k+1$, or, for NV-PrT, the trace's last position $n_c$. Both sides read the same position, which exists for every $k \in [1,n_c-1]$, so no fallback branch arises and the values agree by the accessor identities alone.

\emph{Searched index.} NE-NPaM, NE-NMPr, NE-NMPa, and NV-TNM minimize $j$, and NV-PaT maximizes $z$, over the same index range---$\{k{+}1,\ldots,n_c\}$ for the former four, $[1,n_c]$ for $I_p$ in NV-PaT---under conditions just shown to agree pointwise. Hence the optimum is attained at the same index on both sides, and it fails to exist on both sides simultaneously, in which case each returns its own fallback: $\bot$ for the four, and $0$ for NV-PaT, whose second branch condition $z>k$ likewise depends only on the agreeing index.

\emph{Set-level value.} The remaining four consume a set rather than an index. For the counts NV-NMPr and NV-NMPa, Lemma~\ref{lem:accessor-invariance}(iv) gives a bijection $e' \mapsto m_{e'}$ between the message events of case $c$ and $\mathit{Msgs}(c)$, with $\varepsilon_{\mathit{pos}(m_{e'})} = \mu_E(e')$, so $\mathit{pos}(m_{e'})>k$ holds exactly when $e'$ sits beyond position $k$; the object-level and event-level filtering conditions therefore select corresponding sets, which, being in bijection, have equal cardinality. For the Booleans OB-P and OB-M, both sides existentially quantify the agreeing conditions over the same range, so they yield $\mathsf{true}$ together and $\mathsf{false}$ together; the negative branch is a Boolean value, not a sentinel, and needs no fallback argument. Substituting the accessor identities of Lemma~\ref{lem:accessor-invariance} into the fixed composition that defines each $\Theta_\tau$, and applying $\mathrm{dec}$ to undo the participant-object, activity-label, or timestamp encoding wherever the composition returns one of those types (Table~\ref{tab:p2-bytask}), yields $\mathrm{dec}(\Theta_\tau) = \Theta_\tau^{\mathit{src}}$ pointwise; Boolean and integer outputs need no such step, since $\mathrm{dec}$ is the identity on them. For NE-NEPa, apply $\mathrm{dec}_{\mathcal{A}}$ and $\mathrm{dec}_{\mathsf{Pa}}$ componentwise to the pair. For the time-based types, $\mathit{time}(\varepsilon_a) \ominus \mathit{time}(\varepsilon_b) = \iota(\mathit{time}_L(e_a)) \ominus \iota(\mathit{time}_L(e_b)) = \iota_\Delta(\mathit{time}_L(e_a) - \mathit{time}_L(e_b))$ by (i) and the definition of $\iota_\Delta$ (\ref{app:formal}), so applying $\mathrm{dec}_{\Delta\mathbb{T}}$ recovers exactly $\mathit{time}_L(e_a) - \mathit{time}_L(e_b)$, rather than merely being some unrelated value.
\end{proof}

For compactness in Table~\ref{tab:p2-bytask}, let $\mathcal{A}_\mu=\mathit{evtype}(E)$, $\mathsf{Pa}_\mu=\{p_p:p\in\mathcal{P}_L\}$, and $\Delta_\mu=\iota_\Delta(\Delta\mathbb{T})$. For any such codomain $X$ with decoder $f$, write $X^\bot=X\uplus\{\bot\}$ and $f^\bot$ for the extension that maps $\bot$ to itself; $\mathrm{id}$ denotes the identity on Boolean and natural-number outputs.

\begin{widetable}
\centering
\caption{Parameters, output typing, accessors, and source-level counterparts of the tasks in Proposition~\ref{prop:p2-equivalence}. Types marked $\dagger$ are documented adaptations, not literal restatements, of the cited definition (see "Departures from a literal restatement" below).}
\label{tab:p2-bytask}
{\scriptsize \setlength{\tabcolsep}{2.5pt}
\begin{tabular}{@{}llllll@{}}
\toprule
Task & Parameter & Output / decoder & Accessors used & Lemma & $\Theta_\tau^{\mathit{src}}$ \\
\midrule
NE-NEPr  & --- & $\mathcal{A}_\mu / \mathrm{dec}_{\mathcal A}$ & $\mathit{evtype}$ & (i) & \cite{delgado2025predictcollab}, Def.~3 \\
NE-NPaA  & --- & $\mathsf{Pa}_\mu / \mathrm{dec}_{\mathsf{Pa}}$ & $\mathit{pa}$ & (ii) & \cite{delgado2025predictcollab}, Def.~4 \\
NE-NEPa$^\dagger$  & --- & $\mathcal{A}_\mu\!\times\!\mathsf{Pa}_\mu / (\mathrm{dec}_{\mathcal A},\mathrm{dec}_{\mathsf{Pa}})$ & $\mathit{evtype}$, $\mathit{pa}$ & (i), (ii) & \cite{delgado2025predictcollab}, Def.~5 \\
NE-NPaM$^\dagger$  & $d\in\{\mathit{snd},\mathit{rcv}\}$ & $\mathsf{Pa}_\mu^\bot / \mathrm{dec}_{\mathsf{Pa}}^\bot$ & $\mathit{msg}$, \emph{from}/\emph{to} & (iii) & \cite{delgado2025predictcollab}, Def.~6 \\
NE-NMPr  & --- & $\mathcal{A}_\mu^\bot / \mathrm{dec}_{\mathcal A}^\bot$ & $\mathit{isMsg}$, $\mathit{evtype}$ & (i), (iii) & Table~2 only (above) \\
NE-NMPa  & $p\in\mathcal P_L$ & $\mathcal{A}_\mu^\bot / \mathrm{dec}_{\mathcal A}^\bot$ & $\mathit{isMsg}$, $\mathit{oc}$, $\mathit{evtype}$ & (i)--(iii) & Table~2 only (above) \\
NV-PrT   & --- & $\Delta_\mu / \mathrm{dec}_{\Delta\mathbb T}$ & $\mathit{time}$ & (i) & \cite{delgado2025predictcollab}, Def.~8 \\
NV-PaT$^\dagger$   & $p\in\mathcal P_L$ & $\Delta_\mu / \mathrm{dec}_{\Delta\mathbb T}$ & $\mathit{oc}$, $\mathit{time}$ & (i), (ii) & \cite{delgado2025predictcollab}, Def.~9 \\
NV-TNE$^\dagger$   & --- & $\Delta_\mu / \mathrm{dec}_{\Delta\mathbb T}$ & $\mathit{time}$ & (i) & \cite{delgado2025predictcollab}, Def.~10 \\
NV-TNM$^\dagger$   & $d\in\{\mathit{snd},\mathit{rcv}\}$ & $\Delta_\mu^\bot / \mathrm{dec}_{\Delta\mathbb T}^\bot$ & $\mathit{time}$ & (i), (iii) & \cite{delgado2025predictcollab}, Def.~11 \\
NV-NMPr  & --- & $\mathbb N / \mathrm{id}$ & $\mathit{Msgs}$, $\mathit{pos}$ & (iv) & Table~2 only (above) \\
NV-NMPa  & $p\in\mathcal P_L$ & $\mathbb N / \mathrm{id}$ & $\mathit{Msgs}$, $\mathit{pos}$, $\mathit{oc}$ & (ii), (iv) & Table~2 only (above) \\
OB-P     & $p\in\mathcal P_L$ & $\mathbb B / \mathrm{id}$ & $\mathit{pa}$ & (ii) & Table~2 only (above) \\
OB-M     & $\hat a\in\mathcal A_L$ & $\mathbb B / \mathrm{id}$ & $\mathit{Msgs}$, $\mathit{pos}$, $\mathit{evtype}$ & (i), (iv) & Table~2 only (above) \\
\bottomrule
\end{tabular}
}
\end{widetable}

\paragraph{Coverage.} \cite{delgado2025predictcollab}'s own Appendix A formalizes nine predictions (its Definitions~3--11), of which eight coincide with one of our fourteen reformulated types, as listed above; its Definition~7 ("Next Participant that Will Send a Message (with Activity)") formalizes a ninth prediction that does not correspond to any of the fourteen types of its own Table~2 taxonomy (its Table~3 labels this row "NE-NMPr", which is inconsistent with Table~2's NE-NMPr, "Next message... in the process") and is out of scope here, since \ref{app:tasks} reformulates exactly the fourteen Table-2 types. The other six of the fourteen types have no formula or reference implementation in \cite{delgado2025predictcollab}: only the one-line description of its Table~2, which the formulas above formalize directly.

\paragraph{Departures from a literal restatement.} For five of the eight types with a formula in \cite{delgado2025predictcollab}'s Appendix A (marked $\dagger$ in Table~\ref{tab:p2-bytask}), $\Theta_\tau^{\mathit{src}}$ is not the cited definition unchanged, but a documented adaptation of it, detailed at each definition in \ref{app:tasks}:
\begin{itemize}
  \item \textbf{NE-NEPa} states the pair that Def.~5 packages as the string concatenation $a\_p$ (its Def.~2) -- a presentation choice for a single-column model input, not part of the target's semantic content. Table~2's row for this type reads "next event that is likely to occur in a participant", which, taken alone, would suggest a participant-parameterized target -- the reading we do adopt for the identically phrased NE-NMPa row, which has no formula. We follow Def.~5 here because it is \cite{delgado2025predictcollab}'s own formalization of this row, and it is unparameterized: it predicts the activity and the participant of the global next event. Stating the pair directly, rather than reproducing the $a\_p$ concatenation, is more than a presentation choice: an unescaped join is injective only if the separator is guaranteed absent from every activity label and participant identifier, a guarantee the source format does not give, so two distinct $(\mathit{evtype}(\varepsilon_{k+1}), \mathit{pa}(\varepsilon_{k+1}))$ pairs could in principle concatenate to the same string and be conflated into one class; the pair representation carries no such risk and never identifies two distinct source classes.
  \item \textbf{NE-NPaM} and \textbf{NV-TNM} generalize Def.~6 and Def.~11, which are formally send-only despite Table 2's "send/receive" wording for both types, to a direction parameter $d \in \{\mathit{snd},\mathit{rcv}\}$, and replace their ad hoc fallback values (the string \textit{dummy}, and the numeral $1$) with the sentinel $\bot$. Among the fourteen types, only NE-NPaM, NE-NMPr, NE-NMPa, and NV-TNM ever fall back to $\bot$: each searches for a next occurrence---a message endpoint, a message kind, or a message event---that may not exist in the remainder of the trace. The other ten need no such fallback: the next-event types (NE-NEPr, NE-NPaA, NE-NEPa) and NV-TNE always have a next event within a valid prefix; the count-based types (NV-NMPr, NV-NMPa) are always defined, returning $0$ where nothing remains; NV-PrT is likewise always defined for every $k<n_c$, but for a different reason: it reads the duration to the trace's fixed last event $\varepsilon_{n_c}$ rather than searching for a next occurrence that could fail to exist, so it never returns $0$ as a fallback; NV-PaT falls back to $0$ rather than $\bot$; and the Boolean tasks OB-P and OB-M return $\mathsf{true}$ or $\mathsf{false}$ and use no sentinel at all. NV-TNM additionally does not carry over Def.~11's literal exclusion of a next send at the trace's last position ($k+s<n$), read as an implementation artifact rather than a deliberate restriction (\ref{app:tasks}).
  \item \textbf{NV-TNE} corrects what we read as a typographical erratum in Def.~10, whose formula returns an absolute timestamp where the type's name ("time until next event") and worked example (its Table~11) indicate a duration.
  \item \textbf{NV-PaT} extends Def.~9's domain -- which presupposes a participant that acts at least once in $c$ -- to participants that never act in $c$ at all, resolved to the same fallback value $0$ that Def.~9 and Predict-Collab's reference implementation already use.
\end{itemize}
So that the right-hand side of Proposition~\ref{prop:p2-equivalence} is fully explicit for these five types as well, we state the adapted definitions with the source accessors of Definition~\ref{def:app-r1}, over $\sigma_c = \langle e_1,\ldots,e_{n_c}\rangle$ ordered by $\prec_L$, writing $E^{\mathit{snd}}_L = S_L$, $E^{\mathit{rcv}}_L = R_L$, $f_{\mathit{snd}} = \mathit{from}$, and $f_{\mathit{rcv}} = \mathit{to}$ (for $e_j \in E^d_L$, $f_d(e_j)$ is defined and equals $\mathit{part}(e_j)$ by well-formedness (i)/(ii)):
{\footnotesize
\[
\begin{aligned}
\Theta_{\mathrm{NEPa}}^{\mathit{src}}(\mathit{hd}^k(\sigma_c)) &= \bigl(\mathit{act}(e_{k+1}),\ \mathit{part}(e_{k+1})\bigr);\\[4pt]
\Theta_{\mathrm{NPaM}}^{\mathit{src}}(\mathit{hd}^k(\sigma_c), d) &=
\begin{cases}
f_d(e_j) & \text{if } j = \min\{\, k < i \le n_c : e_i \in E^d_L \,\} \text{ exists,}\\
\bot & \text{otherwise;}
\end{cases}\\[4pt]
\Theta_{\mathrm{PaT}}^{\mathit{src}}(\mathit{hd}^k(\sigma_c), p) &=
\begin{cases}
\mathit{time}_L(e_z) - \mathit{time}_L(e_k) & \text{if } I_p \neq \emptyset \text{ and } z = \max I_p > k,\\
0 & \text{otherwise,}
\end{cases}\\
&\qquad \text{where } I_p = \{\, i \in [1, n_c] : \mathit{part}(e_i) = p \,\};\\[4pt]
\Theta_{\mathrm{TNE}}^{\mathit{src}}(\mathit{hd}^k(\sigma_c)) &= \mathit{time}_L(e_{k+1}) - \mathit{time}_L(e_k);\\[4pt]
\Theta_{\mathrm{TNM}}^{\mathit{src}}(\mathit{hd}^k(\sigma_c), d) &=
\begin{cases}
\mathit{time}_L(e_j) - \mathit{time}_L(e_k) & \text{if } j = \min\{\, k < i \le n_c : e_i \in E^d_L \,\} \text{ exists,}\\
\bot & \text{otherwise.}
\end{cases}
\end{aligned}
\]
}
Proposition~\ref{prop:p2-equivalence} holds for these five types exactly as stated, against $\Theta_\tau^{\mathit{src}}$ as displayed above, not against the cited formula read literally; the send instantiation of NE-NPaM and NV-TNM additionally coincides with $\Theta_{pm}$ and $\Theta_{tm}$ on their own (send-only) domain, up to the sentinel and concatenation differences just noted and, for NV-TNM, up to Def.~11's literal exclusion of a next send at the trace's last position ($k+s<n$), which $\Theta_{\mathrm{TNM}}^{\mathit{src}}$ deliberately does not carry over.

\paragraph{Scope of the equivalence.} Proposition~\ref{prop:p2-equivalence} covers exactly the fourteen tasks of \cite{delgado2025predictcollab} reformulated in \ref{app:tasks}. Any target outside that taxonomy falls outside its scope for a structural reason rather than for lack of a proof: there is no source-log definition $\Theta_\tau^{\mathit{src}}$ against which an equivalence could be stated, since the extended collaborative log model provides none.

\paragraph{Generality.} Proposition~\ref{prop:p2-equivalence} applies to every \emph{well-formed} extended collaborative log $L$ (the hypothesis already carried by its statement, discharged by the normalization $\nu$ of \ref{app:formal} whenever any own-side endpoint the source does record already agrees with \texttt{collab:participant}, since $\nu$ backfills a missing own side but does not overwrite a recorded one; criterion~P1.3 surfaces the disagreement otherwise), regardless of the dataset, because the core mapping treats each send and each receive as independent message objects. The fourteen tasks read these objects through the accessors of Definition~\ref{def:accessors} alone. The result concerns label equivalence, not the recovery of any real-world correspondence between a send and its reception: the core mapping deliberately asserts no such correspondence (rule~M4), and targets that would require it lie outside the catalog (Sect.~\ref{sec:disc-limitations}).

% ==================
\section{Full-catalog experimentation results}\label{app:fullcatalog}
% ==================

Table \ref{tab:resultsAllFullI} and Table \ref{tab:resultsAllFullII} present the full-catalog experimental results for all 14 reformulated tasks across the five logs. Times are in seconds, except for BPIC~2013 (\textsuperscript{*}), which is in days. \textsuperscript{\dag}~model at or at/near a perfect score / near-zero error under two-decimal rounding; \textsuperscript{\ddag}~constant (degenerate) target, i.e., the trivial baseline rounds to a perfect score; \textsuperscript{\S}~model metric worse than the trivial baseline. Per-row extremes at the same precision: best in \textbf{bold}, worst \underline{underlined}. It is not a ranking; models may tie, and in half of the rows, the two marked models overlap within one standard deviation across folds.

\begin{widesidewaystable}
\setlength{\tabcolsep}{1.5pt}
\centering
\caption{Full-catalog experimentation results (part 1 of 2).}%
\label{tab:resultsAllFullI}
{\scriptsize\renewcommand{\baselinestretch}{1}\selectfont
\renewcommand{\arraystretch}{0.95}
\begin{tabular}{ll p{4.45cm} r @{\hspace{6pt}} l l l l l l}
\toprule
Task & Log & Anchor (parameter) & Samples & Baseline & RF & XGB & Transf. & LSTM & GNN \\
\midrule
\multirow{5}{*}{\begin{tabular}[c]{@{}l@{}}NE-NEPr\\ ($\mathrm{F1}_{\mathrm{macro}}$)\end{tabular}}
 & Healthcare & \texttt{CollaborationCase} & 1{,}350 & 0.01 & \underline{0.65}\,$\pm$\,0.02 & \textbf{0.69}\,$\pm$\,0.03 & 0.68\,$\pm$\,0.03 & 0.67\,$\pm$\,0.04 & 0.66\,$\pm$\,0.02 \\
 & Artificial1 & \texttt{CollaborationCase} & 700 & 0.03 & 0.54\,$\pm$\,0.04 & \underline{0.53}\,$\pm$\,0.03 & 0.57\,$\pm$\,0.04 & 0.57\,$\pm$\,0.04 & \textbf{0.58}\,$\pm$\,0.03 \\
 & Artificial5 & \texttt{CollaborationCase} & 2{,}260 & 0.00 & 0.39\,$\pm$\,0.01 & \underline{0.38}\,$\pm$\,0.02 & 0.41\,$\pm$\,0.02 & \textbf{0.43}\,$\pm$\,0.01 & 0.40\,$\pm$\,0.02 \\
 & Real4 & \texttt{CollaborationCase} & 1{,}700 & 0.01 & \underline{0.65}\,$\pm$\,0.01 & 0.66\,$\pm$\,0.02 & \textbf{0.69}\,$\pm$\,0.01 & \textbf{0.69}\,$\pm$\,0.00 & 0.68\,$\pm$\,0.01 \\
 & BPIC~2013 & \texttt{CollaborationCase} & 62{,}030 & 0.19 & \underline{0.38}\,$\pm$\,0.00 & \underline{0.38}\,$\pm$\,0.00 & 0.43\,$\pm$\,0.05 & \textbf{0.45}\,$\pm$\,0.06 & 0.44\,$\pm$\,0.06 \\
\midrule
\multirow{5}{*}{\begin{tabular}[c]{@{}l@{}}NE-NPaA\\ ($\mathrm{F1}_{\mathrm{macro}}$)\end{tabular}}
 & Healthcare & \texttt{Participant} & 1{,}350 & 0.17 & \underline{0.82}\,$\pm$\,0.03 & 0.86\,$\pm$\,0.02 & 0.87\,$\pm$\,0.02 & \textbf{0.88}\,$\pm$\,0.02 & 0.87\,$\pm$\,0.02 \\
 & Artificial1 & \texttt{Participant} & 700 & 0.34 & \underline{0.73}\,$\pm$\,0.04 & 0.74\,$\pm$\,0.02 & \textbf{0.79}\,$\pm$\,0.04 & 0.78\,$\pm$\,0.04 & 0.78\,$\pm$\,0.04 \\
 & Artificial5 & \texttt{Participant} & 2{,}260 & 0.18 & \underline{0.56}\,$\pm$\,0.01 & \textbf{0.59}\,$\pm$\,0.03 & 0.57\,$\pm$\,0.01 & 0.58\,$\pm$\,0.03 & \underline{0.56}\,$\pm$\,0.03 \\
 & Real4 & \texttt{Participant} & 1{,}700 & 0.23 & \underline{0.86}\,$\pm$\,0.02 & 0.89\,$\pm$\,0.01 & 0.90\,$\pm$\,0.02 & \textbf{0.91}\,$\pm$\,0.02 & 0.90\,$\pm$\,0.02 \\
 & BPIC~2013 & \texttt{Participant} & 62{,}030 & 0.04 & 0.07\,$\pm$\,0.00 & \underline{0.05}\,$\pm$\,0.00 & \textbf{0.12}\,$\pm$\,0.01 & \textbf{0.12}\,$\pm$\,0.02 & 0.09\,$\pm$\,0.01 \\
\midrule
\multirow{5}{*}{\begin{tabular}[c]{@{}l@{}}NE-NEPa\\ ($\mathrm{F1}_{\mathrm{macro}}$)\end{tabular}}
 & Healthcare & \texttt{Participant} & 1{,}350 & 0.01 & \underline{0.65}\,$\pm$\,0.02 & \textbf{0.69}\,$\pm$\,0.03 & 0.68\,$\pm$\,0.03 & 0.67\,$\pm$\,0.04 & 0.66\,$\pm$\,0.02 \\
 & Artificial1 & \texttt{Participant} & 700 & 0.03 & 0.54\,$\pm$\,0.04 & \underline{0.53}\,$\pm$\,0.03 & 0.57\,$\pm$\,0.04 & 0.57\,$\pm$\,0.04 & \textbf{0.58}\,$\pm$\,0.03 \\
 & Artificial5 & \texttt{Participant} & 2{,}260 & 0.00 & 0.39\,$\pm$\,0.01 & \underline{0.38}\,$\pm$\,0.02 & 0.41\,$\pm$\,0.02 & \textbf{0.43}\,$\pm$\,0.01 & 0.40\,$\pm$\,0.02 \\
 & Real4 & \texttt{Participant} & 1{,}700 & 0.01 & \underline{0.65}\,$\pm$\,0.01 & 0.67\,$\pm$\,0.01 & \textbf{0.69}\,$\pm$\,0.01 & \textbf{0.69}\,$\pm$\,0.01 & 0.68\,$\pm$\,0.01 \\
 & BPIC~2013 & \texttt{Participant} & 62{,}030 & 0.01 & 0.04\,$\pm$\,0.00 & \underline{0.03}\,$\pm$\,0.00 & 0.07\,$\pm$\,0.01 & \textbf{0.08}\,$\pm$\,0.00 & 0.07\,$\pm$\,0.01 \\
\midrule
\multirow{5}{*}{\begin{tabular}[c]{@{}l@{}}NE-NPaM\\ ($\mathrm{F1}_{\mathrm{macro}}$)\end{tabular}}
 & Healthcare & \texttt{Participant} & 1{,}250 & 0.25 & \underline{0.98}\,$\pm$\,0.02 & \underline{0.98}\,$\pm$\,0.02 & \textbf{0.99}\,$\pm$\,0.01 & \textbf{0.99}\,$\pm$\,0.01 & \underline{0.98}\,$\pm$\,0.01 \\
 & Artificial1 & \texttt{Participant} & 329 & 1.00\,\textsuperscript{\ddag} & \textbf{1.00}\,$\pm$\,0.00\,\textsuperscript{\dag} & \textbf{1.00}\,$\pm$\,0.00\,\textsuperscript{\dag} & \textbf{1.00}\,$\pm$\,0.00\,\textsuperscript{\dag} & \textbf{1.00}\,$\pm$\,0.00\,\textsuperscript{\dag} & \textbf{1.00}\,$\pm$\,0.00\,\textsuperscript{\dag} \\
 & Artificial5 & \texttt{Participant} & 1{,}661 & 0.23 & \underline{0.62}\,$\pm$\,0.05 & 0.65\,$\pm$\,0.05 & 0.68\,$\pm$\,0.05 & \textbf{0.69}\,$\pm$\,0.05 & \textbf{0.69}\,$\pm$\,0.04 \\
 & Real4 & \texttt{Participant} & 1{,}499 & 0.26 & \textbf{1.00}\,$\pm$\,0.00\,\textsuperscript{\dag} & \textbf{1.00}\,$\pm$\,0.00\,\textsuperscript{\dag} & \textbf{1.00}\,$\pm$\,0.00\,\textsuperscript{\dag} & \textbf{1.00}\,$\pm$\,0.00\,\textsuperscript{\dag} & \textbf{1.00}\,$\pm$\,0.00\,\textsuperscript{\dag} \\
 & BPIC~2013 & \texttt{Participant} & 16{,}499 & 0.04 & 0.07\,$\pm$\,0.01 & \underline{0.06}\,$\pm$\,0.00 & \textbf{0.13}\,$\pm$\,0.02 & \textbf{0.13}\,$\pm$\,0.03 & 0.09\,$\pm$\,0.00 \\
\midrule
\multirow{5}{*}{\begin{tabular}[c]{@{}l@{}}NE-NMPa\\ ($\mathrm{F1}_{\mathrm{macro}}$)\end{tabular}}
 & Healthcare & \begin{tabular}[c]{@{}l@{}}\texttt{Message}\\(Gynecologist)\end{tabular} & 1{,}068 & 0.08 & 0.85\,$\pm$\,0.02 & \textbf{0.86}\,$\pm$\,0.02 & \underline{0.84}\,$\pm$\,0.03 & \underline{0.84}\,$\pm$\,0.03 & 0.85\,$\pm$\,0.04 \\
 & Artificial1 & \begin{tabular}[c]{@{}l@{}}\texttt{Message}\\(PartyA)\end{tabular} & 549 & 0.33 & \textbf{0.62}\,$\pm$\,0.05 & \textbf{0.62}\,$\pm$\,0.05 & \underline{0.56}\,$\pm$\,0.04 & 0.60\,$\pm$\,0.04 & 0.60\,$\pm$\,0.06 \\
 & Artificial5 & \begin{tabular}[c]{@{}l@{}}\texttt{Message}\\(PartyB)\end{tabular} & 2{,}018 & 0.08 & 0.84\,$\pm$\,0.02 & 0.86\,$\pm$\,0.01 & 0.75\,$\pm$\,0.03 & \underline{0.70}\,$\pm$\,0.02 & \textbf{0.87}\,$\pm$\,0.03 \\
 & Real4 & \begin{tabular}[c]{@{}l@{}}\texttt{Message}\\(Zoo)\end{tabular} & 1{,}499 & 0.14 & \textbf{1.00}\,$\pm$\,0.00\,\textsuperscript{\dag} & \textbf{1.00}\,$\pm$\,0.00\,\textsuperscript{\dag} & \textbf{1.00}\,$\pm$\,0.00\,\textsuperscript{\dag} & \underline{0.99}\,$\pm$\,0.01 & \underline{0.99}\,$\pm$\,0.00 \\
 & BPIC~2013 & \begin{tabular}[c]{@{}l@{}}\texttt{Message}\\(Org line C)\end{tabular} & 14{,}538 & 1.00\,\textsuperscript{\ddag} & \textbf{1.00}\,$\pm$\,0.00\,\textsuperscript{\dag} & \textbf{1.00}\,$\pm$\,0.00\,\textsuperscript{\dag} & \textbf{1.00}\,$\pm$\,0.00\,\textsuperscript{\dag} & \textbf{1.00}\,$\pm$\,0.00\,\textsuperscript{\dag} & \textbf{1.00}\,$\pm$\,0.00\,\textsuperscript{\dag} \\
\midrule
\multirow{5}{*}{\begin{tabular}[c]{@{}l@{}}NE-NMPr\\ ($\mathrm{F1}_{\mathrm{macro}}$)\end{tabular}}
 & Healthcare & \texttt{Message} & 1{,}350 & 0.02 & \underline{0.70}\,$\pm$\,0.02 & \textbf{0.72}\,$\pm$\,0.03 & \textbf{0.72}\,$\pm$\,0.03 & \textbf{0.72}\,$\pm$\,0.06 & \underline{0.70}\,$\pm$\,0.06 \\
 & Artificial1 & \texttt{Message} & 549 & 0.11 & \textbf{0.58}\,$\pm$\,0.05 & 0.54\,$\pm$\,0.04 & \underline{0.52}\,$\pm$\,0.04 & \underline{0.52}\,$\pm$\,0.02 & 0.53\,$\pm$\,0.04 \\
 & Artificial5 & \texttt{Message} & 2{,}068 & 0.02 & \underline{0.64}\,$\pm$\,0.02 & 0.66\,$\pm$\,0.02 & 0.66\,$\pm$\,0.03 & \underline{0.64}\,$\pm$\,0.03 & \textbf{0.67}\,$\pm$\,0.02 \\
 & Real4 & \texttt{Message} & 1{,}600 & 0.07 & \textbf{1.00}\,$\pm$\,0.00\,\textsuperscript{\dag} & \textbf{1.00}\,$\pm$\,0.00\,\textsuperscript{\dag} & 0.99\,$\pm$\,0.00 & \underline{0.98}\,$\pm$\,0.01 & 0.99\,$\pm$\,0.00 \\
 & BPIC~2013 & \texttt{Message} & 18{,}481 & 1.00\,\textsuperscript{\ddag} & \textbf{1.00}\,$\pm$\,0.00\,\textsuperscript{\dag} & \textbf{1.00}\,$\pm$\,0.00\,\textsuperscript{\dag} & \textbf{1.00}\,$\pm$\,0.00\,\textsuperscript{\dag} & \textbf{1.00}\,$\pm$\,0.00\,\textsuperscript{\dag} & \textbf{1.00}\,$\pm$\,0.00\,\textsuperscript{\dag} \\
\midrule
\multirow{5}{*}{\begin{tabular}[c]{@{}l@{}}NV-PrT\\ (MAE)\end{tabular}}
 & Healthcare & \texttt{CollaborationCase} & 1{,}350 & 17.45 & 15.29\,$\pm$\,1.79 & 15.07\,$\pm$\,2.07 & \textbf{14.76}\,$\pm$\,1.52 & 14.78\,$\pm$\,2.14 & \underline{15.46}\,$\pm$\,2.46 \\
 & Artificial1 & \texttt{CollaborationCase} & 700 & 8.59 & \underline{4.65}\,$\pm$\,0.68 & 4.40\,$\pm$\,0.68 & 3.99\,$\pm$\,0.74 & \textbf{3.98}\,$\pm$\,0.67 & 4.05\,$\pm$\,0.76 \\
 & Artificial5 & \texttt{CollaborationCase} & 2{,}260 & 33.57 & \underline{11.15}\,$\pm$\,0.64 & 10.70\,$\pm$\,0.74 & \textbf{10.07}\,$\pm$\,0.98 & 10.53\,$\pm$\,0.78 & 10.41\,$\pm$\,1.05 \\
 & Real4 & \texttt{CollaborationCase} & 1{,}700 & 19.43 & \underline{7.62}\,$\pm$\,0.41 & 7.05\,$\pm$\,0.55 & 7.03\,$\pm$\,0.40 & \textbf{6.61}\,$\pm$\,0.46 & 6.67\,$\pm$\,0.65 \\
 & BPIC~2013\textsuperscript{*} & \texttt{CollaborationCase} & 62{,}030 & 13.39 & \underline{18.79}\,$\pm$\,1.75\,\textsuperscript{\S} & 16.73\,$\pm$\,1.80\,\textsuperscript{\S} & 14.60\,$\pm$\,1.81\,\textsuperscript{\S} & 16.37\,$\pm$\,2.20\,\textsuperscript{\S} & \textbf{13.37}\,$\pm$\,1.98 \\
\bottomrule
\end{tabular}\\[2pt]
}
\end{widesidewaystable}

\begin{widesidewaystable}
\setlength{\tabcolsep}{1.5pt}
\centering
\caption{Full-catalog experimentation results (part 2 of 2).}%
\label{tab:resultsAllFullII}
{\scriptsize\renewcommand{\baselinestretch}{1}\selectfont
\renewcommand{\arraystretch}{0.95}
\begin{tabular}{ll p{4.45cm} r @{\hspace{6pt}} l l l l l l}
\toprule
Task & Log & Anchor (parameter) & Samples & Baseline & RF & XGB & Transf. & LSTM & GNN \\
\midrule
\multirow{5}{*}{\begin{tabular}[c]{@{}l@{}}NV-PaT\\ (MAE)\end{tabular}}
 & Healthcare & \begin{tabular}[c]{@{}l@{}}\texttt{OrchestrationCase}\\(Gynecologist)\end{tabular} & 1{,}350 & 12.85 & \underline{9.29}\,$\pm$\,0.72 & 9.01\,$\pm$\,0.92 & 8.40\,$\pm$\,0.75 & 8.33\,$\pm$\,1.13 & \textbf{8.32}\,$\pm$\,1.07 \\
 & Artificial1 & \begin{tabular}[c]{@{}l@{}}\texttt{OrchestrationCase}\\(PartyA)\end{tabular} & 700 & 8.57 & \underline{4.59}\,$\pm$\,0.72 & 4.38\,$\pm$\,0.71 & 3.95\,$\pm$\,0.73 & \textbf{3.93}\,$\pm$\,0.72 & \textbf{3.93}\,$\pm$\,0.75 \\
 & Artificial5 & \begin{tabular}[c]{@{}l@{}}\texttt{OrchestrationCase}\\(PartyB)\end{tabular} & 2{,}260 & 33.50 & \underline{10.59}\,$\pm$\,0.62 & 10.08\,$\pm$\,0.73 & 9.50\,$\pm$\,0.90 & 9.64\,$\pm$\,0.60 & \textbf{9.47}\,$\pm$\,1.00 \\
 & Real4 & \begin{tabular}[c]{@{}l@{}}\texttt{OrchestrationCase}\\(Zoo)\end{tabular} & 1{,}700 & 19.00 & \underline{6.79}\,$\pm$\,0.19 & 6.36\,$\pm$\,0.19 & 6.51\,$\pm$\,0.42 & \textbf{5.67}\,$\pm$\,0.41 & 5.81\,$\pm$\,0.37 \\
 & BPIC~2013\textsuperscript{*} & \begin{tabular}[c]{@{}l@{}}\texttt{OrchestrationCase}\\(Org line C)\end{tabular} & 62{,}030 & 7.86 & \underline{9.56}\,$\pm$\,0.85\,\textsuperscript{\S} & 8.53\,$\pm$\,0.87\,\textsuperscript{\S} & 7.97\,$\pm$\,1.06\,\textsuperscript{\S} & 8.55\,$\pm$\,0.83\,\textsuperscript{\S} & \textbf{7.37}\,$\pm$\,0.92 \\
\midrule
\multirow{5}{*}{\begin{tabular}[c]{@{}l@{}}NV-TNE\\ (MAE)\end{tabular}}
 & Healthcare & \texttt{CollaborationCase} & 1{,}350 & 1.88 & \underline{1.95}\,$\pm$\,0.08\,\textsuperscript{\S} & 1.85\,$\pm$\,0.07 & 1.78\,$\pm$\,0.08 & 1.81\,$\pm$\,0.03 & \textbf{1.77}\,$\pm$\,0.05 \\
 & Artificial1 & \texttt{CollaborationCase} & 700 & 2.15 & \underline{2.10}\,$\pm$\,0.07 & 2.06\,$\pm$\,0.06 & 1.97\,$\pm$\,0.10 & 2.00\,$\pm$\,0.14 & \textbf{1.95}\,$\pm$\,0.09 \\
 & Artificial5 & \texttt{CollaborationCase} & 2{,}260 & 2.86 & 2.86\,$\pm$\,0.10\,\textsuperscript{\S} & 2.75\,$\pm$\,0.08 & \textbf{2.72}\,$\pm$\,0.07 & \underline{2.95}\,$\pm$\,0.04\,\textsuperscript{\S} & 2.78\,$\pm$\,0.08 \\
 & Real4 & \texttt{CollaborationCase} & 1{,}700 & 2.13 & \underline{2.29}\,$\pm$\,0.02\,\textsuperscript{\S} & 2.12\,$\pm$\,0.05 & 2.05\,$\pm$\,0.04 & 2.11\,$\pm$\,0.04 & \textbf{2.02}\,$\pm$\,0.07 \\
 & BPIC~2013\textsuperscript{*} & \texttt{CollaborationCase} & 62{,}030 & 1.47 & 2.44\,$\pm$\,0.18\,\textsuperscript{\S} & 2.24\,$\pm$\,0.16\,\textsuperscript{\S} & \underline{2.55}\,$\pm$\,0.20\,\textsuperscript{\S} & 2.10\,$\pm$\,0.11\,\textsuperscript{\S} & \textbf{1.46}\,$\pm$\,0.08 \\
\midrule
\multirow{5}{*}{\begin{tabular}[c]{@{}l@{}}NV-TNM\\ (MAE)\end{tabular}}
 & Healthcare & \texttt{Message} & 1{,}250 & 4.27 & \underline{3.17}\,$\pm$\,0.12 & 2.99\,$\pm$\,0.06 & 2.93\,$\pm$\,0.07 & \textbf{2.88}\,$\pm$\,0.06 & 2.95\,$\pm$\,0.07 \\
 & Artificial1 & \texttt{Message} & 329 & 2.52 & 2.24\,$\pm$\,0.07 & \underline{2.27}\,$\pm$\,0.15 & \textbf{2.07}\,$\pm$\,0.09 & 2.09\,$\pm$\,0.09 & 2.08\,$\pm$\,0.10 \\
 & Artificial5 & \texttt{Message} & 1{,}661 & 9.51 & 8.41\,$\pm$\,0.38 & \textbf{8.11}\,$\pm$\,0.30 & 8.28\,$\pm$\,0.30 & \underline{8.53}\,$\pm$\,0.38 & 8.29\,$\pm$\,0.32 \\
 & Real4 & \texttt{Message} & 1{,}499 & 8.18 & \underline{4.58}\,$\pm$\,0.32 & 4.33\,$\pm$\,0.30 & 4.03\,$\pm$\,0.25 & \textbf{3.81}\,$\pm$\,0.26 & 4.25\,$\pm$\,0.33 \\
 & BPIC~2013\textsuperscript{*} & \texttt{Message} & 16{,}499 & 2.85 & \underline{4.32}\,$\pm$\,0.49\,\textsuperscript{\S} & 3.96\,$\pm$\,0.56\,\textsuperscript{\S} & 3.73\,$\pm$\,0.51\,\textsuperscript{\S} & 4.04\,$\pm$\,0.50\,\textsuperscript{\S} & \textbf{2.84}\,$\pm$\,0.68 \\
\midrule
\multirow{5}{*}{\begin{tabular}[c]{@{}l@{}}NV-NMPr\\ (MAE)\end{tabular}}
 & Healthcare & \texttt{Message} & 1{,}350 & 2.82 & 2.23\,$\pm$\,0.18 & 2.23\,$\pm$\,0.20 & \textbf{2.19}\,$\pm$\,0.11 & 2.20\,$\pm$\,0.22 & \underline{2.36}\,$\pm$\,0.37 \\
 & Artificial1 & \texttt{Message} & 700 & 1.29 & \textbf{0.00}\,$\pm$\,0.00\,\textsuperscript{\dag} & 0.01\,$\pm$\,0.00 & \underline{0.18}\,$\pm$\,0.01 & 0.07\,$\pm$\,0.01 & 0.07\,$\pm$\,0.01 \\
 & Artificial5 & \texttt{Message} & 2{,}260 & 2.61 & 0.81\,$\pm$\,0.07 & \textbf{0.80}\,$\pm$\,0.06 & \underline{0.95}\,$\pm$\,0.03 & 0.94\,$\pm$\,0.03 & 0.91\,$\pm$\,0.05 \\
 & Real4 & \texttt{Message} & 1{,}700 & 1.76 & \textbf{0.00}\,$\pm$\,0.00\,\textsuperscript{\dag} & \textbf{0.00}\,$\pm$\,0.00\,\textsuperscript{\dag} & \underline{0.25}\,$\pm$\,0.01 & 0.08\,$\pm$\,0.00 & 0.20\,$\pm$\,0.01 \\
 & BPIC~2013 & \texttt{Message} & 62{,}030 & 1.39 & \underline{1.84}\,$\pm$\,0.11\,\textsuperscript{\S} & 1.72\,$\pm$\,0.11\,\textsuperscript{\S} & 1.45\,$\pm$\,0.12\,\textsuperscript{\S} & 1.63\,$\pm$\,0.13\,\textsuperscript{\S} & \textbf{1.42}\,$\pm$\,0.13\,\textsuperscript{\S} \\
\midrule
\multirow{5}{*}{\begin{tabular}[c]{@{}l@{}}NV-NMPa\\ (MAE)\end{tabular}}
 & Healthcare & \begin{tabular}[c]{@{}l@{}}\texttt{Message}\\(Gynecologist)\end{tabular} & 1{,}350 & 1.50 & 0.80\,$\pm$\,0.06 & 0.80\,$\pm$\,0.06 & \underline{0.82}\,$\pm$\,0.02 & \textbf{0.79}\,$\pm$\,0.08 & 0.81\,$\pm$\,0.11 \\
 & Artificial1 & \begin{tabular}[c]{@{}l@{}}\texttt{Message}\\(PartyA)\end{tabular} & 700 & 0.67 & \textbf{0.01}\,$\pm$\,0.00 & 0.02\,$\pm$\,0.00 & \underline{0.11}\,$\pm$\,0.01 & 0.03\,$\pm$\,0.01 & 0.04\,$\pm$\,0.01 \\
 & Artificial5 & \begin{tabular}[c]{@{}l@{}}\texttt{Message}\\(PartyB)\end{tabular} & 2{,}260 & 1.25 & 0.45\,$\pm$\,0.02 & \textbf{0.44}\,$\pm$\,0.02 & \underline{0.49}\,$\pm$\,0.02 & \underline{0.49}\,$\pm$\,0.02 & 0.46\,$\pm$\,0.02 \\
 & Real4 & \begin{tabular}[c]{@{}l@{}}\texttt{Message}\\(Zoo)\end{tabular} & 1{,}700 & 0.96 & \textbf{0.00}\,$\pm$\,0.00\,\textsuperscript{\dag} & \textbf{0.00}\,$\pm$\,0.00\,\textsuperscript{\dag} & \underline{0.15}\,$\pm$\,0.00 & 0.06\,$\pm$\,0.00 & 0.14\,$\pm$\,0.01 \\
 & BPIC~2013 & \begin{tabular}[c]{@{}l@{}}\texttt{Message}\\(Org line C)\end{tabular} & 62{,}030 & 0.51 & \underline{0.74}\,$\pm$\,0.04\,\textsuperscript{\S} & 0.69\,$\pm$\,0.04\,\textsuperscript{\S} & \textbf{0.56}\,$\pm$\,0.05\,\textsuperscript{\S} & 0.66\,$\pm$\,0.05\,\textsuperscript{\S} & 0.60\,$\pm$\,0.05\,\textsuperscript{\S} \\
\midrule
\multirow{5}{*}{\begin{tabular}[c]{@{}l@{}}OB-P\\ ($\mathrm{F1}_{\mathrm{macro}}$)\end{tabular}}
 & Healthcare & \begin{tabular}[c]{@{}l@{}}\texttt{Participant}\\(Gynecologist)\end{tabular} & 1{,}350 & 0.44 & \underline{0.92}\,$\pm$\,0.01 & 0.93\,$\pm$\,0.01 & \underline{0.92}\,$\pm$\,0.03 & 0.93\,$\pm$\,0.03 & \textbf{0.94}\,$\pm$\,0.01 \\
 & Artificial1 & \begin{tabular}[c]{@{}l@{}}\texttt{Participant}\\(PartyA)\end{tabular} & 700 & 0.60 & \underline{0.53}\,$\pm$\,0.07\,\textsuperscript{\S} & 0.61\,$\pm$\,0.07 & 0.94\,$\pm$\,0.07 & 0.97\,$\pm$\,0.07 & \textbf{1.00}\,$\pm$\,0.00\,\textsuperscript{\dag} \\
 & Artificial5 & \begin{tabular}[c]{@{}l@{}}\texttt{Participant}\\(PartyB)\end{tabular} & 2{,}260 & 0.50 & \underline{0.92}\,$\pm$\,0.07 & 0.93\,$\pm$\,0.06 & 0.98\,$\pm$\,0.02 & 0.95\,$\pm$\,0.05 & \textbf{1.00}\,$\pm$\,0.00\,\textsuperscript{\dag} \\
 & Real4 & \begin{tabular}[c]{@{}l@{}}\texttt{Participant}\\(Zoo)\end{tabular} & 1{,}700 & 0.47 & \textbf{1.00}\,$\pm$\,0.00\,\textsuperscript{\dag} & \textbf{1.00}\,$\pm$\,0.00\,\textsuperscript{\dag} & \textbf{1.00}\,$\pm$\,0.00\,\textsuperscript{\dag} & \textbf{1.00}\,$\pm$\,0.00\,\textsuperscript{\dag} & \textbf{1.00}\,$\pm$\,0.00\,\textsuperscript{\dag} \\
 & BPIC~2013 & \begin{tabular}[c]{@{}l@{}}\texttt{Participant}\\(Org line C)\end{tabular} & 62{,}030 & 0.42 & \underline{0.57}\,$\pm$\,0.01 & 0.64\,$\pm$\,0.01 & 0.63\,$\pm$\,0.01 & 0.61\,$\pm$\,0.01 & \textbf{0.65}\,$\pm$\,0.01 \\
\midrule
\multirow{5}{*}{\begin{tabular}[c]{@{}l@{}}OB-M\\ ($\mathrm{F1}_{\mathrm{macro}}$)\end{tabular}}
 & Healthcare & \begin{tabular}[c]{@{}l@{}}\texttt{Message}\\(Communicate disease)\end{tabular} & 1{,}350 & 1.00\,\textsuperscript{\ddag} & \textbf{1.00}\,$\pm$\,0.00\,\textsuperscript{\dag} & \textbf{1.00}\,$\pm$\,0.00\,\textsuperscript{\dag} & \textbf{1.00}\,$\pm$\,0.00\,\textsuperscript{\dag} & \textbf{1.00}\,$\pm$\,0.00\,\textsuperscript{\dag} & \textbf{1.00}\,$\pm$\,0.00\,\textsuperscript{\dag} \\
 & Artificial1 & \begin{tabular}[c]{@{}l@{}}\texttt{Message}\\(Activity CZ)\end{tabular} & 700 & 0.39 & \textbf{0.92}\,$\pm$\,0.02 & \textbf{0.92}\,$\pm$\,0.02 & \textbf{0.92}\,$\pm$\,0.00 & \underline{0.91}\,$\pm$\,0.01 & \textbf{0.92}\,$\pm$\,0.01 \\
 & Artificial5 & \begin{tabular}[c]{@{}l@{}}\texttt{Message}\\(Activity AA)\end{tabular} & 2{,}260 & 0.49 & \textbf{1.00}\,$\pm$\,0.00\,\textsuperscript{\dag} & \textbf{1.00}\,$\pm$\,0.00\,\textsuperscript{\dag} & 0.96\,$\pm$\,0.08 & \underline{0.91}\,$\pm$\,0.17 & \textbf{1.00}\,$\pm$\,0.00\,\textsuperscript{\dag} \\
 & Real4 & \begin{tabular}[c]{@{}l@{}}\texttt{Message}\\(Send info. to the ZooClub dept.)\end{tabular} & 1{,}700 & 0.47 & \textbf{1.00}\,$\pm$\,0.00\,\textsuperscript{\dag} & \textbf{1.00}\,$\pm$\,0.00\,\textsuperscript{\dag} & \textbf{1.00}\,$\pm$\,0.00\,\textsuperscript{\dag} & \textbf{1.00}\,$\pm$\,0.00\,\textsuperscript{\dag} & \textbf{1.00}\,$\pm$\,0.00\,\textsuperscript{\dag} \\
 & BPIC~2013 & \begin{tabular}[c]{@{}l@{}}\texttt{Message}\\(Queued)\end{tabular} & 62{,}030 & 0.41 & 0.65\,$\pm$\,0.00 & \textbf{0.71}\,$\pm$\,0.01 & \underline{0.64}\,$\pm$\,0.02 & \textbf{0.71}\,$\pm$\,0.01 & \textbf{0.71}\,$\pm$\,0.01 \\
\bottomrule
\end{tabular}\\[2pt]
}
\end{widesidewaystable}

\newpage

%=====================================================================
\section*{CRediT authorship contribution statement}
\textbf{[Daniel Calegari]:} Conceptualization, Methodology, Investigation, Formal analysis, Software, Validation, Writing -- original draft, Writing -- review \& editing.
\textbf{[Andrea Delgado]:} Conceptualization, Investigation, Validation, Writing -- review \& editing.
\textbf{[Leonel Peña]:} Software, Validation.
\textbf{[Martín Rubio]:} Software, Validation.

\section*{Declaration of competing interest}
Given her role as Guest Editor of the special issue 
``Collaboration Process Mining for Distributed Systems'', 
Andrea Delgado had no involvement in the peer-review of this article and has no access to information regarding its peer-review. Full responsibility for the editorial process for this article was delegated to another journal editor.  Other authors declare that they have no known competing financial interests or personal relationships that could have appeared to influence the work reported in this paper.

\section*{Declaration of generative AI and AI-assisted technologies in the manuscript preparation process}
During the preparation of this work, the authors used Claude, ChatGPT, and Gemini in order to assist with drafting and language refinement of the manuscript, to support the structuring of its conceptual and methodological argumentation, and to assist with code generation. After using these tools, the authors reviewed and
edited the content as needed and took full responsibility for the content of the published article.

\section*{Data availability}
The source code, experimentation scripts, and results are publicly available \cite{OCPPMcollab}; the original event logs are those from \cite{delgado2025predictcollab} and \cite{DBLP:conf/bpm/2013bpic}.

\bibliographystyle{elsarticle-num}
\bibliography{bibliography}
 
\end{document}